\documentclass[11pt]{article}

\usepackage{amsfonts, amsmath,amssymb,array,latexsym, hyperref, amscd}

\usepackage{cancel}

\usepackage{fancyhdr,a4wide}
\usepackage{amsthm}
\usepackage{caption}
\usepackage{subcaption}
\usepackage{graphicx}
\usepackage{xcolor}
\usepackage{enumitem}
\usepackage{physics}

\newtheorem{theorem}{Theorem}[section]
\newtheorem{lemma}[theorem]{Lemma}
\newtheorem{corollary}[theorem]{Corollary}
\newtheorem{proposition}[theorem]{Proposition}

\newtheorem{definition}[theorem]{Definition}
\newtheorem{remark}[theorem]{Remark}

\numberwithin{equation}{section}

\newcommand{\lr}[1]{\left(#1\right)}
\newcommand{\set}[1]{\left\{#1\right\}}
\newcommand{\E}[1]{\mathbb E\left[#1\right]}
\newcommand{\inprod}[2]{\left \langle #1,#2\right\rangle }

\newcommand{\x}{\times}
\newcommand{\R}{\mathbb R}

\newcommand{\gives}{\rightarrow}

\newcommand{\mN}{\mathcal N}
\newcommand{\mL}{\mathcal L}

\newcommand{\mD}{\mathcal D}

\newcommand{\twomat}[4]{\lr{\begin{array}{cc} #1 & #2 \\ #3 & #4 \end{array}}}

\newcommand{\sgn}[1]{\mathrm{sgn}\left(#1\right)}

\newcommand{\w}{\omega}

\DeclareMathOperator*{\argmin}{arg\,min}

\newcommand{\Esub}[2]{\mathbb E_{#1}\left[#2\right]}
\newcommand{\wx}{\widehat{x}}

\newcommand{\wtheta}{\widehat{\theta}}

\newcommand{\omu}{\overline{\mu}}
\newcommand{\onu}{\overline{\nu}}
\newcommand{\wpsi}{\widehat{\psi}}

\newcommand{\wbXd}{\widehat{X}_{\beta}^\dag}
 
\newcommand{\wX}{\widehat{X}}
\newcommand{\wxtpb}{(\wx^{\otimes 3})_{\perp,\beta}}

\newcommand{\wxpb}{\wx_{\perp,\beta}}

\newcommand{\diag}{\mathrm{Diag}}
\newcommand{\se}{\sigma_\epsilon}
\newcommand{\pri}{\mathbb P_{\text{prior}}}
\newcommand{\postb}{\mathbb P_{\text{post},\beta}}

\author{Boris Hanin\footnote{BH is funded by a 2024 Sloan Fellowship in Mathematics, NSF CAREER grant DMS-2143754, NSF grant DMS-2133806, and a MURI in Foundations of Deep Learning.},  Alexander Zlokapa\footnote{AZ is supported by the Hertz Foundation, and by the DoD NDSEG.}\\
$\,\textsuperscript{*}$Princeton ORFE\\
$\,\textsuperscript{$\dagger$}$MIT Center for Theoretical Physics, Google Quantum AI}
\title{Bayesian Inference with Deep Weakly Nonlinear Networks}
\begin{document}
\maketitle

\begin{abstract}
    We show at a physics level of rigor that Bayesian inference with a fully connected neural network and a shaped nonlinearity of the form $\phi(t)= t + \psi t^3/L$ is (perturbatively) solvable in the regime where the number of training datapoints ($P$), the input dimension ($N_0$), the network’s hidden layer widths ($N$), and the network depth ($L$) are simultaneously large. Our results hold with weak assumptions on the data model; the main constraint is that $P < N_0$. We provide techniques to compute the model evidence and posterior to arbitrary order in $1/N$ and at arbitrary temperature. We report the following results from the first-order computation.
    \begin{itemize}
        \item When the width $N$ is much larger than the depth $L$ and training set size $P$, neural network Bayesian inference coincides with Bayesian inference using a kernel. The value of $\psi$ determines the curvature of a sphere, hyperbola, or plane into which the training data is implicitly embedded under the feature map.
        \item When $LP/N$ is a small constant, neural network Bayesian inference departs from the kernel regime. At zero temperature, neural network Bayesian inference is equivalent to Bayesian inference using a data-dependent kernel, and $LP/N$ serves as an effective depth that controls the extent of feature learning.
        \item In the restricted case of deep linear networks ($\psi=0$) and noisy data, we show a simple data model for which evidence and generalization error are optimal at zero temperature. As $LP/N$ increases, both evidence and generalization further improve, demonstrating the benefit of depth in benign overfitting.
    \end{itemize}
\end{abstract}

\newpage
\tableofcontents
\newpage

\section{Introduction}
Neural networks are efficiently differentiable parametric families of functions that underlie state-of-the-art algorithms in applied domains ranging from natural language processing \cite{brown2020language} to computer vision \cite{krizhevsky2012imagenet} and structural biology/chemistry \cite{jumper2021highly}. The empirical success of neural networks is driven in large part by their remarkable ability to efficiently adapt to and learn from rich training datasets. 

This effect is observed most strongly in practice when exceptionally large networks, with billions or even trillions of parameters, are trained to fit comparably large high-dimensional datasets \cite{kaplan2020scaling, hoffmann2022training, hestness2017deep}. Explaining how network architecture, statistics of the training data, and learning rules interact to enable data-driven adaptation at inference time is therefore a core question in deep learning theory, especially in the setting where model size and dataset size are both large. 

In this article, we make progress on this question by developing at a physics level of rigor the first (perturbatively) solvable model of learning with neural networks that are nonlinear in their parameters as well as their inputs and in which the number of training datapoints, the input dimension, the network’s hidden layer widths, and the network depth are simultaneously large.\footnote{The only structural constraint on these parameters is that we need to assume that the number of training datapoints is less than the input dimension.} Specifically, we consider depth $L$ fully connected networks $x\in \R^{N_0}\mapsto f(x;\theta)\in\R$ with hidden layer widths $N_1,\ldots, N_L\geq 1$, defined by 
\begin{align}
    \label{eq:f-def} f(x;\theta) = \frac{1}{\sqrt{N_L}}W^{(L+1)}x^{(L)},\qquad x^{(\ell)}:=\begin{cases}
        \phi\lr{\frac{1}{\sqrt{N_{\ell-1}}}W^{(\ell)}x^{(\ell-1)}}\in \R^{N_{\ell}},&\quad \ell \geq 1\\
        x,&\quad \ell = 0
    \end{cases}.
\end{align}
Since we are interested in the behavior of these models as the depth $L$ grows we consider the \textit{shaped activation}
\begin{equation}\label{eq:phi-def}
\phi(t) := t + \frac{\psi}{3L}t^3,    
\end{equation}
in which the parameter $\psi\in \R$ determines the strength of the nonlinearity. Although the $1/L$ scaling of the cubic term causes each network layer to be close to linear, the aggregate effect over all $L$ network layers is order $1$. The $1/L$ scaling is in fact necessary to ensure non-degenerate prior distributions over network outputs $f(x;\theta)$ at large depth \cite{li2022neural,noci2024shaped, zhang2022deep}.\footnote{The nonlinearity \eqref{eq:phi-def} can be thought of as a generic odd shaped nonlinearity since the leading order contribution in $1/L$ is obtained by ``shaping'' $\sqrt{L}\cdot \sigma(t/\sqrt{L})$ for a smooth odd function $\sigma:\R\gives \R$ with $\sigma(0)=0,\sigma'(0)=1$, which is generic in the sense that $\sigma'''(0)\neq 0$ and higher powers of $t$ are suppressed by further factors of $L$.} For our learning rule we take Bayesian inference with an iid Gaussian prior over the network weights and a quadratic negative log-likelihood 
\[
\mL(\theta\,|\mD):=\frac{1}{2P}\sum_{\mu=1}^P\lr{f(x_\mu;\theta)-y_\mu}^2
\]
given by the mean-squared error over a training dataset
\[
\mathcal D = \set{\lr{x_\mu , y_\mu},\, \mu = 1,\ldots, P},\qquad x_\mu \in\R^{N_0},\, y_\mu \in \R.
\]

\subsection{Informal Overview of Results}\label{sec:inf} In this section we present an informal overview of our results, deferring to \S \ref{sec:precise} the precise statements. We will work in the regime where the width, depth, and number of datapoints are all large:
\[
N,L,P\gg 1\qquad \text{and\qquad }c N \leq N_\ell \leq C N,\quad \ell=1,\ldots, L,\qquad 0<c<C, 
\]
with hidden layer widths proportional to a large constant $N$. The prior over network parameter is iid Gaussian:
\[
   W_{ij}^{(\ell)}\sim \mN(0,1)\,\, \text{iid}.\footnote{In our results we actually consider the more general case of weights variances of the form $\sigma^2 = 1+\eta/L$, which are the only ones that lead to non-degenerate priors at large $L$.}
\]
To understand the distribution of network outputs $f(x;\theta)$ when $\theta$ is sampled from the posterior (see \S \ref{sec:bayes} for the definition) we develop a set of new analytical tools for computing the corresponding partition function with a source
\begin{equation}\label{eq:part-def}
Z_\beta(x;\tau):=\Esub{\text{prior},\beta}{\exp\left[-i\tau f(x;\theta)-\beta \mathcal L(\theta\,|\,\mD)\right]},\qquad x\in \R^{N_0}, \, \tau \in \R,\footnote{In our results we always consider the case of a single test point $x$, but there is no additional difficulty in computing the joint statistics of $f(x;\theta)$ at multiple $x$s.}    
\end{equation}
where $\beta$ is the inverse temperature. The characteristic function of the predictive posterior (i.e., the distribution of $f(x;\theta)$ with $\theta$ sampled from the posterior) is determined by
\[
\Esub{\text{post},\beta}{\exp\left[-i\tau f(x;\theta)\right]} = \frac{Z_\beta(x;\tau)}{Z_\beta(0)},\qquad Z_\beta(0):=Z_\beta(x;0).
\]
The normalizing constant $Z_\beta(0)$ is referred to in the Bayesian inference literature as the Bayesian model evidence and will play a key role in our results (see e.g., \cite{mackay1992bayesian}). 

Since the exponent inside the expectation in \eqref{eq:part-def} is quadratic in the outputs of the network and the expectation is with respect to the prior, a crucial starting point for understanding $Z_\beta(x;\tau)$ is to characterize the joint distribution of the field $x\mapsto f(x;\theta)$ evaluated at both a test point $x$ and all training inputs $x_\mu$ when $\theta$ is sampled from its prior. While this prior is simple when viewed in parameter space, it corresponds to a complicated distribution over network outputs \cite{banta2023structures,hanin2020products,hanin2022random,hanin2023bayesian,roberts2022principles,schoenholz2017correspondence,yaida2019non}. Note from \eqref{eq:f-def} that for any network input $x$, the corresponding output $f(x;\theta)$ under the prior is the scalar product of the independent standard $N_L$-dimensional Gaussian vector $W^{(L+1)}$ with the normalized vector $N_L^{-1/2} x^{(L)}$ of post-activations in the final hidden layer. The joint distribution of the network outputs at any finite number of inputs under the prior is thus a centered Gaussian with an independent random covariance whose entries are overlaps of the final hidden layer representations:
\[
\lr{f(x_\mu;\theta),\, \mu =1,\ldots, P} ~\sim~ \mathcal N\lr{0,\lr{\frac{1}{N_L}\inprod{x_{\mu}^{(L)}}{x_\nu^{(L)}}}_{1\leq \mu,\nu\leq P}}.
\]
The statistics of the field $x\mapsto f(x;\theta)$ under the prior are hence determined by all moments of such overlaps:
    \begin{equation}\label{eq:prior-moments-def}
        \Esub{\text{prior}}{\prod_{p=1}^q \frac{1}{N_L}\inprod{x_{\mu_p}^{(L)}}{x_{\nu_P}^{(L)}}},\qquad x_{\mu_p}, x_{\nu_p}\in \R^{N_0},\qquad q \geq 1.
    \end{equation}
The first contribution of this article is a probabilistic combinatorial model for evaluating these expectations at large $L$, valid for any $\psi\in \R$ (see \S \ref{sec:prior-model}). The key aspects of our model are:
\begin{itemize}
    \item Our model interprets the expectations in  \eqref{eq:prior-moments-def} as averages of products of edge-weights over a certain ensemble of random graphs. The vertex set in this graph ensemble is deterministic and is labeled by inputs $x_{\mu_p}, x_{\nu_p},\, p=1,\ldots, q$ appearing in the product. But the number of edges between any two vertices is random. Each copy of an edge $(x_\mu,x_\nu)$ is weighted by a factor of $N_0^{-1}\inprod{x_\mu}{x_\nu}$ (see \S \ref{sec:prior-model} for details). This model is a significant extension of the combinatorial sum-over-paths approach developed in \cite{hanin2018neural,hanin2020products} for the case $\psi = 0$ and a computationally tractable alternative to the covariance SDE approach of \cite{li2022neural}.
    
    \item While we do not know how to solve the combinatorial model exactly as $N,L\gives \infty$ with constant $L/N$, we are able to solve it perturbatively in powers of $1/N$ in the regime where $L, N \to \infty$ while $L/N \to 0$.
    
    \item Analyzing this combinatorial model allows us to compute the partition function $Z_\beta(x;\tau)$ at finite temperature perturbatively in $1/N$. This is of interest even for deep linear networks with $\psi=0$, for which virtually all prior work required passing to zero temperature \cite{li2021statistical,hanin2023bayesian,zavatone2022contrasting}. Though we do not pursue this in the present article, our analysis also allows one to consider the case of any fixed output dimension. 
\end{itemize}
In principle, one can use our formalism to compute the partition function to any order in $1/N$. In the present article, we work it out explicitly to zeroth and first order in $1/N$ in \S \ref{sec:part-exp}. This already reveals a range of both qualitative and quantitative insights into the nature of Bayesian inference with deep shaped MLPs, including the following.
\begin{itemize}
\item \textbf{Nonlinearity strength $\psi$ as curvature of implicit data manifold.} To zeroeth order in $1/N$ (i.e., at leading order in the regime where $N\gg L,P$), we find that if $N,L,P\gives \infty$ then Bayesian inference using deep fully connected networks with shaped nonlinearities \eqref{eq:phi-def} is equivalent to Bayesian inference using a kernel with feature map 
    \begin{equation}\label{eq:hat-def}
        x\in \R^{N_0}\quad \mapsto \quad x_\psi:=\lr{1-2\psi \frac{\norm{x}^2}{N_0}}^{-1/2}\cdot \frac{x}{\sqrt{N_0}}.
    \end{equation}
As we explain in \S \ref{sec:kernel-regime}, the map \eqref{eq:hat-def} can be thought of as embedding $x$ into either a hyperbola, the sphere, or a plane in $\R^{N_0+1}$, depending on the sign of $\psi$. The magnitude of $\psi$ then determines the curvature. We summarize this in the following 
\begin{align*}
    \text{\underline{Takeaway}:}\quad &\text{When $N\gg L,P$, Bayesian inference with shaped MLPs is a kernel}\\
    &\text{method in which the nonlinearity strength }\psi\text{ implicitly determines}\\
    &\text{the curvature of the manifold (sphere or hyperbola or plane) into }\\
    &\text{which training data inputs are embedded.}
\end{align*}
\item \textbf{$LP/N$ as emergent effective posterior depth.} When working perturbatively in $1/N$ around the kernel regime $N\gg L,P$, we find that the first order correction to the posterior is cubic in the features \eqref{eq:hat-def} and in general does not coincide with any kernel method. Whenever the $1/N$ correction increases the Bayesian model evidence, we have that (just as in the case of deep linear networks \cite{hanin2023bayesian}) the strength of this correction scales like $LP/N$ (see Corollary \ref{cor:zero-T-scaling}), giving the following:
\begin{align*}
    \text{\underline{Takeaway}:}\quad &\frac{LP}{N} = \frac{\text{depth}\times \text{samples}}{\text{width}} \text{ plays the role of an effective depth of the}\\
    &\text{network under the posterior. Zero effective depth corresponds to a}\\
    &\text{data-independent kernel method.}
\end{align*}
This provides a concrete description of how increasing $LP/N$ controls the onset of feature learning.
\end{itemize}
To analyze the structure of the posterior at first order in $1/N$, let us suppose that after passing through the feature map \eqref{eq:hat-def} inputs are generated iid from a distribution with 
\begin{align}
    \label{eq:input-stats}\E{(x_\mu)_\psi}=0,\qquad \Sigma:=\E{\lr{x_\mu}_\psi \lr{x_\mu}_\psi^T}.
\end{align}
The explicit form of the partition function $Z_\beta(x;\tau)$ perturbatively to first order in $1/N$ (see Proposition \ref{prop:part-exp}) reveals several ways in which the posterior adapts to the statistics of the training data.

\begin{itemize}
    \item \textbf{Criterion for posterior linearity wrt inputs.} In the setting of iid training and test inputs for any $\psi\in \R$ and to first order in $1/N$, we show that the predictive posterior of a deep nonlinear network coincides with the posterior of a deep linear network on the transformed data $\set{(x_\mu)_\psi, y_\mu}$ if and only if 
    \begin{equation}\label{eq:linear-post}
        \frac{\tr(\Sigma^2)}{\tr(\Sigma)^2}=o(1).
    \end{equation}
We conjecture that this holds at all orders in $1/N$ and summarize this in the following:
\begin{align*}
    \text{\underline{Takeaway}:}\quad & \text{When presented with data satisfying \eqref{eq:linear-post}, Bayesian inferences gives}\\
    &\text{to first order in $1/N$ posteriors in which $f(x;\theta)$ is linear in $x_\psi$.}
\end{align*}
At first sight this is rather surprising: why would the posterior ``ignore'' the nonlinearity for data satisfying \eqref{eq:linear-post}? We give an intuitive explanation in \S \ref{sec:res} (see around \eqref{eq:linear-post-formal}). 
\end{itemize}
To understand the nature of posteriors that are nonlinear in $x_\psi$, we investigate a broad class of ``nice'' data generating processes in which \eqref{eq:linear-post} is violated (see \S \ref{sec:perturbation}, specifically around \eqref{eq:power-law}). Here we find:

\begin{itemize}
\item \textbf{Exact form of $LP/N$ posterior correction.} For ``nice'' data generating processes, maximizing the evidence requires taking the temperature to zero (i.e., $\beta \gives \infty$). At zero temperature, we find the following:
\begin{align}
 \notag      \text{\underline{Takeaway}:}\quad &\text{At first order in }1/N,\text{ the posterior coincides with a \textit{data-dependent}}\\
       \notag &\text{kernel method in which train and test inputs are re-weighted}\\
&\label{eq:data-kernel}\qquad \qquad \qquad x_\psi \quad \mapsto \quad\lr{1+\text{const}_\psi\times \frac{LP}{N}\times x_\psi^T \Sigma x_\psi } x_{\psi}\\
\notag &\text{according to their overlap with the sample covariance matrix $\Sigma$.}
\end{align}
Since this re-weighting depends on the statistics of the data-generating process, we see that the first order in $1/N$ correction corresponds to a data-dependent kernel method. 
\end{itemize}
Finally, we add label noise in the data-generating process and ask when the Bayesian evidence and generalization error are maximized by overfitting (i.e., passing to zero temperature).
\begin{itemize}
\item \textbf{Depth makes overfitting more benign.} We examine a ``nice'' data generating process in the presence of label noise, and we show that the perturbative solutions for deep linear networks remain valid. In this setting, we find:
\begin{align*}
    \text{\underline{Takeaway}:}\quad &\text{At first order in $1/N$ and for label noise with variance $\se^2 < \sigma_0^2$ for some $\sigma_0^2$,}\\
    \quad &\text{the evidence and generalization error of a deep linear network are both}\\
    \quad &\text{optimal at zero temperature and improve with increasing $LP/N$.}
\end{align*}
Hence, in both a Bayesian and traditional sense, depth improves the behavior of overfitting in the regime of benign overfitting. See Corollary \ref{cor:bayes-benign}.
\end{itemize}

\subsection*{Acknowledgements}
The present line of work began during the 2022 Les Houches Summer school on Statistical Physics of Machine Learning. It was then further developed at the 2023 Physics of Machine Learning Workshop at the Aspen Center for Physics and the 2023 Workshop on Statistical Physics and Machine Learning at Institute d'\'Etudes Scientifiques de Cargese. We gratefully acknowledge all these wonderful venues and workshop organizers.

\section{Statement of Results and Relation to Prior Work}\label{sec:precise}
In this section, we present in detail our model and main results. In the course of doing so, we also review a range of prior work (see also \S \ref{sec:lit-rev}). We begin in \S \ref{sec:def} with describing our assumptions on the model, data, and Bayesian inference procedure. We then present in \S \ref{sec:res} our main results. 

\subsection{Definitions}\label{sec:def}
We discuss in \S \ref{sec:shaped} the definition of shaped MLPs. We then present in \S \ref{sec:data} our assumptions on the data-generating process. Finally, in \S \ref{sec:bayes}, we specify the details of the Bayesian inference procedure.  

\subsubsection{Model: Shaped Neural Networks}\label{sec:shaped}
As described above, this article concerns shaped MLPs $x\in \R^{N_0}\mapsto f(x;\theta)\in \R$ with depth $L$ and hidden layer widths $N_1,\ldots, N_L$, which we recall are defined by 
\begin{align}
    \label{eq:f-def-formal} f(x;\theta) = \frac{1}{\sqrt{N_L}}W^{(L+1)}x^{(L)},\qquad x^{(\ell)}:=\begin{cases}
        \phi\lr{\frac{1}{\sqrt{N_{\ell-1}}}W^{(\ell)}x^{(\ell-1)}}\in \R^{N_{\ell}},&\quad \ell \geq 1\\
        x,&\quad \ell = 0
    \end{cases}.
\end{align}
Throughout this article, we assume that
\begin{align}\label{eq:const-width}
    N, L\gives \infty\qquad\text{with}\qquad c N \leq N_\ell \leq C N,\quad \ell=1,\ldots, L,\qquad 0<c<C.
\end{align}
Our analysis in this article exclusively concerns the \textit{shaped} nonlinearity 
\begin{equation}\label{eq:phi-def-formal}
\phi(t) := t + \frac{\psi}{3L}t^3.    
\end{equation}
The $1/L$ scaling is necessary to keep output correlations order $1$ at large $L$ when weights are sampled from iid standard Gaussians \cite{li2022neural, noci2024shaped, zhang2022deep}. To illustrate this point, consider a fixed nonlinearity such as 
\[
\phi(t)=\tanh(t).
\]
In this case for any two network inputs $x_\mu, x_\nu$ if some for $C>0$ we have
\[
N,L\gives \infty\quad\text{and}\quad  L\leq CN\qquad \Longrightarrow\qquad 
\frac{1}{N_L}\inprod{x_\mu^{(L)}}{x_\nu^{(L)}} \gives 0. 
\]
See Appendix B in \cite{hanin2022random} and \S 5.3.3 in \cite{roberts2022principles}. Intuitively this is because, integrating out the weights $W^{(\ell+1)}$ in layer $\ell+1$ shrinks the size of inputs:
\begin{align*}
\Esub{W^{(\ell+1)}}{\frac{1}{N_{\ell+1}}\norm{x_\mu^{(\ell+1)}}^2} &= \Esub{W^{(\ell+1)}}{\frac{1}{N_{\ell+1}}\sum_{i=1}^{N_{\ell+1}}\phi\lr{\frac{1}{\sqrt{N_\ell}}W_i^{(\ell+1)}x_\mu^{(\ell)}}^2}\\
&< \Esub{W^{(\ell+1)}}{\frac{1}{N_{\ell+1}}\sum_{i=1}^{N_{\ell+1}}\lr{\frac{1}{\sqrt{N_\ell}}W_i^{(\ell+1)}x_\mu^{(\ell)}}^2}\\
&=\frac{1}{N_\ell}\norm{x_\mu^{(\ell)}}^2,
\end{align*}
where we used that $\abs{\tanh(z)}<\abs{z}$ whenever $z\neq 0$.\footnote{Actually all that matters is the behavior of $\tanh$ near zero (see \cite{roberts2022principles,hanin2022random}).} Similar conclusions hold for not only for $\tanh$ but also for any smooth $\tanh$-like nonlinearity (see e.g., \cite{roberts2022principles,hanin2022random, price2024deep}). In contrast, given a smooth odd function $\sigma:\R\gives \R$ with $\sigma(0)=0,\sigma'(0)=1$, consider the corresponding shaped nonlinearity
\begin{equation}\label{eq:shaping}
\phi(t):=\sqrt{L}\sigma\lr{t/\sqrt{L}}=t + \frac{\sigma'''(0)}{6L}t^3+O(L^{-2}).    
\end{equation}
Then repeating the same computation as above shows that
\begin{align*}
\Esub{W^{(\ell+1)}}{\frac{1}{N_{\ell+1}}\norm{x_\mu^{(\ell+1)}}^2} &=
\frac{1}{N_\ell}\norm{x_\mu^{(\ell)}}^2+O\lr{L^{-1}}.
\end{align*}
More formally, defining a continuous analog of the layer index $\tau: = \ell/ L$, the article \cite{li2022neural} shows that when $N, L\gives \infty$ with $L/N$ remaining bounded the matrix of overlaps
\[
V^{(\ell)}:=\lr{\frac{1}{N_\ell}\inprod{x_\mu^{(\ell)}}{x_\nu^{(\ell)}}}_{\mu,\nu=1,\ldots, P}
\]
satisfies an explicit nonlinear stochastic differential equation with time parameter $\tau$ in which $V^{(\ell)}$ neither converges to zero nor diverges in finite time. This is an illustration of the fact that the shaping \eqref{eq:shaping} is necessary for the prior distribution over network outputs to be non-degenerate.

\subsubsection{Data}\label{sec:data}
We will consider training datasets 
\begin{align}\label{eq:D-def}
    \mD=\set{\lr{x_\mu,y_\mu},\, \mu=1,\ldots, P},\quad x_\mu \in \R^{N_0},\, y_\mu \in \R,
\end{align}
for which we always make the following assumptions:
\begin{align}
    \label{eq:over-param}\text{(i)}&\quad  \text{dataset size} ~=~P ~<~N_0~=~\text{ input dimension} \\
    \label{eq:non-degen-data}\text{(ii)}&\quad \text{the matrix of inputs } X = \lr{x_\mu,\, \mu=1,\ldots, P}\in \R^{N_0\times P}\text{ has full rank}\\
    \label{eq:order-one-outputs}\text{(iii)}&\quad \text{labels $y_\mu$ are order $1$ for all $\mu$}
\end{align}
While assumptions (ii) and (iii) are rather mild, assumption (i) is a significant restriction, which we do not know how to remove (see \S \ref{sec:limitations}). In many of our results, the minimal norm interpolant of the data will appear:
\begin{align}\label{eq:theta-def}
    \theta_* := \argmin_{\theta \in \R^{N_0}} \norm{\theta}_2\quad \text{s.t.}\quad \theta^Tx_\mu =y_\mu\quad \forall \mu.
\end{align}
We work in the normalization where all examples $x_\mu$ in the training dataset have norm $\E{\norm{x_\mu}}^2 = O(1)$ and all labels have $\E{y_\mu^2} = O(1)$.
When determining which terms remain to leading order, we also introduce a concentration / self-averaging assumption: for two data points $x_\mu, x_\nu$ in the training dataset, and their corresponding vectors $(x_\mu)_\psi, (x_\nu)_\psi$ under the feature map introduced above (see \eqref{eq:hat-def}),
\begin{align*}
    \inprod{(x_\mu)_\psi}{(x_\nu)_\psi}^2 \approx \E{\inprod{(x_\mu)_\psi}{(x_\nu)_\psi}^2}.
\end{align*}
This is only used to determine the scalings of terms; once the leading-order terms have been identified, we report the partition function (and evidence and posterior) without use of this approximation. Below, we present more concrete takeaways under a power law data model where this self-averaging behavior explicitly holds.

\subsubsection{Prior, Likelihood, Posterior}\label{sec:bayes}
Bayesian inference in neural networks has a long history \cite{neal1996priors,gal2016dropout,zavatone2021exact,zavatone2022contrasting,hron2022wide,novak2018bayesian,maddox2019simple,wilson2020bayesian,hanin2023bayesian,he2020bayesian,li2021statistical,izmailov2021bayesian,naveh2021self,seroussi2023separation,lee2017deep,ariosto2022statistical,garriga2018deep,cui2023optimal} (see \S \ref{sec:lit-rev} for a brief review). Just as in virtually all those articles, we take the prior distribution over network weights to be Gaussian
\begin{equation}\label{eq:prior-def}
    \theta \sim \pri(\theta\,|\,N_\ell, L, \eta)\quad\Longleftrightarrow \quad W_{ij}^{(\ell)}\sim \mathcal N\lr{0,1+\frac{\eta}{L}}\text{ iid},
\end{equation}
where $\eta \in \R$ is fixed and the $1+O(1/L)$ scaling for the weight variance is required to keep the prior over network outputs non-degenerate at large $L$. We also consider the negative log-likelihood
\[
\mathcal L(\theta\,|\,\mathcal D)= \frac{1}{2P}\sum_{\mu=1}^P \lr{f(x_\mu;\theta)-y_\mu}^2.
\]
The prior and likelihood determine, for every value of the inverse temperature $\beta\in [0,\infty],$ a posterior
\begin{equation}\label{eq:pre-def}
    \postb(\theta\,|\,\mD, N_\ell, L, \psi, \eta)=\frac{\pri(\theta\,|\,N_\ell, L,\eta)\exp\left[-\beta\mathcal L(\theta\,|\,\mathcal D)\right]}{\mathbb P\lr{\mD\,|\,N_\ell, L, \psi, \eta}}
\end{equation}
over network parameters. We will often write $\Esub{\text{prior},\beta}{\cdot}$ and $\Esub{\text{post},\beta}{\cdot}$ for expectations over the prior and posterior at a given inverse temperature $\beta$, respectively. 

In this article we are primarily concerned with understanding the predictive posterior, i.e., the distribution of network outputs $f(x;\theta)$ when $\theta$ is sampled from the posterior. This predictive posterior is determined by its characteristic function:
\[
\Esub{\text{post},\beta}{\exp\left[-i\tau f(x;\theta)\right]}=\frac{Z_\beta(x,\tau)}{Z_\beta(0)},
\]
where
\[
Z_\beta(x,\tau) = \Esub{\text{prior},\beta}{\exp\left[-i\tau f(x;\theta)-\beta\mL(\theta\,|\,\mD)\right]}
\]
is the partition function with a source at inverse temperature $\beta$. Throughout this article we consider just a single source, i.e., the distribution of $f(x;\theta)$ at a single input $x$. However, our results generalize with no further work to the case of a finite number of sources and hence determine the finite-dimensional distributions of the predictive posterior.

\subsection{Results}\label{sec:res}
In this section, we present our main results, which we derive at a physics level of rigor. We begin in \S \ref{sec:kernel-regime} by stating an equivalence between Bayesian inference with deep shaped MLPs and a kernel method in the regime where the network width is much larger than both the depth and dataset size. We then describe in \S \ref{sec:perturbation} the nature of the correction to this kernel regime at first order in $1/\text{width}.$ Finally, we present in \S \ref{sec:prior-model} a combinatorial model for the moments of the prior distribution over network outputs in deep shaped MLPs. 

\subsubsection{Kernel Regime for Shaped Networks}\label{sec:kernel-regime}
A fundamental observation is that if a parameterized family of functions $x\mapsto g(x;\theta)$ is a Gaussian field when $\theta$ is sampled from its prior, then performing Bayesian inference with a log-likelihood given by the mean squared error on any training dataset yields a predictive posterior over $g(x;\theta)$ that is Gaussian as well \cite{mackay1992bayesian,rasmussen2010gaussian}. 

An important example of this in the context of the present article is Bayesian inference with wide neural networks. Specifically, consider a fully connected neural network as in \eqref{eq:f-def-formal} but with at fixed depth $L$, an arbitrary (not necessarily shaped) activation, and hidden layer widths proportional to a large constant $N$. It is well-known that if network weights are independent standard Gaussians, then the field $x\mapsto f(x;\theta)$ is asymptotically Gaussian in the infinite width limit $N\gives \infty$ \cite{hanin2021random,yang2019scaling,yang2019tensori,yang2020tensorii, lee2017deep, garriga2018deep,hron2022wide, favaro2023quantitative}. Hence, when the log-likelihood is the mean-squared error over a fixed training dataset, the Bayesian predictive posterior at infinite width is also Gaussian and is equivalent to that of a linear model with associated kernel 
\begin{equation}\label{eq:NNGP}
K^{(L)}(x_\mu,x_{\nu}):=\lim_{N_1,\ldots, N_L\gives \infty} \frac{1}{N_{L}}\inprod{x_\mu^{(L)}}{x_\nu^{(L)}}.    
\end{equation}
The overlaps $N_L^{-1}\inprod{x_\mu^{(L)}}{x_\nu^{(L)}}$ are self-averaging in the large width limit and the kernel $K^{(L)}$ is therefore deterministic. A natural question is then what happens if not only the network width $N$ but also dataset size $P$ and the number of layers $L$ can be large. There are two important strands of prior work in this direction. 

The first concerns the regime in which the network depth is fixed but both the dataset size $P$ and the network width $N$ tend to infinity. The articles \cite{naveh2021self,seroussi2023separation,cui2023optimal,ariosto2022statistical,hanin2023bayesian,li2021statistical,aiudi2023local,fischer2024critical} show that Bayesian inference with such networks is equivalent to the kernel method with kernel \eqref{eq:NNGP} provided that $P/N\gives 0$. In contrast, if $P/N$ remains bounded away from zero, these article shows that the predictive posterior does not come from a kernel method. The exact description of the deviation for the kernel method depends on the details of how the training data and network architecture are constructed and how learning is analyzed (e.g., using statistical mechanics, the replica trick, etc). 

The second line of prior work concerns the regime in which network depth grows but the dataset size is fixed. The articles \cite{hanin2023bayesian,li2022neural,roberts2022principles, yaida2019non} show in this case that Bayesian inference is equivalent to a linear model with kernel obtained by the large $L$ limit of \eqref{eq:NNGP} provided that $L/N\gives 0$.\footnote{Strictly speaking, to take $L$ large requires either studying infinitesimally separated inputs as in \cite{roberts2022principles} or considering shaped activations as in \cite{li2021statistical}.} In contrast, these same articles also demonstrate that if $L/N$ remains bounded away from zero, Bayesian inference is not equivalent to a linear model. The preceding discussion naturally raises two questions.
\begin{itemize}
    \item \textbf{Q1.} Is it still true that when $N,L,P\gives \infty$ we Bayesian inference with shaped MLPs still reduces to a kernel method when $N$ is much larger than $L,P$ and, if so, how does the kernel depend on $\psi$? 
    \item \textbf{Q2.} Assuming the answer to Q1 is affirmative, how does one generalize the conditions $L/N\gives 0$ and $P/N\gives 0$, which determined convergence to the kernel regime when $P$ and $L$ were fixed respectively, to the case when $N,L,P$ are simultaneously large?
\end{itemize}
Our first set of results provide answers to both questions. As an answer to Q1, we find that when $N\gg L,P$ Bayesian inference with deep shaped MLPs is indeed equivalent to Bayesian inference with a linear model whose underlying kernel is given by
\begin{equation}\label{eq:xpsi-def}
 K_\psi\lr{x_\mu,x_\nu}:=\inprod{(x_\mu)_\psi}{(x_\nu)_\psi},\quad    x_\psi:=e^{\eta}\lr{1-\frac{\psi}{\eta}\lr{e^{2\eta}-1}\frac{\norm{x}^2}{N_0}}^{-1/2}\frac{x}{\sqrt{N_0}}.
\end{equation}
A simple computation, provided in \S \ref{sec:kernel-derivation}, shows that this kernel is obtained by simply taking the large $L$ limit of \eqref{eq:NNGP}: 
\[
K_\psi\lr{x_\mu, x_\nu} = \lim_{L\gives \infty}K^{(L)}\lr{x_\mu,x_\nu}.
\]
Thus, when $N$ is much larger than $L,P$ we may first take $N\gives \infty$, then $L\gives \infty$ and finally $P\gives \infty$. To answer Q2, we must understand just how large $N$ must be for Bayesian inference to stay in this kernel regime. We provide an answer to this question at zero temperature. In the case $\psi = 0$, the article \cite{hanin2023bayesian} shows that so long as $LP/N\gives 0$, Bayesian inference at zero temperature stays in the kernel regime. When $\psi\neq 0$, we find a similar but slightly more complex picture, which depends on the statistics of the training data through the matrix
\[
X_\psi = \lr{\lr{x_\mu}_\psi,\, \mu=1,\ldots, P}\in \R^{N_0\times P}
\]
as follows.
\begin{itemize}
    \item If the training data is such that $\tr((X_\psi^T X_\psi)^{-1})\ll P^2$ then Bayesian inference at zero temperature with shaped MLPs is equivalent to a kernel method if and only if $LP/N\gives 0$. 
    \item If $\tr((X_\psi^T X_\psi)^{-1})\gg P^2$, then Bayesian inference with shaped MLPs is equivalent to a kernel method if and only if $\tr((X_\psi^T X_\psi)^{-1})\cdot L /N\gives 0$. However, for such data generating processes, the Bayesian model evidence (i.e., the partition function $Z_\beta(0)$) decreases at first order in $1/N$.
\end{itemize}
\begin{figure}
    \centering
    \includegraphics[scale=.6]{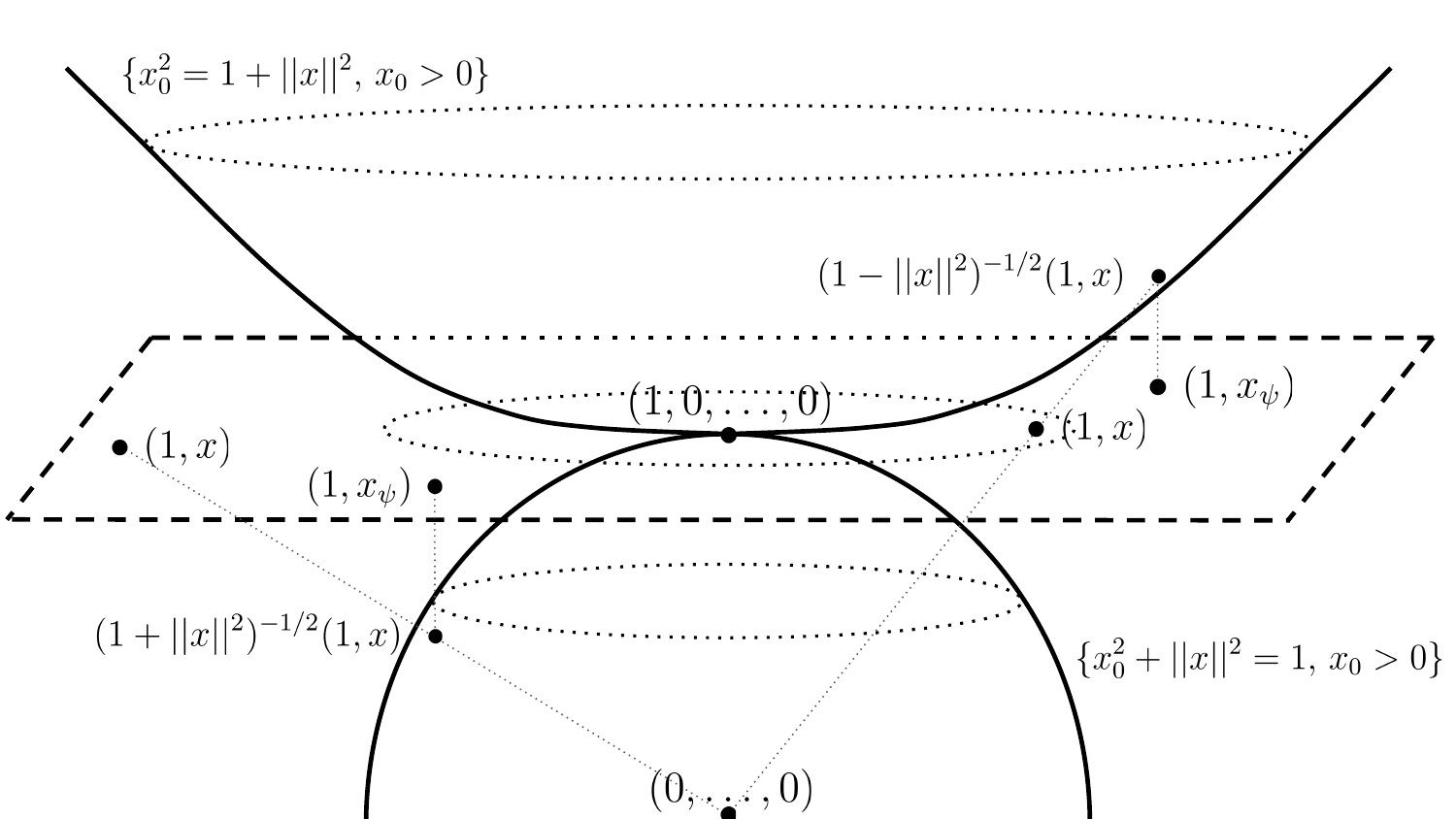}
    \caption{A geometric description of the kernel method that is equivalent to Bayesian inference with shaped MLPs in the regime where $N\gg L,P$. In the figure, we fix the strength $\psi$ of the nonlinearity; varying $\psi$ corresponds to varying the curvature of the hyperbola ($\psi > 0$) or sphere ($\psi < 0$). For simplicity of notation, we write $x$ in the figure for $x/\sqrt{N_0}$.}
    \label{fig:psi-geometry}
\end{figure}

The condition $\tr((X_\psi^TX_\psi)^{-1})\ll P^2$ is rather generic and holds when whenever $\alpha < 2$ in the data-generating processes described in \S \ref{sec:data}. In the next section, we discuss the correction to the predictive posterior relative to this kernel limit to first order in $1/N$. Before doing so, we give a geometric interpretation for the feature map \eqref{eq:xpsi-def} (see Figure \ref{fig:psi-geometry}) which shows that $\psi, \eta$ determine the curvature of a manifold (either a sphere, a hyperbola, or a plane) into which training data are implicitly embedded. To illustrate this let's first take $\eta=0$ and $\abs{\psi}=1/2$, so that the feature map \eqref{eq:xpsi-def} has the form 
    \[
    x_\psi = \lr{1-\sgn{\psi} \frac{\norm{x}^2}{N_0}}^{-1/2}\frac{x}{\sqrt{N_0}}.
    \]
To see the geometric meaning of this, consider first the case when $\sgn{\psi}>0$. Here, we must have that $\norm{x}^2/N_0 < 1$ for the feature map to be well-defined. The point $x_\psi\in \R^{N_0}$ --- or more precisely its lift $(1,x_\psi)\in \R^{N_0+1}$ --- can be obtained from $x$ through the following geometric procedure (see Figure \ref{fig:psi-geometry}):
    \begin{itemize}
        \item identify $x\in \R^{N_0}$ with $(1,x/\sqrt{N_0})\in \R^{N_0+1}$
        \item view the unit disc in $\set{1}\times \R^{N_0}$ as a global coordinate chart for the hyperbola
        \[
        \mathcal H_{N_0}=\set{(x_0,x/\sqrt{N_0})\in \R^{N_0+1}\,|\,x_0^2 =1 + \norm{x}^2/N_0}
        \]
        given by radial projection from the origin and identify $(1,x/\sqrt{N_0})$ with the corresponding point 
        \[
        \lr{1-\norm{x}^2/N_0}^{-1/2} (1,x/\sqrt{N_0})\in \mathcal H_{N_0}
        \]
        \item obtain $(1,x_\psi)$ by representing this point in the alternative coordinate $\mathcal H_{N_0}\gives \set{1}\times \R^{N_0}$ given by orthogonal projection in $\R^{N_0+1}$ onto $\set{1}\times \R^{N_0}$.
    \end{itemize} 
    When $\sgn{\psi} < 0$, as illustrated in Figure \ref{fig:psi-geometry}, we may apply that the same construction to obtain $x_\psi$, but with the hyperbola $\mathcal H_{N_0}$ replaced by the upper hemisphere
    \[
    \mathcal S_{N_0}^+=\set{(x_0,x/\sqrt{N_0})\in \R^{N_0+1}\,|\, x_0^2 + \norm{x}^2/N_0 = 1,\, x_0>0}.
    \] 
    Finally, relaxing the conditions that $\eta = 0, \abs{\psi}=1/2$ has the effect of changing the curvature of the hyperbola or hemisphere from which the features $x_\psi$ are derived, with the exceptional case $\psi=0$ corresponding to the flat Euclidean plane $\set{1}\times \R^{N_0}$.

\subsubsection{Perturbation Theory Around Kernel Regime}\label{sec:perturbation}
In the previous section, we saw that Bayesian inference with shaped MLPs in the regime $L,P,N\gives \infty$ with $N\gg L,P$ corresponds to a $\psi$-dependent kernel method. Our second set of results concerns the properties of the first order in $1/N$ correction to this regime. 

We compute the first order in $1/N$ correction to the partition function $Z_\beta(x;\tau)$ relative to the kernel regime in \S \ref{sec:part-exp}. The result is a deformation of $Z_\beta(x;\tau)$ which is cubic as a function of the feature map $x\mapsto x_\psi$ \eqref{eq:xpsi-def}. We then analyze this formula assuming that test and training points are drawn iid with mean $0$ and population covariance $\Sigma$ after applying the feature map $x\mapsto x_\psi$:
\begin{equation}\label{eq:Sigma-def}
    \E{\lr{x_\mu}_\psi}=0,\qquad \E{\lr{x_\mu}_\psi \lr{x_\mu}_\psi^T}=\Sigma.
\end{equation}
We find:
\begin{itemize}
    \item \textbf{Uninformative data gives deep linear posterior over $K_\psi$.} As we mentioned in the Introduction, for iid training and test inputs we find for every $\psi\in \R$ the first order in $1/N$ correction to the predictive posterior coincides with that of Bayesian inference with a deep linear network (the case $\psi = 0$) with likelihood given by MSE over the dataset
    \[
    \mD_\psi:=\set{\lr{\lr{x_\mu}_\psi, y_\mu},\, \mu=1,\ldots, P}
    \]
    if and only if 
    \begin{equation}\label{eq:linear-post-formal}
        \frac{\tr(\Sigma^2)}{\tr(\Sigma)^2}=o(1).
    \end{equation}
To make sense of this, note that the condition \eqref{eq:linear-post-formal} is equivalent to 
    \begin{equation*}
        \E{\inprod{\lr{x_\mu}_\psi}{\lr{x_\nu}_\psi}^2}=o\lr{\E{\norm{\lr{x_\mu}_\psi}^2}^2},\qquad \mu\neq \nu.
    \end{equation*}
Hence, data generating processes that satisfy \eqref{eq:linear-post-formal} are those for which the training (and test) inputs $\lr{x_\mu}_\psi$ are close to orthogonal. Under our standing hypothesis that the number of datapoints $P$ is strictly less than the input dimension $N_0$, it is not possible to distinguish between iid random labels $y_\mu\sim \mathcal N(0,1)$, labels from a linear generating process $y_\mu = \theta_*^T \lr{x_\mu}_\psi$, or in fact any other procedure $y_\mu = f(\lr{x_\mu}_\psi)$ which gives order $1$ answers. 
\end{itemize}
To study the structure of posteriors that are nonlinear in $x_\psi$ (and hence must violate \eqref{eq:linear-post-formal}), we make a very rough two parameter ansatz for the data-generating process. Specifically, we consider the SVD of the data matrix
\begin{align*}
    \frac{1}{\sqrt{N_0}}X_\psi=:\sum_{j=1}^P \sqrt{\lambda_j} u_j v_j^T
\end{align*}
and assume it has normalized power law spectrum
\begin{align}
    \label{eq:power-law}   \sum_{j=1}^P \lambda_j =1,\qquad \lambda_j\sim j^{-\alpha},\qquad \alpha >1.
\end{align}
The condition $\alpha >1$ means that the sample covariance matrix $ \frac{1}{N_0} X_\psi X_\psi^T$ concentrates around the population covariance $\Sigma$ at $N_0,P\gives \infty$ so the power law condition \eqref{eq:power-law} can equivalently be put directly on $\Sigma$. The restriction $\alpha >1$ also ensures that \eqref{eq:linear-post} is violated. We further suppose that the targets are generated according to 
\begin{align}
\label{eq:power-label}
Y = \lr{y_\mu,\, \mu=1,\ldots, P}^T = \sqrt{P}v_k.
\end{align}
In this way, we use the parameter $\alpha$ as a proxy for measuring correlation between different training inputs, while the parameter $k=1,\ldots, P$ controls the alignment between the inputs and outputs. 

As a first question about posteriors for such data-generating processes, we ask when Bayesian inference prefers deeper networks. To make sense of this, recall that the partition function $Z_\beta(0)$ is often called the Bayesian model evidence \cite{mackay1992bayesian}. It represents for a fixed model --- in our case specified by $L,N, \psi, \beta$ --- the likelihood of observing the training data $\mD$ under the posterior. Maximizing the Bayesian model evidence over the space of models is therefore equivalent to performing maximal likelihood and hence gives a principled Bayesian way to perform model selection. We find:
\begin{itemize}
    \item \textbf{When does depth increase evidence?}  In the case of deep linear networks $\psi = 0$ at zero temperature, the results in \cite{hanin2023bayesian} show that increasing the effective network depth $LP/N$ always increases the Bayesian model evidence. This effect is non-perturbative and holds for all values of $LP/N$. The interpretation of this is that deeper networks correspond to lower rank (and hence more parsimonious) priors. However, when $\psi\neq 0$, this is no longer the case. Specifically, we show in \S \ref{sec:gen-T-analysis} under some self-averaging assumptions (see \S \ref{sec:data-ass}) that the first order in $1/N$ correction to the Bayesian model evidence (at all temperatures) is positive for our data-generating processes if and only if
\begin{equation}\label{eq:nice-def}
\alpha < 2\qquad \text{and}\qquad k = o(P^{1/\alpha}).    
\end{equation}
Thus, depth increases model evidence at first order in $1/N$ if and only if the effective data dimension is not too low ($\alpha < 2$) and the inputs are well-enough aligned with the target ($k=o(P^{1/\alpha})$). See Proposition \ref{prop:evidence} and Corollary \ref{cor:nonlin-benign}.
\end{itemize}

\begin{figure}
    \centering
    \includegraphics[scale=.8]{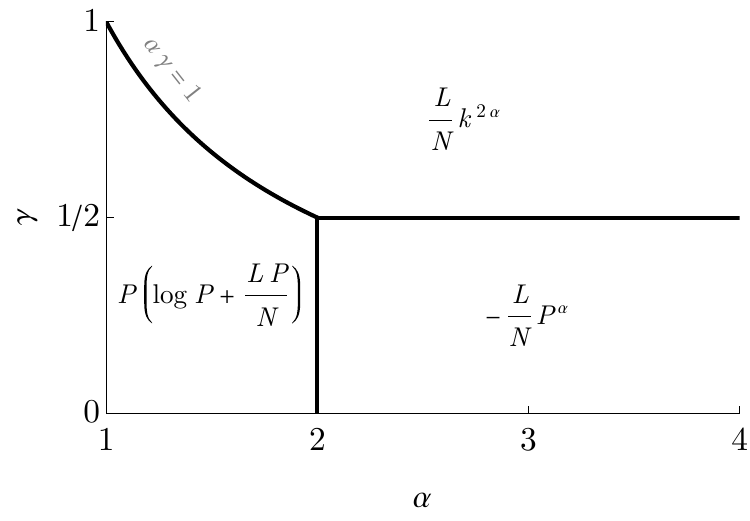}
    \caption{Phase diagram of the log evidence of a deep nonlinear network at zero temperature. The dataset covariance matrix has a power law spectrum $\lambda_j \sim j^{-\alpha}$ \eqref{eq:power-law} and the label vector lies in the $k$th direction \eqref{eq:power-label} for $k = P^\gamma$. The first-order in $1/N$ is perturbatively valid for $\gamma < 1/\alpha$ and $\alpha < 2$. Within the perturbative regime, depth improves the evidence; at the two boundaries of the regime, depth either increases or decreases the evidence. See Fig.~\ref{fig:finite-temp-phase} for the phase diagram at nonzero temperature.}
    \label{fig:zero-temp-phase}
\end{figure}

For the data models \eqref{eq:nice-def}, we then investigate how the network architecture interacts with the data-generating process to form the posterior.

\begin{itemize}

\item \textbf{Exact form of $LP/N$ posterior correction at zero temperature.} For the data generating processes satisfying \eqref{eq:nice-def} at zero temperature, the $LP/N$-correction to the posterior distribution of the network output $f(x;\theta)$ at a test point $x$ is simply obtained through shifting the posterior mean by 
\begin{equation}\label{eq:mean-shift-inf}
     \text{const}_\psi \times \frac{LP}{N}\times \sum_{\mu=1}^P y_\mu a_\mu \lr{x_\psi^T \Sigma x_\psi - \lr{x_\mu}_\psi^T \Sigma \lr{x_\mu}_\psi},
\end{equation}
where the coefficients $a_\mu$ are defined by
\[
\text{projection of }x_\psi\text{ onto }\mathrm{Span}\set{\lr{x_\mu}_\psi,\,\mu=1,\ldots, P} =:\sum_{\mu=1}^P a_\mu \lr{x_\mu}_\psi.
\]
A short computation provided  shows that the shift \eqref{eq:mean-shift-inf} this is equivalent to Bayesian inference with the data-dependent kernel method \eqref{eq:data-kernel}. See Proposition \ref{prop:post}.
 \end{itemize}

Finally, we examine the role of depth in benign overfitting.

\begin{itemize}
    
\item \textbf{Role of depth in benign overfitting of deep linear networks.} We introduce label noise to the data-generating process, choosing labels
\begin{align*}
    Y = \sqrt{P}v_k + \epsilon
\end{align*}
for $\epsilon \in \R^P$ with iid entries from $\mathcal{N}(0, \se^2)$. In the setting of a deep linear network, we show that the perturbative expansion in $1/N$ remains well-defined for all label noise of variance at most $\se^2 = O(P^{1-\alpha})$. Below this threshold, we show benign overfitting: the evidence is maximized and the generalization error is minimized at zero temperature. Moreover, below this threshold, depth improves the benign overfitting to first order in $1/N$:
\begin{align*}
    \frac{\partial(\text{evidence})}{\partial (LP/N)} > 0 \quad \text{and} \quad \frac{\partial(\text{generalization error})}{\partial (LP/N)} < 0.
\end{align*}
\end{itemize}

\subsubsection{Combinatorial Model for Prior Moments}\label{sec:prior-model}
To describe the basis of our technical contributions, we develop in this section a new combinatorial/probabilistic formalism for computing the moments 
\begin{equation}\label{eq:MqL-def}
    M_q^{(L)}:=\E{\lr{\frac{1}{N_\ell}\norm{X^{(L)}t}^2}^q},\qquad X^{(\ell)} t = \sum_{\mu=1}^P t^\mu x_\mu^{(\ell)}
\end{equation}
of the prior in the final hidden layer. To do this, let us write
\begin{align*}
    M_q^{(L)} &=\sum_{\substack{\omu=\lr{\mu_1,\ldots, \mu_q}\\ \onu=\lr{\nu_1,\ldots, \nu_q}\\ \mu_j,\, \nu_j=1,\ldots, P}} t^{\omu}t^{\onu} M_{\omu,\onu}^{(L)},
\end{align*}
where
\[
t^{\omu}=\prod_{p=1}^q t^{\mu_p},\qquad M_{\omu,\onu}^{(L)} = \E{\prod_{p=1}^q \frac{1}{N_\ell}\inprod{x_{\mu_p}^{(L)}}{x_{\nu_P}^{(L)}}}.
\]
The goal of this section is for fixed 
\[
\omu = \lr{\mu_p,\, p=1,\ldots,q},\, \onu=\lr{\nu_p,,\, p=1,\ldots, q},\qquad \mu_p,\nu_p\in \set{1,\ldots, P}
\]
to obtain a probabilistic representation for $M_{\omu,\onu}^{(L)}$, which replaces integration over random weights by integration over a collection of random graph, defined as follows:

\begin{definition}[Random Graph Process]\label{def:rand-graph}
Given $\omu,\onu$, we define a layerwise Markov Chain 
\[
    G_{\omu,\onu}^{(\ell)}=\lr{V_{\omu,\onu} ,E^{(\ell)}},\qquad \ell = L, L-1, \ldots, 0
\]
as follows:
    \begin{itemize}
        \item For all $\ell$ we set
        \[
            V_{\omu,\onu} = \lr{\gamma_p,\, \gamma\in \set{\mu, \nu}, \, p=1,\ldots, q}.
        \]
        \item  The initial condition is 
        \[
        E^{(L)}:=\set{(\mu_p,\nu_p),\, p=1,\ldots, q}.
        \]
        \item Given $E^{(\ell+1)}$, we produce $E^{(\ell)}$ as follows:
        \begin{itemize}
            \item Independently for each vertex $\gamma_p\in V$ we add to $E^{(\ell+1)}$ a self-loop $(\gamma_p,\gamma_p)$ with probability 
        \[
        \frac{\abs{\psi}}{L}\mathrm{deg}_{E^{(\ell+1)}}(\gamma_p)=\frac{\abs{\psi}}{L}\lr{1+2s_{\gamma}^{(\ell)}},\quad s_{\gamma_p}^{(\ell+1)}=\#\set{\text{self-loops at vertex }\gamma_p\text{ in }E^{(\ell+1)}}.
        \]
        We will refer to $s_{\gamma_p}^{(\ell)}$ as the self-loop process and will sometimes write
        \begin{equation}\label{eq:sell-deff}
            s^{(\ell)} := \sum_{\gamma=\mu,\nu} \sum_{p=1}^q s_{\gamma_p}^{(\ell)} =\text{total number of self-loops in }E^{(\ell)}.
        \end{equation}
        \item We define independent Bernoulli-random variables
        \[
        \delta^{(\ell)} = {\bf 1}\set{\text{edge shuffle occurs at layer }\ell}\sim \mathrm{Bernoulli}\lr{\frac{2}{N_\ell}\binom{|E^{(\ell+1)}|}{2}}.
        \]
        On the event that $\delta^{(\ell)}=1$, 
        we uniformly at random select a pair of edges $e_1=(\mu_{p_1},\nu_{p_1}),e_2=(\mu_{p_2},\nu_{p_2})\in E^{(\ell+1)}$, remove them from $E^{(\ell+1)}$ and recplace them instead by 
        \[
        \begin{cases} (\mu_{p_1}, \mu_{p_2}), (\nu_{p_1},\nu_{p_2})&\quad \text{ with probability }1/2\\
        (\mu_{p_1}, \nu_{p_2}), (\nu_{p_1},\mu_{p_2})&\quad \text{ with probability }1/2
        \end{cases},
        \]
        independent of the choice of edges and of whether $\delta^{(\ell)}$ occurs.
        \end{itemize}
        \end{itemize}
\end{definition}

Our main result about this combinatorial model is the following. 

\begin{proposition}\label{prop:M-rep}
    Fix $\omu, \onu$ and consider the random graph process $G^{(\ell)}=\lr{V_{\omu,\onu},E^{(\ell)}}$ from Definition \ref{def:rand-graph}. Define 
    \[
c^{(\ell)}:=\begin{cases}
    1,&\quad \ell  = L \\
    c^{(\ell+1)} \sigma^{2|E^{(\ell+1)}|}\lr{1+\frac{2\abs{\psi}|E^{(\ell+1)}|}{L}+\frac{2}{N_\ell}\binom{|E^{(\ell+1)}|}{2}},&\quad \ell < L
\end{cases}.
\]
    We then have
    \begin{align}
    \label{eq:M-rep}M_{\omu,\onu}^{(L)}&= \E{\sgn{\psi}^{s^{(0)}}c^{(0)}\prod_{e\in E^{(0)}} \w_e^{(0)}} + O\lr{\frac{L}{N^2}}+O\lr{\frac{1}{L}},
\end{align}
where $\sgn{\cdot}$ is the sign function, for any $\ell=0,\ldots, L$ we define the edge-weights by 
\[
e = \lr{\gamma_{p_1},\gamma_{p_2}'}\qquad \Longrightarrow\qquad w_e^{(\ell)}:=\frac{\sigma^2}{N_\ell}\inprod{x_{\gamma_{p_1}}^{(\ell)}}{x_{\gamma_{p_2}'}^{(\ell)}},
\]
and the expectation is over the random graph process.
\end{proposition}
We prove Proposition \ref{prop:M-rep} in \S \ref{sec:M-rep} below.

\subsection{Review of Literature}\label{sec:lit-rev}
Bayesian inference with neural networks has a long history, going back at least to the article \cite{neal1996priors}, which showed that one layer networks with random weights converge to Gaussian processes at infinite width. This subject was taken up and expanded upon in \cite{lee2017deep}, which considered the case of networks of any fixed depth. The existence and universality of Gaussian process priors at infinite width and finite depth were then established at a variety of levels of rigor and generality in articles such as \cite{garriga2018deep,hron2022wide, hanin2021random}. 

While the preceding articles sought to analyze neural networks at finite depth, the results in \cite{hanin2020products, hanin2018neural, yaida2019non, roberts2022principles, li2022neural, zhang2022deep} gave new information about the prior over neural network outputs in the regime where depth and width are simultaneously large. Two important ideas emerged from these works. First, when the network depth scales like a constant times the network width, the prior distribution over network outputs has a complex non-Gaussian structure. Second, in order for the two point function at be well-behaved at large depth, \cite{li2022neural,zhang2022deep} showed that one must consider nonlinearities close to the identity at every layer. This motivates the shaped nonlinearities we consider here (see \eqref{eq:phi-def}).

To help elucidate the nature of the priors computed in the regime where both depth and width are large, two independent articles \cite{zavatone2021exact,noci2021precise} gave an exact non-asymptotic characterization of these priors in the case of deep linear networks in terms of so-called Meijer G-functions. This special functions point of view was then taken up in \cite{hanin2023bayesian} to compute exact non-asymptotic formulas for the posterior and model evidence for Bayesian inference with deep linear networks with output dimension $1$ at zero temperature. It was this article which showed that while the depth-to-width (in this article denoted by $L/N$) ratio naturally plays the role of the ``effective depth of the prior'', it is actually $LP/N$ that plays the role of the ``effective depth of the posterior.'' We give partial corroboration in the present article that it is indeed $LP/N$ that plays this role for nonlinear networks as well. In addition to the work in \cite{hanin2023bayesian}, we point the reader to the excellent article \cite{li2022neural}, which computes at a physics level of rigor the predictive posterior in deep linear networks at finite depth, as well as the posterior of the intermediate hidden layer representations. The results for deep linear networks were then reproduced in the thermodynamic limit in \cite{zavatone2022contrasting}.

In addition to these articles on linear and shaped networks, several interesting lines of research have studied Bayesian inference with fully nonlinear neural networks at finite depth. These articles include \cite{seroussi2023separation, naveh2021self}, which obtain explicit but nonlinear saddle point equations for the posterior distribution of feature-feature covariance kernels at each hidden layer. Further, we also point the reader to \cite{cui2023optimal}, which uses the replica trick to make several precise and numerically accurate predictions for generalization error in both regression and classification settings. Finally, we mention the article \cite{ariosto2022statistical}. This work uses tools from statistical mechanics to connect Bayesian inference with finite depth neural networks in terms of student $t$-processes. It also proposes exact formulas  for the generalization error for Bayesian inference in a regression task with a one-hidden layer network in the regime where the number of training datapoints is a constant times the hidden layer width.

\subsection{Limitations}\label{sec:limitations}
We conclude by emphasizing what we view as the three main limitations of the present work.
\begin{itemize}
    \item \textbf{Perturbative vs. exact solutions.} The present article is able to compute only perturbatively  statistics of the predictive posterior in shaped networks. Obtaining higher order perturbative corrections and re-summing the resulting series is something we leave for future work. 
    \item \textbf{More general nonlinearities.} The present article considers only cubic shaped nonlinearities due to their natural appearance in tanh-like activation functions. While our techniques readily generalize the case of other monomials, they do not seem to extend in a simple manner to general shaped nonlinearities of the form $\phi(t) = t + \frac{1}{L}\sigma(t)$. Nor does this article consider more general reshaped nonlinearities such as the reshaped ReLU (see \cite{li2022neural}), which are not odd. 
    \item \textbf{Beyond linear overparameterization.} The results in the present article assume that the number of datapoints $P$ is strictly less than $N_0$. Without this assumption, it is not clear how to define a suitable analog for the minimal norm interpolant $\theta_*$ of the training data, which plays an important role in all our results. However, understanding the nature of Bayesian posteriors with $P > N_0$ and indeed $P = N_0^\gamma$ with $\gamma \geq 1$ is an open problem.
\end{itemize}

\section{Computing the Prior: Graphical Model}

\subsection{Derivation of Proposition \ref{prop:M-rep}}\label{sec:M-rep}
The starting point for deriving a representation of the form \eqref{eq:M-rep} is to note that the graph 
\[
G^{(L)} = \lr{V_{\omu,\onu},E^{(L)}},\qquad 
E^{(L)} 
\]
is deterministic and that we have
\begin{equation}
    M_{\omu,\onu}^{(L)} = \E{\prod_{e\in E^{(L)}} w_e^{(L)}}.
\end{equation}
We can now obtain the representation \eqref{eq:M-rep} by integrating out network weights recursively in the layer index. Since the argument works the same way for each layer, let us explain in detail how to integration out the weights $W^{(L)}$ between layers $L-1$ and $L$. We have
\begin{align*}
    \Esub{W^{(L)}}{\prod_{e\in E^{(L)}} w_e^{(L)}}&=\Esub{W^{(L)}}{\prod_{p=1}^q \frac{1}{N_L}\inprod{x_{\mu_p}^{(L)}}{x_{\nu_p}^{(L)}}}\\
    &=\mathbb E_{W^{(L)}}\bigg[\frac{1}{N_L^q}\sum_{k_1,\ldots, k_q=1}^{N_L} \prod_{p=1}^q \lr{W_{k_p}^{(L)}x_{\mu_p}^{(L-1)}+\frac{\psi}{3L}\lr{W_{k_p}^{(L)}x_{\mu_p}^{(L-1)}}^{3}} \\
    &\qquad \qquad \qquad \qquad \qquad \quad \times \lr{W_{k_p}^{(L)}x_{\mu_p}^{(L-1)}+\frac{\psi}{3L}\lr{W_{k_p}^{(L)}x_{\mu_p}^{(L-1)}}^{3}} \bigg],
\end{align*}
where we have written $W_k^{(L)}$ for the $k$-th row of the matrix $W^{(L)}$. Wick contracting the weights, which we recall have mean $0$ and variance $\sigma^2/N_{L-1}$, yields 
\begin{align}
    \label{eq:triv}\Esub{W^{(L)}}{\prod_{e\in E^{(L)}} w_e^{(L)}}&=\left\{\prod_{p=1}^q \frac{\sigma^2}{N_{L-1}} \inprod{x_{\mu_p}^{(L-1)}}{x_{\nu_p}^{(L-1)}}\right\}\lr{1 + \frac{1}{N_{\ell}}\binom{|E^{(L)}|}{2}}\\
    \label{eq:sl} &+\frac{2\psi}{L}\prod_{\substack{p=1}}^q \frac{\sigma^2}{N_{L-1}} \inprod{x_{\mu_p}^{(L-1)}}{x_{\nu_p}^{(L-1)}}\sum_{p'=1}^q \frac{\sigma^2}{2N_{L-1}}\lr{\norm{x_{\mu_{p'}}^{(L-1)}}^2+\norm{x_{\nu_{p'}}^{(L-1)}}^2} \\
    \notag &+\frac{2}{N_{L}}\sum_{\substack{p_1',p_2'=1\\ p_1'< p_2'}}^q \prod_{\substack{p=1\\ p\neq p_1',p_2'}}^q \frac{\sigma^2}{N_{L-1}}\inprod{x_{\mu_{p}}^{(L-1)}}{x_{\nu_p}^{(L-1)}}\\
    \notag &\qquad  \qquad \qquad \times \frac{\sigma^4}{2N_{L-1}^2}\bigg(\inprod{x_{\mu_{p_1'}}^{(L-1)}}{x_{\mu_{p_2'}}^{(L-1)}}\inprod{x_{\nu_{p_1'}}^{(L-1)}}{x_{\nu_{p_2'}}^{(L-1)}}\\
    \label{eq:shuf} &\qquad \qquad \qquad \qquad \qquad+ \inprod{x_{\mu_{p_1'}}^{(L-1)}}{x_{\nu_{p_2'}}^{(L-1)}}\inprod{x_{\nu_{p_1'}}^{(L-1)}}{x_{\mu_{p_2'}}^{(L-1)}}\bigg)\\
    \notag &+ O\lr{N^{-2}+L^{-2}}.
\end{align}
The key observation is that each of the terms \eqref{eq:triv}, \eqref{eq:sl}, and \eqref{eq:shuf} can be re-written in terms of products of edge-weights over a graph $G= (V,E^{(L-1)})$. For example, for the term from \eqref{eq:triv}, we find
\begin{align*}
    \prod_{p=1}^q \frac{1}{N_{L-1}} \inprod{x_{\mu_p}^{(L-1)}}{x_{\nu_p}^{(L-1)}} &= \prod_{e\in E_0^{(L-1)}} w_e^{(L-1)},
\end{align*}
where we have set
\[
E_0^{(L-1)} = E^{(L)}.
\]
Similarly, inspecting the expression in \eqref{eq:sl}, we see that for each $p'=1,\ldots, q$, we have
\begin{align*}
    &\left[\prod_{\substack{p=1}}^q \frac{1}{N_{L-1}} \inprod{x_{\mu_p}^{(L-1)}}{x_{\nu_p}^{(L-1)}}\right]
    \times \frac{1}{2N_{L-1}}\lr{\norm{x_{\mu_{p'}}}^2 +\norm{x_{\nu_{p'}}}^2}\\
    &\qquad = \frac{1}{2}\lr{\prod_{e\in E_{\mu, p'}^{(L-1)}}w_e^{(L-1)} + \prod_{e\in E_{\nu, p'}^{(L-1)}}w_e^{(L-1)}},
\end{align*}
where 
\begin{align*}
E_{\gamma, p'}^{(L-1)} &:= E^{(L)}\cup \set{(\gamma_{p'},\gamma_{p'})},\qquad \gamma \in \set{\mu,\nu}
\end{align*}
is obtained from $E^{(L)}$ by adding a self-loop at either the vertex $\gamma_{p'}$. Finally, for each $p_1'<p_2'=1,\ldots, q$ we have
\begin{align*}
    &\left[\prod_{\substack{p=1\\ p\neq p_1',p_2'}}^q \frac{1}{N_{L-1}}\inprod{x_{\mu_{p}}^{(L-1)}}{x_{\nu_p}^{(L-1)}}\right]\times \frac{1}{2N_{L-1}^2}\bigg(\inprod{x_{\mu_{p_1'}}^{(L-1)}}{x_{\mu_{p_2'}}^{(L-1)}}\inprod{x_{\nu_{p_1'}}^{(L-1)}}{x_{\nu_{p_2'}}^{(L-1)}}\\
    &\qquad\qquad \qquad \qquad \qquad \qquad\qquad \qquad \qquad \qquad + \inprod{x_{\mu_{p_1'}}^{(L-1)}}{x_{\nu_{p_2'}}^{(L-1)}}\inprod{x_{\nu_{p_1'}}^{(L-1)}}{x_{\mu_{p_2'}}^{(L-1)}}\bigg)\\
    &= \frac{1}{2}\lr{\prod_{e\in E_{(\mu_1,\mu_2),(\nu_1,\nu_2), p_1',p_2'}^{(L-1)}}w_e^{(L-1)}+\prod_{e\in E_{(\mu_1,\nu_2),(\mu_2,\nu_1), p_1',p_2'}^{(L-1)}}w_e^{(L-1)}},
\end{align*}
where for a pair-partition $\pi$ of the set $\set{\mu_{p_1}, \mu_{p_2}, \nu_{p_1},\nu_{p_2}}$ the edge-set $E_{\pi, p_1,p_2}^{(L-1)} $ is obtained from $E^{(L)}$ by removing the edges $(x_{\mu_{p_1}}, x_{\nu_{p_1}}),\, (x_{\mu_{p_2}},x_{\nu_{p_2}})$ and ``re-wiring'' their endpoints according to $\pi$. Putting this all together yields the representation
\begin{align}
    \notag \sigma^{-2|E^{(\ell+1)}|} \Esub{W^{(L)}}{\prod_{e\in E^{(L)}} w_e^{(L)}}&= \left\{\prod_{e\in E_0^{(L-1)}} w_e^{(L-1)} \right\}\lr{1+\frac{2}{N_L}\binom{|E^{(L)}|}{2}}\\
    \notag &+ \frac{\psi}{L}\times \sum_{p=1}^q \sum_{\gamma\in \set{\mu, \nu}} \prod_{e\in E_{\gamma, p}^{(L-1)}} w_e^{(L-1)} \\
    \notag &+ \frac{2\binom{q}{2}}{N_{L}} \times \frac{1}{\binom{q}{2}}\sum_{\substack{p_1,p_2=1\\ p_1 < p_2}}^q \frac{1}{2} \sum_{\pi} \prod_{e\in E_{\pi, p_1,p_2}^{(L-1)}} w_e^{(L-1)}\\
    \label{eq:M-comb}&+O(N^{-2})+O(L^{-2}).
\end{align}
The sum over $\pi$ in the second to last line runs over the two non-trivial pair-partitions of $\set{\mu_1, \mu_2, \nu_1,\nu_2}$. To recast this combinatorial formula into an expectation over random graphs let us recall that
\[
c^{(L-1)}=\sigma^{2|E^{(L)}|}\lr{1+\frac{2\abs{\psi}|E^{(L)}|}{L}+\frac{2}{N_L}\binom{|E^{(L)}|}{2}}.
\]
The expression \eqref{eq:M-comb} then admits the representation
\begin{align*}
    \Esub{W^{(L)}}{\prod_{e\in E^{(L)}} w_e^{(L)}} &= \Esub{E^{(L-1)}}{\sgn{\psi}^{{\bf 1}\set{\text{self-loop in }E^{(L-1)}}}c^{(L)}\prod_{e\in E^{(L-1)}} \w_e^{(L-1)}}\\
    &+ O\lr{N^{-2}+L^{-2}},
\end{align*}
where this expectation is now over the random edge-set $E^{(L-1)}$ as in Definition \ref{def:rand-graph}. Repeating this derivation at each layer noticing that the $O(N^{-2}+ L^{-2})$ error term gets added $L$ times in this procedure completes the derivation. \hfill $\square$

\subsection{Combinatorial Model}

In the previous section, we rewrote in Proposition \ref{prop:M-rep} the prior moments $M_{\omu,\onu}^{(L)}$ in terms of expectations over the random graph process in Definition \eqref{def:rand-graph}. The purpose of the present section is to simplify these expectations over graphs to contain only expectations over the self-loop process 
\[
s_{\gamma_p}^{(\ell)} = \#\set{\text{self-loops at vertex }\gamma_{p}\text{ in }E^{(\ell)}},
\]
which was introduced as part of Definition \ref{def:rand-graph}. The main result the following
\begin{proposition}\label{prop:Mq-expand}
    Write
    \begin{equation}\label{eq:Tij-def}
        T_{ij}^{(\ell)} = \diag\lr{\E{c_{\mu}^{(0)}\lr{s_{\mu}^{(\ell)}}^i\lr{\frac{\sigma^2\norm{x_{\mu}}^2}{N_0}}^{s_{\mu}^{(0)}-j}},\, \mu=1,\ldots, P},
    \end{equation}
    where
\begin{equation}\label{eq:cmu-def}
        c_\mu^{(0)}:=\sgn{\psi}^{s_{\mu}^{(0)}}\exp\left[\frac{2(\eta+\abs{\psi})}{L}\sum_{\ell=1}^Ls_\mu^{(\ell)}\right].
    \end{equation}
    Then, 
    \begin{align}
    &\notag e^{-2q(\abs{\psi}+\eta)}M_{q}^{(L)}\\
    &\notag \quad =\sum_{\ell=1}^L \frac{1}{N_\ell}\bigg\{  \mathbb E\bigg[\lr{\frac{1}{N_0}\norm{X T_{0,0}t}^2}^q\bigg] \\
    \notag &\qquad \qquad\quad  \quad +2q\lr{\frac{1}{N_0}\norm{XT_{0,0}t}^2}^{q-1} \frac{1}{N_0^3}\norm{X^{\otimes 3}T_{1,1}^{(\ell)}t}^2\\
    \notag &\qquad \qquad \quad \quad +4q\lr{\frac{1}{N_0}\norm{XT_{0,0}t}^2}^{q-1}\frac{1}{N_0}\inprod{XT_{0,0}t}{XT_{2,0}^{(\ell)}t}\\
     \notag &\qquad \qquad \quad \quad +4q(q-1)\lr{\frac{1}{N_0}\norm{XT_{0,0}t}^2}^{q-2} \frac{1}{N_0^4}\inprod{XT_{0,0}t\otimes X^{\otimes 3} T_{1,1}^{(\ell)}t}{X^{\otimes 3} T_{1,1}^{(\ell)}t\otimes XT_{0,0}t}\\
     \notag &\qquad \qquad \quad \quad +2q(q-1)\lr{\frac{1}{N_0}\norm{XT_{0,0}t}^2}^{q-2} \frac{1}{N_0^3}\inprod{X^{\otimes 3}T_{1,1}^{(\ell)}t}{XT_{0,0}t\otimes XT_{0,0}t\otimes XT_{0,0}t}\\
     \label{eq:Mq-expand}&\quad \qquad \qquad \quad +q(q-1)\lr{\frac{1}{N_0}\norm{XT_{0,0}t}^2}^q\bigg\},
    \end{align}
    plus error terms of size $O(L/N^2) + O(1/L)$.
\end{proposition}
We derive the expression \eqref{eq:Mq-expand} in \S \ref{sec:Mq-expand} below. Before doing so, we record the following immediate Corollary.
\begin{corollary}\label{cor:prior-moments}
    Up to errors of size $O(L/N^2)+O(1/L)$ we have
    \begin{align}
       \notag &\E{\exp\left[-\frac{\sigma^2}{2N_L}\norm{X^{(L)}t}^2\right]}\\
        \notag&\quad = \exp\left[-e^{2(\abs{\psi} + \eta)}\frac{1}{2N_0}\norm{XT_{0,0}t}^2\right]\\
       \notag &\qquad \times  \bigg\{ 1-\sum_{\ell=1}^L \frac{e^{2(\abs{\psi}+\eta)}}{N_\ell}\lr{\frac{1}{N_0^3}\norm{X^{\otimes 3}T_{1,1}^{(\ell)}t}^2+\frac{2}{N_0}\inprod{XT_{0,0}t}{XT_{2,0}^{(\ell)}t}}\\
       \notag &\qquad \qquad \,\, +\sum_{\ell=1}^L\frac{e^{4(\abs{\psi} + \eta)}}{N_\ell}\bigg(\frac{1}{2N_0^3}\inprod{X^{\otimes 3}T_{1,1}^{(\ell)}t}{\lr{XT_{0,0}t}^{\otimes 3}}+\lr{\frac{1}{2N_0}\norm{XT_{0,0}t}^2}^2\\
       \label{eq:prior-MGF} &\qquad \qquad \qquad \qquad \qquad\qquad  +\frac{1}{N_0^4}\inprod{XT_{0,0}t\otimes X^{\otimes 3}T_{1,1}^{(\ell)}t}{X^{\otimes 3}T_{1,1}^{(\ell)}t\otimes XT_{0,0}t}\bigg)\bigg\}.
    \end{align}
\end{corollary}
The corollary gives a formula for the Laplace transform of $\norm{X^{(L)}t}^2$ under the prior distribution and is a key step in our computation of the partition function. After evaluating the $T$ matrices, we will see an immediate consequence of this result in Corollary~\ref{cor:prior-moments-hat}; the partition function is then derived in \S \ref{sec:part-exp}.

\subsubsection{Derivation of Proposition \ref{prop:Mq-expand}}\label{sec:Mq-expand}
Recall the indicator random variables
\[
\delta^{(\ell)} = {\bf 1}\set{\text{edge shuffle occurs at layer $\ell$}}
\]
for which we had
\[
\mathbb P\lr{\delta^{(\ell)}=1} = \binom{|E^{(\ell+1)}|}{2}\frac{1}{N_\ell}=O(N^{-1}).
\]
Due to this $1/N$ factor we have
\begin{align*}
    M_q^{(L)} &=M_{q,0}^{(L)}+M_{q,1}^{(L)}+O(N^{-1}+L^{-1})\\
    &:= \sum_{\substack{\omu=\lr{\mu_1,\ldots, \mu_q}\\ \onu=\lr{\nu_1,\ldots, \nu_q}\\ \mu_j,\, \nu_j=1,\ldots, P}} t^{\omu}t^{\onu}\E{{\bf 1}\set{\delta^{(\ell)}=0,\, \ell=1,\ldots, L}\sgn{\psi}^{s^{(0)}}c^{(0)}\prod_{e\in E^{(0)}} \w_e^{(0)}}\\
    &+\sum_{\substack{\omu=\lr{\mu_1,\ldots, \mu_q}\\ \onu=\lr{\nu_1,\ldots, \nu_q}\\ \mu_j,\, \nu_j=1,\ldots, P}} t^{\omu}t^{\onu}\sum_{\ell=1}^L  \E{{\bf 1}\set{\delta^{(\ell)}=1,\, \delta^{(\ell')}=0,\, \ell'\neq \ell} \sgn{\psi}^{s^{(0)}}\prod_{e\in E^{(0)}}\w_e^{(0)}}\\
    &+ O\lr{\frac{L}{N^2}}+O\lr{\frac{1}{L}},
\end{align*}
where the subscripts $0,1$ in $M_{q,0}^{(L)},M_{q,1}^{(L)}$ record whether there are $0$ or $1$ edge shuffles that occur in the course of generating $E^{(L-1)},\ldots, E^{(0)}$ and we recall the  abbreviation
\[
    s^{(\ell)} := \sum_{\gamma\in \set{\mu,\nu}}\sum_{p=1}^q s_{\gamma_p}^{(\ell)}=\#\set{\text{self-loops in }E^{(\ell)}},\qquad \ell=0,1,\ldots, L.
\]
Our first goal is to rewrite $M_{q,0}$ in terms of expectations only over the self-loop process. For this, note that conditional both on the self-loop process (i.e., on the values of all $s_{\gamma_p}^{(\ell)}$) and the event that $\delta^{(\ell)}=0$ for all $\ell$, we have
\[
    \prod_{e\in E^{(0)}} w^{(0)}(e) = \prod_{p=1}^q \frac{\sigma^2}{N_0}\inprod{\wx_{\mu_p}}{\wx_{\nu_p}},\qquad \wx_{\gamma_p} :=\lr{\frac{\sigma^2\norm{x_{\gamma_p}}^2}{N_0}}^{s_{\gamma_p}^{(0)}}x_{\gamma_p}.
\]
and
\begin{align*}
    \sgn{\psi}^{s^{(0)}}c^{(0)}&= e^{2q(\eta + \abs{\psi})}\lr{\prod_{p=1}^q \prod_{\gamma \in \set{\mu,\nu}} c_{\gamma_p}^{(0)}}\lr{1+\frac{1}{N_{\ell}}\sum_{\ell=1}^L\lr{q+s^{(\ell)}}\lr{q-1+s^{(\ell)}}} \\
    &+O (N^{-2}) + O(L^{-1}),
\end{align*}
where
\begin{align*}
    c_{\gamma_p}^{(0)} = \sgn{\psi}^{s_{\gamma_p}^{(0)}} \exp\left[(\abs{\psi}+\eta)\frac{2}{L}\sum_{\ell=1}^L s_{\gamma_p}^{(\ell)}\right].
\end{align*}
Hence, since
\begin{align*}
&\mathbb P\lr{{\bf 1}\set{\delta^{(\ell)}=0,\, \ell=1,\ldots, L}~\bigg|~s_{\gamma_p}^{(\ell')},\, \gamma=\mu,\nu,\, p=1,\ldots,q,\,\ell' = 1,\ldots, L} \\&\qquad = \prod_{\ell=1}^L \lr{1-\frac{1}{N_\ell}\lr{q+s^{(\ell)}}\lr{q-1+s^{(\ell)}}}    
\end{align*}
we obtain 
\begin{align*}
   M_{q,0}^{(L)}&= e^{2q(\eta + \abs{\psi})}\sum_{\substack{\omu=\lr{\mu_1,\ldots, \mu_q}\\ \onu=\lr{\nu_1,\ldots, \nu_q}\\ \mu_j,\, \nu_j=1,\ldots, P}} t^{\omu}t^{\onu} \mathbb E\bigg[\prod_{p=1}^q c_{\mu_p}^{(0)}c_{\nu_p}^{(0)}\frac{1}{N_0}\inprod{\wx_{\mu_p}}{\wx_{\nu_p}}\bigg] + O\lr{\frac{1}{L}}+ O\lr{\frac{L}{N^2}}.
\end{align*}
Using that the self-loop processes are independent at different vertices we may now sum over the pre-factors $t^{\omu}t^{\onu}$ in the definition of $M_{q,0}$ to arrive at the expression
\begin{align*}
   M_{q,0}^{(L)}&= e^{2q(\eta+\abs{\psi})}\mathbb E\bigg[\lr{\frac{1}{N_0}\norm{X T_{0,0}t}^2}^q\bigg] + O\lr{\frac{1}{L}}+ O\lr{\frac{L}{N^2}},
\end{align*}
in which the expectation is only over the self-loop process. This is the first term in \eqref{eq:Mq-expand}. Next, we turn to computing $M_{q,1}^{(L)}$, which to leading order equals
\begin{equation}\label{eq:one-shuf}
e^{2q(\eta +\abs{\psi})}\sum_{\substack{\omu=\lr{\mu_1,\ldots, \mu_q}\\ \onu=\lr{\nu_1,\ldots, \nu_q}\\ \mu_j,\, \nu_j=1,\ldots, P}} t^{\omu}t^{\onu}\sum_{\ell=1}^L  \E{{\bf 1 }\set{\delta^{(\ell)}=1}\cdot \prod_{\gamma\in \set{\mu,\nu}}\prod_{p=1}^q c_{\gamma_p}^{(0)}\prod_{e\in E^{(0)}} \w^{(0)}(e)}.    
\end{equation}
The first step is to decompose the shuffle event into cases, formalized in the following
\begin{lemma}
   In the expression \eqref{eq:one-shuf}
   we may substitute 
   \begin{align*}
    {\bf 1}\set{\delta^{(\ell)} =1}&=   q{\bf 1}\set{\delta_{S_{\mu_1}, S_{\nu_1}}^{(\ell)}=1}+2q\lr{{\bf 1}\set{\delta_{S_{\mu_1}, S_{\mu_1}}^{(\ell)}=1}+{\bf 1}\set{\delta_{S_{\mu_1}, (\mu_1, \nu_1)}^{(\ell)}}=1}\\
     &  +\frac{1}{2}q(q-1)\lr{4\cdot {\bf 1}\set{\delta_{S_{\mu_1}, S_{\mu_2}}^{(\ell)}=1}+2\cdot {\bf 1}\set{\delta_{S_{\mu_1}, (\mu_2,\nu_2)}^{(\ell)}=1}+{\bf 1}\set{\delta_{(\mu_{1},\nu_1), (\mu_2, \nu_2)}=1}}
\end{align*}
where for instance $\delta_{S_{\mu_1},S_{\mu_2}}^{(\ell)}$ is the event that a shuffle between two edges in $E^{(\ell)}$, which are self-loops at vertices $x_{\mu_1}$ and $x_{\mu_2}$, happens at layer $\ell$, which $\delta_{S_{\mu_1}, (\mu_2,\nu_2)}^{(\ell)}$ is the event that a shuffle between two edges in $E^{(\ell)}$, which are a self-loop at $x_{\mu_1}$ and an edge $(x_{\mu_2},x_{\mu_2})$, happens at layer $\ell$.
\end{lemma}
\begin{proof}
Consider the more granular decomposition of the shuffle event:
\begin{align*}
     \delta^{(\ell)} &=  \sum_{p=1}^q \sum_{\gamma\in \set{\mu, \nu}}\delta_{S_{\gamma_p}, S_{\gamma_p}}^{(\ell)}\\
     &+\sum_{p=1}^q \delta_{S_{\mu_p}, S_{\nu_p}}^{(\ell)}\\
     &+\sum_{p=1}^q \delta_{S_{\mu_p}, (\mu_p, \nu_p)}^{(\ell)}+\delta_{S_{\nu_p}, (\mu_p, \nu_p)}^{(\ell)}\\
     &+ \sum_{\substack{p_1,p_2=1\\ p_1\neq p_2}}^q\sum_{\gamma_1,\gamma_2\in \set{\mu,\nu}}\delta_{S_{(\gamma_1)_{p_1}}, S_{(\gamma_2)_{p_2}}}^{(\ell)}\\
     &+\sum_{\substack{p_1,p_2=1\\ p_1\neq p_2}}^q \sum_{\gamma \in \set{\mu,\nu}} \delta_{S_{\gamma_{p_1}}, (\mu_{p_2}, \nu_{p_2})}^{(\ell)}+\delta_{S_{\gamma_{p_2}}, (\mu_{p_1}, \nu_{p_1})}^{(\ell)}\\
     &+\sum_{\substack{p_1,p_2=1\\p_1\neq p_2}}^q \delta_{(\mu_{p_1},\nu_{p_1}),(\mu_{p_2}, \nu_{p_2})}^{(\ell)}
\end{align*}
By symmetry over the $p=1,\ldots q$ of between the $\mu,\nu$ indices, we may therefore replace
\begin{align*}
     {\bf 1}\set{\delta^{(\ell)} =1}&=   q{\bf 1}\set{\delta_{S_{\mu_1}, S_{\nu_1}}^{(\ell)}=1}+2q\lr{{\bf 1}\set{\delta_{S_{\mu_1}, S_{\mu_1}}^{(\ell)}=1}+{\bf 1}\set{\delta_{S_{\mu_1}, (\mu_1, \nu_1)}^{(\ell)}}=1}\\
     &  +\frac{1}{2}q(q-1)\lr{4\cdot {\bf 1}\set{\delta_{S_{\mu_1}, S_{\mu_2}}^{(\ell)}=1}+2\cdot {\bf 1}\set{\delta_{S_{\mu_1}, (\mu_2,\nu_2)}^{(\ell)}=1}+{\bf 1}\set{\delta_{(\mu_{1},\nu_1), (\mu_2, \nu_2)}=1}}
\end{align*}
as desired.
\end{proof}
Substituting the result in the previous Lemma into \eqref{eq:one-shuf} yields
\begin{align*}
M_{q,1}^{(L)}&=q M_{(S_{\mu_1},S_{\nu_1})}^{(L)}\\
&+2q\lr{M_{(S_{\mu_1},S_{\mu_1})}^{(L)}+ M_{S_{\mu_1},(\mu_1,\nu_1)}^{(L)}}\\
&+\frac{q(q-1)}{2}\lr{4M_{(S_{\mu_1},S_{\mu_2})}+2M_{(S_{\mu_1}, (\mu_2,\nu_2))}+M_{(\mu_1,\nu_1),(\mu_2,\nu_2)}}\\
&+O(N^{-1}),
\end{align*}
where
\[
M_{*,**}^{(L)} :=e^{2q( \eta+\abs{\psi})}\sum_{\substack{\omu=\lr{\mu_1,\ldots, \mu_q}\\ \onu=\lr{\nu_1,\ldots, \nu_q}\\ \mu_j,\, \nu_j=1,\ldots, P}} t^{\omu}t^{\onu}\sum_{\ell=1}^L  \E{{\bf 1 }\set{\delta_{*, **}^{(\ell)}=1}\cdot \prod_{\gamma\in \set{\mu,\nu}}\prod_{p=1}^q c_{\gamma_p}^{(0)}\prod_{e\in E^{(0)}} \w^{(0)}(e)}.
\]
To complete the derivation it now remains to compute each of the expectations in the line above for all the relevant values of $*$ and $**$. Since they are all very similar, we present the details only in the case where 
\[
* = S_{\mu_1},\quad ** = S_{\nu_1}.
\]
In this case, for each $\ell$, we find conditional on the self-loop process and on $\delta_{S_{\mu_1},S_{\nu_1}}^{(\ell)}$ that 
\begin{align*}
    \prod_{e\in E^{(0)}} w_e^{(0)} &= \lr{\frac{\sigma^2}{N_0}\inprod{x_{\mu_1}}{x_{\nu_1}}}^3 \lr{\frac{\sigma^2\norm{x_{\mu_1}}^2}{N_0}}^{s_{\mu_1}^{(0)}-1}\lr{\frac{\sigma^2\norm{x_{\nu_1}}^2}{N_0}}^{s_{\nu_1}^{(0)}-1}\prod_{p=2}^q \frac{\sigma^2}{N_0}\inprod{\wx_{\mu_p}}{\wx_{\nu_p}}.
\end{align*}
Further, 
\begin{align*}
    \mathbb P\lr{\delta_{S_{\mu_1}, S_{\nu_1}}^{(\ell)}=1~\bigg|~s_{\gamma_p}^{(\ell)}} = \frac{2}{N_\ell}s_{\mu_1}^{(\ell)}s_{\nu_1}^{(\ell)}.
\end{align*}
Thus, we have
\begin{align*}
    &e^{-2q(\abs{\psi} + \eta)}M_{S_{\mu_1},S_{\nu_1}}^{(L)}\\
    &\quad =\sum_{\substack{\omu=\lr{\mu_1,\ldots, \mu_q}\\ \onu=\lr{\nu_1,\ldots, \nu_q}\\ \mu_j,\, \nu_j=1,\ldots, P}} t^{\omu}t^{\onu}\frac{2}{N_\ell}\sum_{\ell=1}^L\mathbb E\left[c_{\mu_1}^{(0)}c_{\nu_1}^{(0)}s_{\mu_1}^{(\ell)}s_{\nu_1}^{(\ell)}\lr{\frac{\sigma^2}{N_0}\inprod{x_{\mu_1}}{x_{\nu_1}}}^3 \lr{\frac{\sigma^2\norm{x_{\mu_1}}^2}{N_0}}^{s_{\mu_1}^{(0)}-1}\lr{\frac{\sigma^2\norm{x_{\nu_1}}^2}{N_0}}^{s_{\nu_1}^{(0)}-1}\right.\\
    &\quad \qquad \qquad \qquad \qquad \qquad\qquad\qquad \times \left.\prod_{p=2}^q c_{\mu_p}^{(0)}c_{\nu_p}^{(0)}\frac{\sigma^2}{N_0}\inprod{\wx_{\mu_p}}{\wx_{\nu_p}}\right].
\end{align*}
Recalling that the self-loop processes are independent at different vertices and summing over $\mu_j,\nu_j,\, j\geq 2$ reveals
\begin{align*}
    &e^{-2q(\abs{\psi} + \eta)}M_{S_{\mu_1},S_{\nu_1}}^{(L)}\\
    &\quad =\frac{2}{N_\ell}\sum_{\ell=1}^L\sum_{\mu_1,\nu_1=1}^P t^{\mu_1}t^{\nu_1}\mathbb E\left[c_{\mu_1}^{(0)}c_{\nu_1}^{(0)}s_{\mu_1}^{(\ell)}s_{\nu_1}^{(\ell)}\lr{\frac{\sigma^2}{N_0}\inprod{x_{\mu_1}}{x_{\nu_1}}}^3 \lr{\frac{\sigma^2\norm{x_{\mu_1}}^2}{N_0}}^{s_{\mu_1}^{(0)}-1}\lr{\frac{\sigma^2\norm{x_{\nu_1}}^2}{N_0}}^{s_{\nu_1}^{(0)}-1}\right]\\
    &\qquad \qquad \times \E{\lr{\frac{\sigma^2}{N_0}\norm{X T_{0,0}t}^2}^{q-1}}.
\end{align*}
Next, note that 
\begin{align*}
    \inprod{x_{\mu_1}}{x_{\nu_1}}^3 = \inprod{x_{\mu_1}^{\otimes 3}}{x_{\nu_1}^{\otimes 3}}.
\end{align*}
Hence,
\begin{align*}
    M_{S_{\mu_1},S_{\nu_1}}^{(L)} &=e^{2q(\abs{\psi} + \eta)}\frac{2}{N_\ell}\sum_{\ell=1}^L\frac{1}{N_0^3}\norm{X^{\otimes 3}T_{1,1}^{(\ell)}t}^2\cdot  \E{\lr{\frac{1}{N_0}\norm{X T_{0,0}t}^2}^{q-1}}.
\end{align*}
This is the second term in the expansion \eqref{eq:Mq-expand}. The remaining values of $*$ and $**$ give the other terms. \hfill $\square$


\section{Computing the Prior: Statistics of the Self-Loop Process}

The purpose of this section is to compute certain expectations over the self-loops $s_\mu^{(\ell)}$ appearing in the definition of the random graph process of Definition \ref{def:rand-graph} that come up in computing the partition function. 
\begin{proposition}\label{prop:sl-ints}
    For each $\mu$, we have
    \begin{align}
        \label{eq:t00}\E{c_\mu^{(0)} \lr{\frac{\sigma^2}{N_0}\norm{x_\mu}^2}^{s_\mu^{(0)}}} &=e^{-\abs{\psi}}\lr{1-2\wpsi_\mu }^{-1/2}\\
        \label{eq:t11a}\E{c_\mu^{(0)} \lr{\frac{\sigma^2}{N_0}\norm{x_\mu}^2}^{s_\mu^{(0)}-1}s_\mu^{(\ell)}} &=\widehat{\psi} e^{-\abs{\psi}}\lr{1-2\wpsi_\mu }^{-3/2}g_3(\eta,\ell/L)\\
        \label{eq:t11b}\E{c_\mu^{(0)} \lr{\frac{\sigma^2}{N_0}\norm{x_\mu}^2}^{s_\mu^{(0)}-1}\frac{1}{L}\sum_{\ell=1}^Ls_\mu^{(\ell)}} &=\wpsi e^{-\abs{\psi}}\lr{1-2\wpsi_\mu }^{-3/2}g_1(\eta)\\
        &=\psi e^{-\abs{\psi}}\frac{(2\eta-1)e^{2\eta}+1}{4\eta^2}\lr{1-2\wpsi_\mu}^{-3/2}\\
        \label{eq:t20}\E{c_\mu^{(0)} \lr{\frac{\sigma^2}{N_0}\norm{x_\mu}^2}^{s_\mu^{(0)}}\frac{1}{L}\sum_{\ell=1}^L \lr{s_\mu^{(\ell)}}^2} &=\wpsi_\mu e^{-\abs{\psi}}\lr{1-2\wpsi_\mu }^{-5/2}\left[\lr{1+\wpsi_\mu}g_1(\eta)-3\wpsi_\mu g_2(\eta)\right]
    \end{align}
    where
    \begin{align*}
        \wpsi =\psi \frac{\exp[2\eta]-1}{2\eta}, \quad \wpsi_\mu = \wpsi \frac{\norm{x_\mu}^2}{N_0}, \quad \sigma^2 =1+\frac{2\eta}{L}
    \end{align*}
    and
    \begin{align*}
        g_1(\eta) =\frac{1}{2}\lr{1+\coth(\eta)-\frac{1}{\eta}}, \quad g_2(\eta) =\frac{1}{4}\lr{\frac{\coth(\eta)}{\eta}-\csch^2(\eta)}, \quad g_3(\eta, \ell/L) = \frac{e^{2\eta}-e^{2\eta \ell/L}}{e^{2\eta}-1}.
    \end{align*}
\end{proposition}
Before giving a derivation of Proposition \ref{prop:sl-ints} let us write for any $x\in \R^{N_0}$
\begin{align*}
    \wx:= e^{\eta}\lr{1-2\wpsi_\mu}^{-1/2}\frac{x}{\sqrt{N_0}}
\end{align*}
and for convenience take instead
\begin{align*}
    \wx^{\otimes 3} := e^{\eta}\lr{1-2\wpsi_\mu}^{-3/2}\frac{x^{\otimes 3}}{\sqrt{N_0}}.
\end{align*}
Hence, we obtain the following as a direct consequence of Corollary \ref{cor:prior-moments}:

\begin{corollary}\label{cor:prior-moments-hat}
Up to errors of size $O(L/N^2)+O(1/L)$ we have
\begin{align*}
    &\E{\exp\left[-\frac{\sigma^2}{2N_L}\norm{X^{(L)}t}^2\right]}\\
    &\quad = \exp\left[-\frac{1}{2}\norm{\widehat{X}t}^2\right]\times \bigg\{ 1 - \wpsi^2 \lr{\sum_{\ell=1}^L \frac{1}{N_\ell} g_3(\eta,\ell/L)^2}\norm{\widehat{X}^{\otimes 3}t}^2 - 2\wpsi \lr{\sum_\ell \frac{1}{N_\ell}}\inprod{\widehat{X}t}{\widehat{X}\widehat{M}t}\\
    &\qquad \qquad\qquad \qquad\qquad \quad~~ + \frac{\wpsi }{2}\lr{\sum_{\ell=1}^L \frac{1}{N_\ell}g_3(\eta, \ell/L)} \inprod{\widehat{X}^{\otimes 3}t}{(\widehat{X}t)^{\otimes 3}} + \frac{1}{4}\lr{\sum_{\ell=1}^L\frac{1}{N_\ell}} \norm{\widehat{X}t}^4\\
    &\qquad \qquad\qquad \qquad\qquad \quad~~ + \wpsi^2\lr{\sum_{\ell=1}^L \frac{1}{N_\ell}g_3(\eta,\ell/L)^2} \inprod{\widehat{X}t\otimes \widehat{X}^{\otimes 3}t}{\widehat{X}^{\otimes 3}t\otimes \widehat{X}t}\bigg\}
\end{align*}
for
\begin{align*}
	\widehat{M} &= \diag\lr{\frac{\norm{x_\mu}^2}{N_0}} \diag\lr{1-2\wpsi_\mu}^{-2} \diag\lr{(1+\wpsi_\mu)g_3(\eta, \ell/L) - 3 \wpsi_\mu g_4(\eta, \ell/L)},
\end{align*}
where we introduced
\begin{align*}
	g_4(\eta, \ell/L) &= \frac{g_3(\eta, \ell/L)}{2} \frac{\coth(\eta)-1}{e^{2\eta\ell/L}-1}.
\end{align*}
\end{corollary}
To simplify computations, we will later take $N_\ell = N$. We note that this implies
\begin{align*}
	\widehat{M} = \diag\lr{\frac{\norm{x_\mu}^2}{N_0}} \diag\lr{1-2\wpsi_\mu}^{-2} \diag\lr{(1+\wpsi_\mu) g_1(\eta) - 3\wpsi_\mu g_2(\eta)}
\end{align*}
since
\begin{align*}
	\sum_\ell g_4(\eta, \ell/L) = g_2(\eta).
\end{align*}

\subsection{Derivation of Proposition \ref{prop:sl-ints}}
Recall that
\[
c_\mu^{(0)} = \sgn{\psi}^{s_\mu^{(0)}} \exp\left[\frac{2(\eta + \abs{\psi})}{L}\sum_{\ell=1}^L s_\mu^{(\ell)}\right].
\]
To obtain \eqref{eq:t00} we must therefore compute
\[
\E{c_\mu^{(0)} \lr{\frac{1}{N_0}\norm{x_\mu}^2}^{s_\mu^{(0)}}} = \E{\sgn{\psi}^{s_\mu^{(0)}} \exp\left[\frac{2(\eta + \abs{\psi})}{L}\sum_{\ell=1}^L s_\mu^{(\ell)}\right] \lr{\frac{1}{N_0}\norm{x_\mu}^2}^{s_\mu^{(0)}}}.
\]
Let us introduce the following random variables:
\begin{align*}
    j_{\mu,k} &= \text{ layer index in which self-loop }k\text{ is added for } k=1,\ldots, s_\mu^{(0)}
\end{align*}
and agree that by convention 
\[
j_{\mu,0} = L,\qquad j_{\mu, s_{\mu}^{(0)}+1} = 0.
\]
Note that for a given collection of values 
\[
L = j_0 \geq j_1\geq \cdots \geq j_{s_0} \geq j_{s_0+1}=0
\]
\begin{align*}
    \mathbb P\lr{j_{\mu,k}=j_k,\, k=1,\ldots, s_0,\, s_\mu^{(0)}=s_0} = \left[\prod_{k=0}^{s_0-1}\frac{\abs{\psi}}{L}\lr{1+2k} \right] \times \left[\prod_{k=0}^{s_0}\lr{1- \frac{\abs{\psi}}{L}}^{j_{k}-j_{k}-1}\right],
\end{align*}
with the first term computing the probability of adding $s_0$ self-loops and the second term computing the probability that no self-loops are added in layers $j_{k}+1,\ldots, j_{k+1}-1$. Moreover, on this event, we have
\[
\sum_{\ell=1}^L s_\mu^{(\ell)} = \sum_{k=1}^{s_0} k\lr{j_k-j_{k+1}}
\]
We thus have
\begin{align*}
    &\E{c_\mu^{(0)} \lr{\frac{1}{N_0}\norm{x_\mu}^2}^{s_\mu^{(0)}}}\\
    &\quad \sum_{s_0\geq 0}\sum_{L\geq j_1\geq \cdots j_{s_0}\geq 0} \left[\prod_{k=0}^{s_0-1}\frac{\abs{\psi}}{L}\lr{1+2k} \right] \times \left[\prod_{k=0}^{s_0}\lr{1- \frac{\abs{\psi}}{L}(1+2k)}^{j_{k}-j_{k-1}}\right]\\
    &\qquad\qquad\qquad\qquad \qquad \times \sgn{\psi}^{s_0} \exp\left[\frac{2(\eta + \abs{\psi})}{L}\sum_{k=1}^{s_0} k\lr{j_k-j_{k+1}}\right] \lr{\frac{1}{N_0}\norm{x_\mu}^2}^{s_0}.
\end{align*}
At large $L$ we replace the sum over $j_1,\ldots, j_{s_0}$ by an integral over
\[
\beta_k := \frac{j_k}{L}
\]
to obtain 
\begin{align*}
    &\E{c_\mu^{(0)} \lr{\frac{1}{N_0}\norm{x_\mu}^2}^{s_\mu^{(0)}}}\\
    &\quad =\sum_{s_0\geq 0}\lr{\frac{\psi}{N_0}\norm{x_\mu}^2}^{s_0}(2s_0-1)!! \int_{1\geq \beta_1\geq \cdots \geq \beta_{s_0}\geq 0}\exp\left[2(\eta + \abs{\psi})\sum_{k=1}^{s_0} k\lr{\beta_k-\beta_{k+1}}\right]\\
    &\qquad \quad \qquad\qquad \qquad\qquad \qquad\qquad \qquad\qquad \qquad\times \exp\left[- \abs{\psi}\sum_{k=0}^{s_0}(1+2k)\lr{\beta_k - \beta_{k+1}}\right] d\beta_1\cdots d\beta_{s_0}\\
    &\quad =\exp\left[-\abs{\psi}\right]\sum_{s_0\geq 0}\lr{\frac{\psi}{N_0}\norm{x_\mu}^2}^{s_0}(2s_0-1)!! \int_{1\geq \beta_1\geq \cdots \geq \beta_{s_0}\geq 0}\exp\left[2\eta\sum_{k=1}^{s_0} \beta_k\right]d\beta_1\cdots d\beta_{s_0}.
\end{align*}
Since the integrand in the final line is a symmetric function of $\beta$ we may replace
\[
 \int_{1\geq \beta_1\geq \cdots \geq \beta_{s_0}\geq 0}d\beta_1\cdots d\beta_{s_0}\quad \mapsto \quad \frac{1}{s_0!}\int_{[0,1]^{s_0}}d\beta_1\cdots d\beta_{s_0}
\]
to obtain 
\begin{align*}
    \E{c_\mu^{(0)} \lr{\frac{1}{N_0}\norm{x_\mu}^2}^{s_\mu^{(0)}}}&=\exp\left[-\abs{\psi}\right]\sum_{s_0\geq 0}\frac{(2s_0-1)!!}{2^{s_0}s_0!}\left[\frac{1}{\eta}\lr{\exp\left[2\eta\right]-1}\frac{\psi}{N_0}\norm{x_\mu}^2\right]^{s_0}\\
    &=\exp\left[-\abs{\psi}\right] \lr{1-\frac{1}{\eta}\lr{\exp\left[2\eta\right]-1}\frac{\psi}{N_0}\norm{x_\mu}^2}^{-1/2},
\end{align*}
where in the last line we've used a well-known hypergeometric series identity. This concludes the derivation of \eqref{eq:t00}. Let us now derive \eqref{eq:t11b} in a similar manner. We seek to compute
\begin{align*}
    &\E{c_\mu^{(0)} \lr{\frac{1}{N_0}\norm{x_\mu}^2}^{s_\mu^{(0)}-1}\frac{1}{L}\sum_{\ell=1}^Ls_\mu^{(\ell)}}\\
    &\quad = \E{\sgn{\psi}^{s_\mu^{(0)}}\lr{\frac{\norm{x_\mu}^2}{N_0}}^{s_\mu^{(0)}-1}\exp\left[\frac{2(\eta + \abs{\psi})}{L}\sum_{\ell=1}^L s_\mu^{(\ell)}\right] \frac{1}{L}\sum_{\ell=1}^L s_\mu^{(\ell)}}\\
    &\quad =\lr{\frac{\norm{x_\mu}^2}{N_0}}^{-1}\frac{1}{2}\frac{d}{d\eta}\E{c_\mu^{(0)} \lr{\frac{1}{N_0}\norm{x_\mu}^2}^{s_\mu^{(0)}}}\\
    &\quad =\lr{\frac{\norm{x_\mu}^2}{N_0}}^{-1}\exp\left[-\abs{\psi}\right]\frac{\psi}{N_0}\norm{x_\mu}^2\left[\frac{d}{d\eta}\lr{\frac{1}{4\eta}\lr{\exp[2\eta]-1}}\right]\\
    &\qquad \times \lr{1-\frac{1}{\eta}(\exp\left[2\eta\right]-1)\frac{\psi}{N_0}\norm{x_\mu}^2}^{-3/2}\\
    &\quad =\psi \exp\left[-\abs{\psi}\right] \frac{(2\eta-1)\exp\left[2\eta\right]+1}{4\eta^2}\lr{1-\frac{1}{\eta}(\exp\left[2\eta\right]-1)\frac{\psi}{N_0}\norm{x_\mu}^2}^{-3/2},
\end{align*}
which gives \eqref{eq:t11b}. We now turn to computing \eqref{eq:t11a}, which is
\begin{align*}
    \E{c_\mu^{(0)}\lr{\frac{\norm{x_\mu}^2}{N_0}}^{s_\mu^{(0)}-1}s_\mu^{(\ell)}} =\E{\sgn{\psi}^{s_\mu^{(0)}}\exp\left[\frac{2(\eta+\abs{\psi})}{L}\sum_{\ell=1}^L s_\mu^{(\ell)}\right]\lr{\frac{\norm{x_\mu}^2}{N_0}}^{s_\mu^{(0)}-1}s_\mu^{(\ell)}}.
\end{align*}
Writing this explicitly in terms of the $j_{\mu,k}$ random variables we introduced when computing \eqref{eq:t00} allows us to express the expectation on the right hand side of the preceding line as follows:
\begin{align*}
    &\sum_{s_0=0}^\infty \sum_{s_\ell = 0}^{s_0}\sgn{\psi}^{s_0}\lr{\frac{\norm{x_\mu}^2}{N_0}}^{s_0-1}s_\ell\\
    &\qquad \qquad \times \sum_{\substack{L\geq j_1\geq j_{s_0}\geq 0\\ j_{s_\ell}\geq \ell \geq j_{s_{\ell+1}}}} \lr{\frac{\abs{\psi}}{L}}^{s_0}(2s_0-1)!! \exp\left[-\abs{\psi}+\frac{2\eta}{L}\sum_{k=0}^{s_0} k(j_k-j_{k+1})\right].
\end{align*}
Approximating as before the sum over $j_k$ by an integral this becomes
\begin{align*}
    &\psi \exp\left[-\abs{\psi}\right]\sum_{s_0=0}^\infty \sum_{s_\ell = 0}^{s_0}\lr{\frac{\psi\norm{x_\mu}^2}{N_0}}^{s_0-1}(2s_0-1)!! s_\ell \times \int_{\substack{1\geq \beta_1\geq \cdots \geq \beta_{s_0}\geq 0\\ \beta_{s_\ell}\geq \ell/L\geq \beta_{s_{\ell+1}}}}  \exp\left[2\eta \sum_{k=1}^{s_0} \beta_k\right]d\beta_1\cdots d\beta_{s_0}.
\end{align*}
Symmetrizing over $\beta_1,\ldots, \beta_{s_\ell}$ and also over $\beta_{s_{\ell+1}},\ldots, \beta_{s_0}$ and then performing the integrals over $\beta$ allows us to write the preceding line as
\begin{align*}
    &\psi \exp\left[-\abs{\psi}\right]\sum_{s_0=1}^\infty \sum_{s_\ell = 1}^{s_0} \lr{\frac{\psi\norm{x_\mu}^2}{N_0}}^{s_0-1}\frac{(2s_0-1)!!}{(s_{\ell}-1)!(s_0-s_\ell)!} \times\left[\frac{\exp\left[2\eta\right]-\exp\left[2\eta\ell/L\right]}{2\eta}\right]^{s_\ell}\left[\frac{\exp\left[2\eta\right]-1}{2\eta}\right]^{s_0-s_\ell}.
\end{align*}
The inner sum over $s_\ell$ is
\begin{align*}
    &\sum_{s_\ell = 1}^{s_0}\frac{1}{(s_{\ell}-1)!(s_0-s_\ell)!} \times\left[\frac{\exp\left[2\eta\right]-\exp\left[2\eta\ell/L\right]}{2\eta}\right]^{s_\ell}\left[\frac{\exp\left[2\eta\right]-1}{2\eta}\right]^{s_0-s_\ell}\\
    &\quad =\frac{1}{(s_0-1)!}\left[\frac{\exp\left[2\eta\right]-\exp\left[2\eta\ell/L\right]}{2\eta}\right] \lr{\frac{\exp\left[2\eta\right]-1}{2\eta}}^{s_0-1}.
\end{align*}
Substituting this into the preceding expression shows that \eqref{eq:t11a} equals
\begin{align*}
    &\psi \exp\left[-\abs{\psi}\right]\left[\frac{\exp\left[2\eta\right]-\exp\left[2\eta\ell/L\right]}{2\eta}\right]\sum_{s_0=1}^\infty \frac{(2s_0-1)!!}{(s_{0}-1)!}\lr{\frac{\psi\norm{x_\mu}^2}{N_0}\frac{\exp\left[2\eta\right]-1}{2\eta}}^{s_0-1}\\
    &=\psi \exp\left[-\abs{\psi}\right]\left[\frac{\exp\left[2\eta\right]-\exp\left[2\eta\ell/L\right]}{2\eta}\right]\lr{1-\frac{\psi\norm{x_\mu}^2}{N_0}\frac{\exp\left[2\eta\right]-1}{2\eta}}^{-3/2},
\end{align*}
as desired. Finally, we turn to deriving \eqref{eq:t20}. We compute a single layer (briefly taking $\psi>0$):
\begin{align*}
	\E{c_\mu^{(0)} (s_\mu^{(\ell)})^2 \lr{\frac{\sigma^2}{N_0}\norm{x_\mu}^2}^{s_0} } &= \sum_{s_0=0}^\infty \lr{\frac{\psi}{L}}^{s_0}(2s_0-1)!! \sum_{j_0\geq j_1 \cdots > j_{s_0} \geq j_{s_0+1}} \prod_{k=0}^{s_0} \lr{1 - \frac{\psi}{L}(1+2k)}^{j_k-j_{k+1}-1} \\
	&\quad \times \lr{\frac{\sigma^2}{N_0}\norm{x}^2}^{s_0} \prod_{k=0}^{s_0} \lr{1+\frac{2(\psi+\eta)}{L}k}^{j_k-j_{k+1}} \sum_{k=0}^{s_0} k^2 1_{j_k\geq \ell \geq j_{k+1}}\\
	&= \exp[-\psi]\sum_{s_0=0}^\infty \psi^{s_0} \lr{\frac{\sigma^2}{N_0}\norm{x}^2}^{s_0} (2s_0-1)!! \\
	&\quad \times \int_{\beta_0 \geq \cdots \geq \beta_{s_0+1}} d\beta_1\cdots d\beta_{s_0} \exp\left[2\eta\sum_{k=1}^{s_0} \beta_k\right] \sum_{k=0}^{s_0} k^2 1_{\beta_k\geq \ell/L \geq \beta_{k+1}}.
\end{align*}
Hence, the integral is
\begin{align*}
	I(\ell/L) &= \int_{\beta_0 \geq \cdots \geq \beta_{s_0+1}} d\beta_1\cdots d\beta_{s_0} \exp\left[2\eta\sum_{k=1}^{s_0} \beta_k\right] \sum_{k=0}^{s_0} k^2 1_{\beta_k\geq \ell/L \geq \beta_{k+1}}\\
	&= \sum_{k=1}^{s_0} k^2 \int_{\ell/L}^1 \int_{\ell/L}^{\beta_1} \cdots \int_{\ell/L}^{\beta_{k-1}}\int_0^{\ell/L} \int_0^{\beta_{k+1}} \cdots \int_0^{\beta_{s_0-1}} d\beta_{s_0} \cdots d\beta_1 \exp\left[2\eta\sum_{j=1}^{s_0} \beta_j\right]\\
	&= \sum_{k=1}^{s_0} k^2 \int_{\ell/L}^1 \int_{\ell/L}^{\beta_1} \cdots \int_{\ell/L}^{\beta_{k-1}} d\beta_{k} \cdots d\beta_1 \exp\left[2\eta\sum_{j=1}^{k} \beta_j\right] \\
	&\quad \times \int_0^{\ell/L}\int_0^{\beta_{k+1}} \cdots \int_0^{\beta_{s_0-1}} d\beta_{s_0} \cdots d\beta_{k+1} \exp\left[2\eta\sum_{j=k+1}^{s_0} \beta_j\right]\\
	&= \sum_{k=1}^{s_0} \frac{k^2}{(s_0-k)!} \lr{\frac{\exp[2\eta\ell/L]-1}{2\eta}}^{s_0-k} \int_{\ell/L}^1 \int_{\ell/L}^{\beta_1} \cdots \int_{\ell/L}^{\beta_{k-1}} d\beta_{k} \cdots d\beta_1 \exp\left[2\eta\sum_{j=1}^{k} \beta_k\right] \\
	&= \sum_{k=1}^{s_0} \frac{k^2}{k!(s_0-k)!} \lr{\frac{\exp[2\eta\ell/L]-1}{2\eta}}^{s_0-k} \lr{\frac{\exp[\eta(1+\ell/L)]\sinh(\eta(1-\ell/L))}{\eta}}^k.
\end{align*}
Summing over $k$ gives
\begin{align*}
	\E{c_\mu^{(0)} (s_\mu^{(\ell)})^2 \lr{\frac{\sigma^2}{N_0}\norm{x_\mu}^2}^{s_0} } &= \exp[-\psi] \sum_{s_0=0}^\infty \wpsi_\mu^{s_0} \frac{(2s_0-1)!!}{(s_0-1)!} g_3^2 \lr{s_0 + \frac{\exp[2\eta\ell/L]-1}{\exp[2\eta]-\exp[2\eta\ell/L]}}\\
	&= \wpsi_\mu \exp[-\psi] \lr{1-2\wpsi_\mu}^{-5/2} \lr{(1+\wpsi_\mu)g_3 - 3 \wpsi_\mu \frac{e^{2\eta\ell/L}-1}{e^{2\eta}-2e^{2\eta\ell/L}}g_3^2},
\end{align*}
as desired. \hfill $\square$
\subsection{Alternative Derivation of $\psi$-dependent Kernel}\label{sec:kernel-derivation}
Recall that when $N,L,P\gives \infty$ with $N\gg L,P$, we found that Bayesian inference with our deep shaped MLPs \eqref{eq:f-def-formal} was equivalent to Bayesian inference with the $\psi$-dependent kernel
\begin{align*}
    K_{\mu\nu}:=K_\psi\lr{x_\mu,x_\nu} = \inprod{(x_\mu)_\psi}{\lr{x_\nu}_\psi}.
\end{align*}
The purpose of this section is to give a simple and self-contained derivation of the following statement:
\begin{align}
    \label{eq:kernel-goal}K_{\mu\nu} = \lim_{L\gives \infty} \lim_{N\gives \infty}\Esub{\mathrm{prior}}{\frac{1}{N_L}\inprod{x_\mu^{(L)}}{x_\nu^{(L)}}},
\end{align}
In other words, we will show that the kernel $K_{\mu\nu}$ appears as the two-point function when first taking the width $N$ to infinity and then the depth $L$ to infinity. For this, let us write for any $L\geq 1$ as in \eqref{eq:NNGP}
\[
K_{\mu\nu}^{(L)}:=\lim_{N\gives \infty} \Esub{\mathrm{prior}}{\frac{1}{N_L}\inprod{x_\mu^{(L)}}{x_\nu^{(L)}}}.
\]
This kernel satisfies the well-known recursion \cite{lee2017deep,roberts2022principles,hanin2021random}
\begin{align}
    \label{eq:K-rec}K_{\mu\nu}^{(L+1)} = \Esub{K^{(L)}}{\phi\lr{z_\mu}\phi\lr{z_\nu}},\qquad K_{\mu\nu}^{(0)}= \frac{1}{N_0}\inprod{x_\mu}{x_\nu},
\end{align}
where the expectation is, by definition, over
\[
\lr{z_\mu,z_\nu}\sim \mN\lr{0,\twomat{K_{\mu\mu}^{(L)}}{K_{\mu\nu}^{(L)}}{K_{\nu\mu}^{(L)}}{K_{\nu\nu}^{(L)}}}.
\]
This recursion holds for any nonlinearity. In the setting \eqref{eq:phi-def-formal} in this article, the relation \eqref{eq:K-rec} becomes
\begin{align}
    \label{eq:K-rec-2}K_{\mu\nu}^{(L+1)} = K_{\mu\nu}^{(L)} \lr{1+\frac{\psi}{L}\lr{K_{\mu\mu}^{(L)}+K_{\nu\nu}^{(L)}}} + O(L^{-2}).
\end{align}
Let us turn the discrete layer index $\ell=0,\ldots, L$ into consider now a continuous time index:
\[
\ell\quad \mapsto\quad \tau := \ell/L.
\]
Taking $L\gives \infty$ and setting $\mu=\nu$ and rewriting \eqref{eq:K-rec-2} yields
\begin{align}
    \label{eq:K-rec-3}\frac{d}{d\tau}K_{\mu\mu}^{(\tau)} = \frac{2\psi}{L}\lr{K_{\mu\mu}^{(\tau)}}^2.
\end{align}
Solving this equation yields 
\begin{align}
    \label{eq:K-rec-soln-diag}K_{\mu\mu}^{(\tau)} = \lr{1-2\psi\tau K_{\mu\mu}^{(0)}}^{-1} K_{\mu\mu}^{(0)}=\lr{1-2\psi\tau \frac{\norm{x_\mu}^2}{N_0}}^{-1} \frac{\norm{x_\mu}^2}{N_0}. 
\end{align}
Plugging this back into \eqref{eq:K-rec-2} and solving a simple ODE gives
\begin{align}
    K_{\mu\nu}^{(\tau)} = \lr{1-2\psi\tau \frac{\norm{x_\mu}^2}{N_0}}^{-1/2}\lr{1-2\psi\tau \frac{\norm{x_\nu}^2}{N_0}}^{-1/2} \frac{\inprod{x_\mu}{x_\nu}}{N_0}.
\end{align}
Taking $\tau =1 $ completes the derivation of \eqref{eq:kernel-goal}.

\section{Perturbative Expansion of Partition Function}\label{sec:part-exp}
As described in \S \ref{sec:inf}, the partition function at inverse temperature $\beta$ 
\[
Z(x;\tau):=\Esub{\text{prior},\beta}{\exp\left[-i\tau f(x;\theta)-\beta \mathcal L(\theta\,|\,\mD)\right]}
\]
provides the posterior via the characteristic function
\[
\Esub{\text{post},\beta}{\exp\left[-i\tau f(x;\theta)\right]} = \frac{Z(x;\tau)}{Z(0)},\qquad Z(0):=Z(x;0),
\]
where $Z(0)$ is the model evidence. The main result in this section is the following perturbative expansion for the partition function.
\begin{proposition}\label{prop:part-exp}
Up to a constant of proportionality, the partition function is given by
\begin{align*}
    Z(x,\tau) &=\sqrt{\det(2\pi M_\beta) } \exp\left[Q(t_*)\right]\bigg[ I_1(x,\tau) + \frac{L}{N} \bigg[-e^{-4\eta}\wpsi^2 (g_1(\eta)-g_2(\eta))I_2(x,\tau) -2\wpsi I_3(x,\tau)\\
    &\quad + \frac{e^{-2\eta}\wpsi}{2}g_1(\eta)I_4(x,\tau) + e^{-4\eta}\wpsi^2(g_1(\eta)-g_2(\eta)) I_5(x,\tau) + \frac{1}{4}I_6(x,\tau)\bigg]\bigg] \\
    &\quad + O(L^2/N^2)+ O(1/L),
\end{align*}
where
\begin{align*}
    Q(t)&=-\frac{1}{2\beta}\norm{t}^2 +i\widehat{\theta}_*^T\widehat{X}t -\frac{1}{2}\norm{\widehat{X}t+\wx \tau}^2\\
    M_\beta &= \lr{\frac{1}{\beta}I + \widehat{X}^T \widehat{X}}^{-1}\\
    t_*&=M_\beta\left[i\widehat{X}^T\widehat{\theta}_* - \widehat{X}^T \wx \tau\right]=:i\wbXd \widehat{\theta}_* - \tau \wbXd \wx\\
    \wbXd&=M_\beta\widehat{X}^T.
\end{align*}
and $I_j$ are given in Proposition \ref{prop:t-ints}.
\end{proposition}
\begin{proof}
We have
\begin{align*}
    Z_\beta(x,\tau) &\simeq \int_{\R^P} \exp\left[-\frac{1}{2\beta}\norm{t}^2 +iY^Tt\right] \Esub{\mathrm{prior}}{\exp\left[-\frac{1}{2N_L}\norm{X^{(L)}t+x^{(L)}\tau}^2\right]} dt.
\end{align*}
Hence, using Corollary \ref{cor:prior-moments-hat} we obtain 
\begin{align*}
    Z_\beta(\tau) &\simeq \int_{\R^P} \exp\left[-\frac{1}{2\beta}\norm{t}^2 +i\widehat{\theta}_*^T\widehat{X}t -\frac{1}{2}\norm{\widehat{X}t+\wx \tau}^2\right]\\
    &\qquad \qquad \times \bigg\{ 1 + \frac{L}{N}\bigg[ - e^{-4\eta}\wpsi^2 \lr{g_1(\eta)-g_2(\eta)}\norm{\widehat{X}^{\otimes 3}t +\wx^{\otimes 3}\tau}^2 - 2\wpsi\inprod{\widehat{X}t+\wx \tau}{\widehat{X}\widehat{M}t+\wx \widehat{m}\tau}\\
    &\qquad \qquad\quad \quad~~ + \frac{e^{-2\eta}\wpsi }{2} g_1(\eta)\inprod{\widehat{X}^{\otimes 3}t+\wx^{\otimes 3}\tau}{(\widehat{X}t+\wx\tau )^{\otimes 3}} + \frac{1}{4} \norm{\widehat{X}t+\wx \tau}^4\\
    &\qquad \qquad\quad \quad~~ + \wpsi^2e^{-4\eta}\lr{g_1(\eta)-g_2(\eta)} \left\langle \lr{\widehat{X}t+\wx \tau}\otimes\lr{ \widehat{X}^{\otimes 3}t+\wx^{\otimes 3}\tau},\right.\\
    &\qquad \qquad\qquad\qquad \qquad\quad \qquad \qquad \qquad \qquad \qquad \quad~~\left. \lr{\widehat{X}^{\otimes 3}t+\wx^{\otimes 3}\tau}\otimes \lr{\widehat{X}t+\wx\tau}\right\rangle\bigg]\bigg\}dt,
\end{align*}
where we used that
\[
\sum_{\ell=1}^L g_3(\eta) = g_1(\eta),\qquad \sum_{\ell=1}^L g_3(\eta)^2 = g_1(\eta)-g_2(\eta)
\]
and defined
\begin{align*}
    \widehat{M} = \diag\lr{\frac{\norm{x_\mu}^2}{N_0}} \diag\lr{1-2\wpsi_\mu}^{-2} \diag\lr{(1+\wpsi_\mu) g_1(\eta) - 3\wpsi_\mu g_2(\eta)}.
\end{align*}
To perform this Gaussian integral, let us write
\[
Q(t):=-\frac{1}{2\beta}\norm{t}^2 +i\widehat{\theta}_*^T\widehat{X}t -\frac{1}{2}\norm{\widehat{X}t+\wx \tau}^2.
\]
A direct computation then yields
\begin{align*}
    Q(t)=Q(t_*)-\frac{1}{2}\inprod{M_\beta^{-1}(t-t_*)}{t-t_*},
\end{align*}
where
\begin{align*}
    M_\beta &= \lr{\frac{1}{\beta}I + \widehat{X}^T \widehat{X}}^{-1}\\
    t_* &= M_\beta\left[i\widehat{X}^T\widehat{\theta}_* - \widehat{X}^T \wx \tau\right]=:i\wbXd \widehat{\theta}_* - \tau \wbXd \wx\\
    \wbXd&=M_\beta\widehat{X}^T.
\end{align*}
Hence, we find that the partition function can be expressed as follows
\begin{align*}
    Z(x,\tau) &\simeq \sqrt{\det(2\pi M_\beta) } \exp\left[Q(t_*)\right]\bigg[ I_1- \frac{L}{N}e^{-4\eta}\wpsi^2 (g_1(\eta)-g_2(\eta))I_2 -2\frac{L}{N}\wpsi I_3\\
    &\qquad \qquad \qquad \qquad \qquad \qquad\quad + \frac{e^{-2\eta}\wpsi L}{2N}g_1(\eta)I_4  + e^{-4\eta}\wpsi^2\frac{L}{N}(g_1(\eta)-g_2(\eta)) I_5 + \frac{L}{4N}I_6\bigg],
\end{align*}
where
\begin{align*} 
    I_1&=1\\
    I_2&=\Esub{t\sim \mathcal N(i\wbXd \widehat{\theta}_*, M_\beta)}{\norm{\widehat{X}^{\otimes 3}t +\tau\lr{\wx^{\otimes 3}}_{\perp,\beta}}^2}\\
    I_3&=\Esub{t\sim \mathcal N(i\wbXd \widehat{\theta}_*, M_\beta)}{\inprod{\widehat{X}t +\tau\wx_{\perp,\beta}}{\widehat{X}\widehat{M}t +\tau\lr{\wx \widehat{m}-\widehat{X}\widehat{M}\wbXd \wx}}}\\
    I_4&=\Esub{t\sim \mathcal N(i\wbXd \widehat{\theta}_*, M_\beta)}{\inprod{\widehat{X}^{\otimes 3}t +\tau\lr{\wx^{\otimes 3}}_{\perp,\beta}}{\lr{\widehat{X}t +\tau\wx_{\perp, \beta}}^{\otimes 3}}}\\
    I_5&=\mathbb E_{t\sim \mathcal N(i\wbXd \widehat{\theta}_*, M_\beta)}\bigg[ \left\langle \lr{\widehat{X}^{\otimes 3}t +\tau\lr{\wx^{\otimes 3}}_{\perp,\beta}}\otimes  \lr{\widehat{X}t +\tau \wx_{\perp,\beta}},\right. \\
    &\qquad \qquad \qquad\qquad \qquad  \left. \lr{\widehat{X}t +\tau\wx_{\perp,\beta}}\otimes \lr{\widehat{X}^{\otimes 3}t +\tau\lr{\wx^{\otimes 3}}_{\perp,\beta}}\right \rangle\bigg]\\
    I_6&=\Esub{t\sim \mathcal N(i\wbXd \widehat{\theta}_*, M_\beta)}{\norm{\widehat{X}t +\tau \wx_{\perp,\beta}}^4}
\end{align*}
and we've set
\begin{align*}
    \wx_{\perp,\beta}&=\wx - \widehat{X}\wbXd \wx\\
    \lr{\wx^{\otimes 3}}_{\perp,\beta} &= \wx^{\otimes 3}-\widehat{X}^{\otimes 3}\wbXd \wx.
\end{align*}
\end{proof}

\subsection{Integrals in the Partition Function}
We present the integrals referenced in Proposition~\ref{prop:part-exp}; this simply requires evaluating the Gaussian expectations provided in the proof of Proposition~\ref{prop:part-exp}.

\begin{proposition}\label{prop:t-ints}
Let
\begin{align*}
    M_\beta &= \lr{\frac{1}{\beta}I + \widehat{X}^T \widehat{X}}^{-1}, \quad
    \wbXd =M_\beta\widehat{X}^T\\
    \widehat{M} &= \diag\lr{\frac{\norm{x_\mu}^2}{N_0}} \diag\lr{1-2\wpsi_\mu}^{-2} \diag\lr{(1+\wpsi_\mu) g_1(\eta) - 3\wpsi_\mu g_2(\eta)},
\end{align*}
and
\begin{align*}
    \wx_{\perp,\beta} =\wx - \widehat{X}\wbXd \wx, \quad
    \lr{\wx^{\otimes 3}}_{\perp,\beta} = \wx^{\otimes 3}-\widehat{X}^{\otimes 3}\wbXd \wx.
\end{align*}
The integrals in Proposition~\ref{prop:part-exp} are given by
\begin{align}
    I_1 &= 1
\end{align}
and
\begin{align}
      \notag   I_2&=\sum_{\mu,\nu=1}^P \inprod{\wx_\mu}{\wx_\nu}^3 \lr{M_\beta}_{\mu\nu} - \norm{\widehat{X}^{\otimes 3}\wbXd\widehat{\theta}_*}^2\\
      \notag &\quad + 2i\tau \inprod{\wX^{\otimes 3}\wbXd \widehat{\theta}_*}{(\wx^{\otimes 3})_{\perp,\beta}}\\
      \label{eq:I2}&\quad + \tau^2 \norm{(\wx^{\otimes 3})_{\perp,\beta}}^2
\end{align}
and
\begin{align}
      \notag I_3&=\sum_{\mu,\nu}\inprod{\wx_\mu}{\wx_\nu} \lr{\widehat{M}M_\beta}_{\mu\nu}- \inprod{\wX\wbXd \wtheta_*}{\wX\widehat{M}\wbXd\wtheta_*}\\
      \notag &\quad + i\tau \left[\inprod{\widehat{X}\wbXd \wtheta_*}{\wx \widehat{m}-\widehat{X}\widehat{M}\wbXd \wx}+\inprod{\wx_{\perp,\beta}}{\widehat{X}\widehat{M}\wbXd\wtheta_*}\right]\\
      \label{eq:I3}&\quad +\tau^2\inprod{\wx_{\perp,\beta}}{\wx\widehat{m}-\widehat{X}\widehat{M}\wbXd \wx}.
    \end{align}
For
\begin{align}
    \label{eq:I4} I_4 &= I_4[0]+iI_4[1] \tau+ 3I_4[2]\tau^2 +iI_4[3]\tau^3+I_4[4]\tau^4,
\end{align}
we have
\begin{align}
    \notag I_4[0]&=\inprod{\wX^{\otimes 3} \wbXd \wtheta_*}{\lr{\wX\wbXd \wtheta_*}^{\otimes 3}} - 3 \inprod{\sum_\mu \lr{\wbXd \wx_\mu}_\mu \wx_\mu^{\otimes 2}}{ \wX\wbXd\wtheta_* \otimes  \wX\wbXd \wtheta_*}\\
    \label{eq:I40}&-3\sum_\mu \inprod{ \wx_\mu}{\wx_\mu}\lr{\wbXd \wtheta_*}_\mu \inprod{\wx_\mu}{ \wX\wbXd \wtheta_*} + 3\sum_\mu \lr{\wbXd \wx_\mu}_\mu \inprod{ \wx_\mu}{\wx_\mu}\\
    \notag I_4[1]&=6\sum_\mu (\wbXd \wx_\mu)_\mu \inprod{\wx_\mu^{\otimes 2}}{\wX\wbXd\wtheta_* \otimes \wxpb }+ 3\sum_\mu\lr{\wbXd\wtheta_*}_\mu\inprod{\wX \wbXd \wx_\mu}{\wx_\mu}\inprod{\wx_\mu}{\wxpb}\\
    &\notag \quad  + 3\inprod{\norm{M_\beta^{1/2}\wX^T\wx}^2\wx-\sum_\mu \lr{\wbXd \wx}_\mu\norm{M_\beta^{1/2}\wX^T\wx_\mu}^2\wx_\mu}{ \wX\wbXd \wtheta_*}\\
    &\label{eq:I41}\quad - 3\inprod{\wX^{\otimes 3}\wbXd\wtheta_*}{(\wX\wbXd\wtheta)^{\otimes 2}\otimes \wxpb}- \inprod{\wxtpb}{(\wX\wbXd\wtheta)^{\otimes 3}}\\
    \notag I_4[2]&=\sum_{\mu}\lr{\wbXd \wx_\mu}_\mu \inprod{\wx_\mu^{\otimes 2}}{\lr{\wxpb}^{\otimes 2}} - \inprod{\wX^{\otimes 3}\wbXd\wtheta_*}{\wX\wbXd\wtheta\otimes\wxpb^{\otimes 2}}\\
    \notag &\quad + \lr{\inprod{\wX \wbXd \wx}{\wx} - \inprod{\wx^{\otimes 2}}{\lr{\wX \wbXd \wtheta_*}^{\otimes 2}}}\inprod{\wx}{\wxpb} \\
    \label{eq:I42}&\quad + \sum_\mu \lr{\wbXd \wx}_\mu \lr{\inprod{\wx_\mu^{\otimes 2}}{\lr{\wX\wbXd\wtheta_*}^{\otimes 2}}-\inprod{\wX \wbXd \wx_\mu}{\wx_\mu}}\inprod{\wx_\mu}{\wxpb}\\
    \label{eq:I43}I_4[3]&=
    \inprod{\wX^{\otimes 3}\wbXd\wtheta_*}{\lr{\wxpb}^{\otimes 3}}+3\inprod{\wxtpb}{\wX \wbXd \wtheta_*\otimes \lr{\wxpb}^{\otimes 2}}\\
    \label{eq:I44}I_4[4]&=\inprod{\wxtpb}{\lr{\wxpb}^{\otimes 3}}.
\end{align}
For 
\begin{align}
    I_5= I_5[0]+ 2iI_5[1]\tau + I_5[2]\tau^2 + 2iI_5[3]\tau^3 + I_5[4]\tau^4, 
\end{align}
we have
\begin{align}
    \notag I_5[0]&= \sum_{\mu\nu} \inprod{\wx_\mu^{\otimes 2}}{\wx_\nu^{\otimes 2}}\bigg[\inprod{\wX \wbXd \wx_\mu}{\wx_\nu}\lr{M_\beta}_{\mu\nu} + \lr{\wbXd \wx_\mu}_\mu\lr{\wbXd \wx_\nu}_\nu + \lr{\wbXd \wx_\mu}_\nu\lr{\wbXd \wx_\nu}_\mu\bigg]\\
    \notag &\quad - 2 \sum_{\mu}\inprod{\wx_\mu^{\otimes 3}}{\wX\wbXd \wtheta_*\otimes \sum_\nu \lr{\wbXd \wtheta_*}_\nu\lr{\wbXd \wx_\nu}_\mu \wx_\nu^{\otimes 2}}\\
    \notag&\quad - 2\inprod{\sum_\mu \lr{\wbXd \wx_\mu}_\mu \wx_\mu^{\otimes 2}\otimes \wX\wbXd\wtheta_*}{\wX^{\otimes 3}\wbXd\wtheta_*}\\
   \notag &\quad-\sum_{\mu\nu} \lr{M_\beta}_{\mu\nu}\inprod{\wx_\mu^{\otimes 3}\otimes  \wX\wbXd\wtheta_*}{ \wX\wbXd\wtheta_*\otimes \wx_\nu^{\otimes 3}} -\norm{\sum_\mu \lr{\wbXd \wtheta_*}_\mu M_\beta^{1/2}\wX^T\wx_\mu \otimes \wx_\mu^{\otimes 2}}^2\\
    \label{eq:I50}&\quad +\inprod{\wX^{\otimes 3}\wbXd \wtheta_*\otimes \wX\wbXd \wtheta_*}{\wX\wbXd \wtheta_*\otimes \wX^{\otimes 3}\wbXd\wtheta_*}
\end{align}
and
\begin{align}
    \notag I_5[1]&=-\inprod{\wX^{\otimes 3}\wbXd \wtheta_* \otimes \wxpb}{\wX\wbXd \wtheta_* \otimes \wX^{\otimes 3}\wbXd \wtheta_*} - \inprod{\wxtpb \otimes \wX\wbXd \wtheta_*}{\wX\wbXd \wtheta_* \otimes \wX^{\otimes 3}\wbXd \wtheta_*}\\
    \notag &\quad + \sum_\nu \inprod{\inprod{\wX\wbXd \wx}{\wx_\nu}\wx^{\otimes 2}- \sum_{\mu} \lr{\wbXd \wx}_\mu \inprod{\wX\wbXd \wx_\nu}{\wx_\mu} \wx_\mu^{\otimes 2}}{\lr{\wbXd \wtheta_*}_\nu \wx_\nu^{\otimes 2}}\\
    \notag  &\quad+ \inprod{\wxtpb}{\wX\wbXd \wtheta_*\otimes \sum_\nu \lr{\wbXd \wx_\nu}_\nu \wx_\nu^{\otimes 2}}\\
    \notag &\quad+\sum_\nu \inprod{\lr{(\wbXd \wx)_\nu \wx^{\otimes 2}-\sum_{\mu}(\wbXd \wx)_\mu (\wbXd \wx_\mu)_\nu \wx_\mu^{\otimes 2}}\otimes \wX\wbXd \wtheta_*}{\wx_\nu^{\otimes 3}}\\
    \notag  &\quad +\inprod{\sum_\mu (\wbXd \wx_\mu)_\mu \wx_\mu^{\otimes 2}\otimes \wxpb}{\wX^{\otimes 3}\wbXd \wtheta_*}\\
    \notag &\quad + \sum_{\mu\nu} (M_\beta)_{\mu\nu} \inprod{\wx_\mu^{\otimes 3}\otimes \wxpb}{\wX\wbXd \wtheta_*\otimes \wx_\nu^{\otimes 3}}\\
    \label{eq:I51} &\quad + \sum_\nu \inprod{\sum_\mu (\wbXd \wtheta_*)_\mu (\wbXd \wx_\mu)_\nu \wx_\mu^{\otimes 2}\otimes \wxpb}{\wx_\nu^{\otimes 3}}
\end{align}
and 
\begin{align}
    \notag I_5[2]&=\sum_{\mu\nu}(M_\beta)_{\mu\nu}\inprod{\wxtpb\otimes \wx_\mu}{\wx_\nu\otimes \wxtpb} - \inprod{\wxtpb\otimes \wX\wbXd \wtheta_*}{\wX\wbXd \wtheta_*\otimes \wxtpb}\\
    \notag&\quad + \sum_{\mu\nu} \lr{M_\beta}_{\mu\nu}\inprod{\wx_\mu^{\otimes 3}\otimes \wxpb}{\wxpb\otimes \wx_\nu^{\otimes 3}} - \inprod{\wX^{\otimes 3}\wbXd \wtheta_* \otimes \wxpb}{\wxpb\otimes \wX^{\otimes 3}\wbXd \wtheta_*}\\
    \notag&\quad + 2\inprod{\wxtpb}{\wxpb\otimes \sum_\mu \lr{\wbXd \wx_\mu}_\mu \wx_\mu^{\otimes 2}} - 2\inprod{\wxtpb\otimes \wX\wbXd \wtheta_*}{\wxpb\otimes \wX^{\otimes 3}\wbXd \wtheta_*}\\
    \notag&\quad + 2\sum_\nu \inprod{\lr{\lr{\wbXd \wx}_\nu \wx^{\otimes 2}-\sum_\mu \lr{\wbXd \wx}_\mu\lr{\wbXd \wx_\mu}_\nu \wx_\mu^{\otimes 2}}\otimes \wxpb}{\wx_\nu^{\otimes 3}} \\
    \label{eq:I52}&\quad - 2\inprod{\wxtpb\otimes \wxpb}{\wX\wbXd \wtheta_*\otimes \wX^{\otimes 3}\wbXd \wtheta_*}
\end{align}
and
\begin{align}
    \label{eq:I53}I_5[3]&=\inprod{\wX^{\otimes 3}\wbXd \wtheta_* \otimes \wxpb}{\wxpb\otimes \wxtpb}+\inprod{\wxtpb\otimes \wX\wbXd\wtheta_*}{\wxpb\otimes \wxtpb}\\
    \label{eq:I54}I_5[4]&=\inprod{\wxtpb\otimes \wxpb}{\wxpb\otimes \wxtpb}.
\end{align}
For
\begin{align*}
    I_6&=I_6[0]+4i\tau I_6[1] + 2\tau^2 I_6[2] + 4i\tau^3 I_6[3] + \tau^4 I_6[4],
\end{align*}
we have
\begin{align}
    \notag I_6[0]&=\norm{\wX\wbXd\wtheta_*}^4 - 2\norm{\wX\wbXd\wtheta_*}^2\sum_{\mu\nu} (M_\beta)_{\mu\nu}\inprod{\wx_\mu}{\wx_\nu} -4\norm{M_\beta^{1/2}\wX^T\wX\wbXd\wtheta_*}^2\\
    \label{eq:I60}    &+\lr{\sum_{\mu,\nu}(M_\beta)_{\mu\nu}\inprod{\wx_\mu}{\wx_\nu}}^2 + 2\sum_{\mu_1,\mu_2,\nu_1,\nu_2} \lr{M_\beta}_{\mu_1\mu_2}\lr{M_\beta}_{\nu_1\nu_2} \inprod{\wx_{\mu_1}}{\wx_{\nu_1}}\inprod{\wx_{\mu_2}}{\wx_{\nu_2}}\\
    \notag I_6[1]&=\sum_{\mu\nu} (M_\beta)_{\mu\nu} \inprod{\wx_\mu}{\wx_\nu} \inprod{\wX\wbXd\wtheta_*}{\wxpb}\\
    \label{eq:I61}&+2\inprod{\wxpb}{\lr{\wbXd}^T\lr{\wX^T\wX}\wbXd\wtheta_*} -\norm{\wX\wbXd\wtheta_*}^2\inprod{\wX\wbXd\wtheta_*}{\wxpb}\\
    \label{eq:I62}I_6[2]&=\norm{\wxpb}^2\lr{\sum_{\mu\nu}\lr{M_\beta}_{\mu\nu}\inprod{\wx_\mu}{\wx_\nu} - \norm{\wX\wbXd\wtheta_*}^2}+2\inprod{\wX\wbXd(\wxpb-\wtheta_*)}{\wxpb}\\
    \label{eq:I63} I_6[3]&=\inprod{\wX\wbXd \wtheta_*}{\wxpb}\norm{\wxpb}^2\\
    \label{eq:I64}I_6[4]&=\norm{\wxpb}^4.
\end{align}
\end{proposition}
\begin{proof}
As the method of computation is similar across all integrals, we show an example of how to compute the necessary expectations. The most involved integral is $I_5$, so we evaluate it here. The initial expectation is
\begin{align*}
    I_5 &= \mathbb E_{t\sim \mathcal N(i\wbXd \widehat{\theta}_*, M_\beta)}\bigg[ \left\langle \lr{\widehat{\wX}^{\otimes 3}t +\tau\lr{\wx^{\otimes 3}}_{\perp,\beta}}\otimes  \lr{\widehat{\wX}t +\tau \wx_{\perp,\beta}},\right. \\
    &\qquad \qquad \qquad\qquad \qquad  \left. \lr{\widehat{\wX}t +\tau\wx_{\perp,\beta}}\otimes \lr{\widehat{\wX}^{\otimes 3}t +\tau\lr{\wx^{\otimes 3}}_{\perp,\beta}}\right \rangle\bigg].
\end{align*}
Separating powers of $\tau$, we have
\begin{align*}
    I_5 &=\E{\inprod{\widehat{\wX}^{\otimes 3}t \otimes \wX t}{\wX t \otimes \widehat{\wX}^{\otimes 3}t}}\\
    &\quad +2\tau\mathbb E \bigg[\inprod{\lr{\wx^{\otimes 3}}_{\perp} \otimes \wX t}{\wX t\otimes \wX^{\otimes 3} t} + \inprod{\wX^{\otimes 3}t\otimes \wxpb}{\wX t \otimes \wX^{\otimes 3}t}\bigg]\\
    &\quad +\tau^2 \mathbb E\bigg[ 2\inprod{\lr{\wx^{\otimes 3}}_{\perp} \otimes \wxpb}{\wX t \otimes \wX^{\otimes 3}t}+ 2\inprod{\lr{\wx^{\otimes 3}}_{\perp}\otimes \wX t}{\wxpb \otimes \wX^{\otimes 3}t}\\
    &\qquad \qquad ~ + \inprod{\lr{\wx^{\otimes 3}}_{\perp}\otimes \wX t}{\wX t \otimes\lr{\wx^{\otimes 3}}_{\perp}} + \inprod{\wX^{\otimes 3}t \otimes \wxpb}{\wxpb \otimes \wX^{\otimes 3}t}\bigg]\\
    &\quad +2i\tau^3\bigg[\inprod{\wX^{\otimes 3}\wbXd \wtheta_*\otimes \wxpb}{\wxpb\otimes \lr{\wx^{\otimes 3}}_{\perp}} + \inprod{\lr{\wx^{\otimes 3}}_{\perp} \otimes \wX\wbXd \wtheta_*}{\wxpb\otimes \lr{\wx^{\otimes 3}}_{\perp}}\bigg]\\
    &\quad +\tau^4\inprod{\lr{\wx^{\otimes 3}}_{\perp}\otimes \wxpb}{\wxpb\otimes \lr{\wx^{\otimes 3}}_{\perp}}.
\end{align*}
Consequently, we write $I_5$ in the form
\begin{align*}
    I_5= I_5[0]+ 2iI_5[1]\tau + I_5[2]\tau^2 + 2iI_5[3]\tau^3 + I_5[4]\tau^4.
\end{align*}
We will only explicitly evaluate $I_5[0]$ here. Due to Wick's theorem, we consider all possible ways to pair up the variable $t$ and assign it either the mean or covariance. When all instances of the variable $t$ takes the value of the mean $i\wbXd\wtheta_*$, we acquire term
\begin{align*}
    &\inprod{\wX^{\otimes 3}\wbXd\wtheta_* \otimes \wX \wbXd \wtheta_*}{\wX \wbXd \wtheta_* \otimes \wX^{\otimes 3} \wbXd \wtheta_*}.
\end{align*}
This term appears directly in $I_5[0]$.
When two instances of $t$ are the mean, we have $\binom{4}{2}$ possible choices for choosing which pair of variables is the mean:
\begin{align*}
    -\inprod{\wX^{\otimes 3}t \otimes \wX \wbXd \wtheta_*}{\wX \wbXd \wtheta_* \otimes \wX^{\otimes 3} t} &= \sum_{\mu \nu} \lr{M_\beta}_{\mu\nu} \inprod{\wx_\mu^{\otimes 3} \otimes \wX \wbXd \wtheta_*}{\wX \wbXd \wtheta_* \otimes \wx_\nu^{\otimes 3}}
\end{align*}
and
\begin{align*}
    -\inprod{\wX^{\otimes 3}\wbXd \wtheta_* \otimes \wX t}{\wX t \otimes \wX^{\otimes 3} \wbXd \wtheta_*} &= -\sum_{\mu\nu} \lr{\wbXd \wtheta_*}_\mu \lr{\wbXd \wtheta_*}_\nu \inprod{\wx_\mu^{\otimes 2}}{\wx_\nu^{\otimes 2}} \wx_\nu^T \wX M_\beta \wX^T \wx_\mu\\
    &= -\norm{\sum_\mu \lr{\wbXd \wtheta_*}_\mu M_\beta^{1/2} \wX^T \wx_\mu \otimes \wx_\mu^{\otimes 2}}
\end{align*}
and two occurrences of the term
\begin{align*}
    -\inprod{\wX^{\otimes 3}t \otimes \wX t}{\wX \wbXd \wtheta_* \otimes \wX^{\otimes 3} \wbXd \wtheta_*} &= -\sum_{\mu\nu\rho} \inprod{\wx_\mu}{\wX \wbXd \wtheta_*} \lr{M_\beta}_{\mu\rho} \inprod{\wx_\rho}{\wx_\nu} \inprod{\wx_\mu^{\otimes 2}}{\wx_\nu^{\otimes 2}} \lr{\wbXd \wtheta_*}_\nu\\
    &= -\sum_{\mu\nu} \inprod{\wx_\mu}{\wX \wbXd \wtheta_*} \inprod{\wx_\mu^{\otimes 2}}{\wx_\nu^{\otimes 2}} \lr{\wbXd \wtheta_*}_\nu \lr{M_\beta \wX^T \wx_\nu}_\mu\\
    &= -\sum_\mu \inprod{\wx_\mu^{\otimes 3}}{\wX \wbXd \wtheta_* \otimes \sum_\nu \lr{\wbXd \wtheta_*}_\nu \lr{\wbXd \wx_\nu}_\mu \wx_\nu^{\otimes 2}}
\end{align*}
and two occurrences of the term
\begin{align*}
    -\inprod{\wX^{\otimes 3} t \otimes \wX \wbXd \wtheta_*}{\wX t \otimes \wX^{\otimes 3} \wbXd \wtheta_*} &= -\sum_{\mu\nu\rho} \inprod{\wx_\mu}{\wx_\rho} \lr{M_\beta}_{\mu\rho} \inprod{\wx_\mu^{\otimes 2}}{\wx_\nu^{\otimes 2}} \lr{\wbXd\wtheta_*}_\nu \inprod{\wX \wbXd \wtheta_*}{\wx_\nu}\\
    &= -\sum_{\mu\nu} \lr{M_\beta \wX^T \wx_\mu}_\mu \inprod{\wx_\mu^{\otimes 2}}{\wx_\nu^{\otimes 2}} \inprod{\wX \wbXd \wtheta_*}{\wx_\nu} \lr{\wbXd\wtheta_*}_\nu\\
    &= -\inprod{\sum_\mu \lr{\wbXd \x_\mu}_\mu \wx_\mu^{\otimes 2} \otimes \wX \wbXd \wtheta_*}{\wX^{\otimes 3} \wbXd \wtheta_*}.
\end{align*}
Finally, choosing all $t$ to take the covariance, we have 3 ways of pairing them. Pairing the first and second (and third and fourth) produces the term
\begin{align*}
    \sum_{\mu\nu\rho\sigma} \inprod{\wx_\mu^{\otimes 2}}{\wx_\nu^{\otimes 2}} \inprod{\wx_\mu}{\wx_\sigma} \inprod{\wx_\rho}{\wx_\nu} \lr{M_\beta}_{\mu\rho} \lr{M_\beta}_{\sigma\nu} &= \sum_{\mu\nu} \inprod{\wx_\mu^{\otimes 2}}{\wx_\nu^{\otimes 2}} \lr{\wbXd \wx_\mu}_\nu \lr{\wbXd \wx_\nu}_\mu,
\end{align*}
while pairing the first and third (and second and fourth) produces the term
\begin{align*}
    \sum_{\mu\nu\rho\sigma} \inprod{\wx_\mu^{\otimes 2}}{\wx_\nu^{\otimes 2}} \inprod{\wx_\mu}{\wx_\sigma} \inprod{\wx_\rho}{\wx_\nu} \lr{M_\beta}_{\mu\sigma} \lr{M_\beta}_{\nu\rho} &= \sum_{\mu\nu} \inprod{\wx_\mu^{\otimes 2}}{\wx_\nu^{\otimes 2}} \lr{\wbXd \wx_\mu}_\mu \lr{\wbXd \wx_\nu}_\nu,
\end{align*}
and pairing the first and fourth (and second and third) produces the term
\begin{align*}
    \sum_{\mu\nu\rho\sigma} \inprod{\wx_\mu^{\otimes 2}}{\wx_\nu^{\otimes 2}} \inprod{\wx_\mu}{\wx_\sigma} \inprod{\wx_\rho}{\wx_\nu} \lr{M_\beta}_{\mu\nu} \lr{M_\beta}_{\rho\sigma} &= \sum_{\mu\nu} \inprod{\wx_\mu^{\otimes 2}}{\wx_\nu^{\otimes 2}} \inprod{\wX \wbXd \wx_\mu}{\wx_\nu}.
\end{align*}
Adding all the above terms, we obtain
\begin{align*}
    I_5[0]&= \sum_{\mu\nu} \inprod{\wx_\mu^{\otimes 2}}{\wx_\nu^{\otimes 2}}\bigg[\inprod{\wX \wbXd \wx_\mu}{\wx_\nu}\lr{M_\beta}_{\mu\nu} + \lr{\wbXd \wx_\mu}_\mu\lr{\wbXd \wx_\nu}_\nu + \lr{\wbXd \wx_\mu}_\nu\lr{\wbXd \wx_\nu}_\mu\bigg]\\
    &\quad - 2 \sum_{\mu}\inprod{\wx_\mu^{\otimes 3}}{\wX\wbXd \wtheta_*\otimes \sum_\nu \lr{\wbXd \wtheta_*}_\nu\lr{\wbXd \wx_\nu}_\mu \wx_\nu^{\otimes 2}}\\
    &\quad - 2\inprod{\sum_\mu \lr{\wbXd \wx_\mu}_\mu \wx_\mu^{\otimes 2}\otimes \wX\wbXd\wtheta_*}{\wX^{\otimes 3}\wbXd\wtheta_*}\\
    &\quad-\sum_{\mu\nu} \lr{M_\beta}_{\mu\nu}\inprod{\wx_\mu^{\otimes 3}\otimes  \wX\wbXd\wtheta_*}{ \wX\wbXd\wtheta_*\otimes \wx_\nu^{\otimes 3}} -\norm{\sum_\mu \lr{\wbXd \wtheta_*}_\mu M_\beta^{1/2}X^T\wx_\mu \otimes \wx_\mu^{\otimes 2}}^2\\
    &\quad +\inprod{\wX^{\otimes 3}\wbXd \wtheta_*\otimes \wX\wbXd \wtheta_*}{\wX\wbXd \wtheta_*\otimes \wX^{\otimes 3}\wbXd\wtheta_*}
\end{align*}
as reported in the proposition statement. The same procedure as above applies to compute all the Gaussian integrals in the proposition.
\end{proof}

\section{Analysis of Evidence and Posterior}

\subsection{Review of Notation}\label{sec:notation}
Before presenting our results, we summarize the notation that has been introduced thus far. Our network has hidden layers of equal width $N$ and depth $L$; our dataset has $P$ examples of dimension $N_0$ with constant $P/N_0 < 1$. The network defined by
\begin{align}
    f(x;\theta) = \frac{1}{\sqrt{N_L}}W^{(L+1)}x^{(L)},\qquad x^{(\ell)}:=\begin{cases}
        \phi\lr{\frac{1}{\sqrt{N_{\ell-1}}}W^{(\ell)}x^{(\ell-1)}}\in \R^{N_{\ell}},&\quad \ell \geq 1\\
        x,&\quad \ell = 0
    \end{cases}
\end{align}
has weights initialized with variance $\sigma^2 = 1 + 2\eta/L$, and nonlinearity is of the form $\phi(t) = t + \psi t^3/3L$. The data is normalized such that $\norm{x}^2$ is order $N_0$. We introduce rescaled nonlinearities
\begin{align*}
    \wpsi :=\psi \frac{\exp[2\eta]-1}{2\eta}, \quad \wpsi_\mu := \wpsi \frac{\norm{x_\mu}^2}{N_0}
\end{align*}
and define rescaled data
\begin{align*}
    \wx:= e^{\eta}\lr{1-2\wpsi_\mu}^{-1/2}\frac{x}{\sqrt{N_0}}, \quad
    \wx^{\otimes 3} := e^{\eta}\lr{1-2\wpsi_\mu}^{-3/2}\frac{x^{\otimes 3}}{\sqrt{N_0}}.
\end{align*}
The hat variables now have norm $\norm{\wx}^2$ that is order 1. We analogously define $\wX$ and $\wX^{\otimes 3}$, e.g., $\widehat{X} := \lr{\wx_\mu,\, \mu=0,\ldots, P}$. We similarly introduce a rescaled interpolant $\wtheta_*$ such that $\wX^T\wtheta_* = Y$. Additional recurring quantities are
\begin{align*}
    \Sigma := \frac{1}{P}\wX \wX^T, \quad M_\beta := \lr{\frac{1}{\beta}I + \widehat{X}^T \widehat{X}}^{-1}, \quad \wbXd := M_\beta \wX^T.
\end{align*}
Below, we sometimes work with a rescaled temperature $B := \beta P$. Finally, the self-loop process led to quantities
\begin{align*}
    g_1(\eta) :=\frac{1}{2}\lr{1+\coth(\eta)-\frac{1}{\eta}}, \quad g_2(\eta) :=\frac{1}{4}\lr{\frac{\coth(\eta)}{\eta}-\csch^2(\eta)},
\end{align*}
where $g_1 > g_2$ for all $\eta$. These appear below both individually and in the constant
\begin{align}
\label{eq:cpsi}
    c_{\psi,\eta} := 2e^{-4\eta}\wpsi^2(g_1(\eta)-g_2(\eta))
\end{align}
which is positive when $\psi\neq 0$ and zero otherwise. This constant will appear in front of the leading-order corrections to the partition function in the presence of a nonlinearity.

\subsection{Data Assumptions}
\label{sec:data-ass}
We will evaluate the posterior and evidence under mild assumptions in the large $P, N_0$ limit. These assumptions permit the scaling of integrals $I_1, \dots, I_6$ in Proposition~\ref{prop:t-ints} to be made explicit, providing the leading-order contributions at zero and ultimately finite temperature.

The first assumption is a normalization
\begin{align*}
    \E{\norm{\wx_\mu}^2} = 1, \quad \E{y_\mu^2} = 1,
\end{align*}
where we emphasize that the data remains non-isotropic. Note that this implies the covariance matrix $\Sigma = \wX \wX^T / P$ has unit trace.

The second assumption is a self-averaging or concentration property. In well-behaved models, we expect that as $N, P_0 \to \infty$, quantities concentrate: for example, for two independent samples $\wx_\mu$ and $\wx_\nu$ from the training dataset, we may take $\inprod{\wx_\mu}{\wx_\nu} \approx \tr(\Sigma^2)$, and we take $\norm{\wx_\mu}^2 \approx \E{\norm{\wx_\mu}^2}$. (Later, we give an example of a concrete data model --- one with a power law covariance spectrum --- where these conditions hold.) Note that we do not rely on the self-averaging property to compute explicit quantities such as the posterior or model evidence. Rather, we use the self-averaging assumption to determine which quantities contribute to leading order.

\subsection{Analysis at Zero Temperature}\label{sec:zero-T-analysis}
We have the following evidence and posterior at zero temperature.

\begin{proposition}[Evidence]\label{prop:evidence}
Under the assumptions of \S \ref{sec:data-ass} and at zero temperature, the evidence is (neglecting an additive constant)
\begin{align}
    \log Z(0) &\simeq \frac{1}{2}\tr\log (\wX^T \wX)^{-1} - \frac{1}{2}\norm{\wtheta_*}^2 \nonumber \\
    &\quad + \frac{1}{4}\frac{LP}{N}\Bigg[\frac{1}{P}\norm{\wtheta_*}^4 - 2\norm{\wtheta_*}^2 + P\Bigg] \nonumber \\
    &\quad + \frac{c_{\psi,\eta}}{2}\frac{L}{N}\Bigg[\tr(\lr{\wX^T \wX}^2) - \sum_\mu \lr{\wX^T \wX}^2_{\mu\mu} y_\mu \lr{\wX^\dag\wtheta_*}_\mu \nonumber \\
    &\quad + \sum_{\mu\nu} y_\mu y_\nu \inprod{\wx_\mu}{\wx_\nu}^2 \lr{\wX^\dag\wtheta_*}_\mu \lr{\wX^\dag\wtheta_*}_\nu - \sum_{\mu\nu} \inprod{\wx_\mu}{\wx_\nu}^3 \lr{\wX^T \wX}^{-1}_{\mu\nu}\Bigg].
\end{align}
\end{proposition}

\begin{proposition}[Posterior]\label{prop:post}
Define vector $a = \wX^\dag \wx$ in $\R^P$; we note identities $\wx_{||} = \sum_\mu a_\mu \wx_\mu$ and $\sum_\mu a_\mu y_\mu = \wtheta_*^T \wx_{||}$. Under the assumptions of \S \ref{sec:data-ass} and at zero temperature, the posterior is given by, to leading order in small $\frac{LP}{N}$ and assuming $\norm{\wtheta_*}^2 = o(P)$,
\begin{align}
    i\frac{\partial \log Z}{\partial \tau}\Bigg|_{\tau=0} &= \sum_\mu a_\mu y_\mu \Big[1 + \frac{LP}{N} c_{\psi,\eta}\lr{\wx_\mu^T \Sigma \wx_\mu - \wx^T \Sigma \wx}\Big]\\
    -\frac{\partial^2 \log Z}{\partial \tau^2}\Bigg|_{\tau=0} &= \norm{\wx_\perp}^2\lr{1 - \frac{LP}{N}},
\end{align}
where the correction to the mean is $O(\frac{LP}{N} \tr \Sigma^2)$. If $\norm{\wtheta_*}^2/P$ is constant, we add
\begin{align}
    \frac{LP}{N}\frac{c_{\psi,\eta}}{P} \inprod{\lr{\wx^{\otimes 3} - \wX^{\otimes 3}\wX^\dag \wx} \otimes \wtheta_*}{\wtheta_* \otimes \wX^{\otimes 3} \wX^\dag \wtheta_*} - \inprod{\wX^{\otimes 3}\wX^\dag \wx}{\wX^{\otimes 3}\wX^\dag \wtheta_*}
\end{align}
to the mean, and we add
\begin{align}
    \norm{\wx_\perp}^2 \frac{LP}{N} \lr{\frac{\norm{\wtheta_*}^2}{P} + \frac{2c_{\psi,\eta}}{P}\sum_\mu \inprod{\wx}{\wx_\mu}^2 a_\mu y_\mu}
\end{align}
to the variance.
\end{proposition}

\subsubsection{Derivation of Evidence}
The log evidence is
\begin{align*}
    \log Z(0) &\simeq \frac{1}{2}\tr\log (2\pi M_\beta) - \norm{M_\beta^{1/2} \wX^T \wtheta_*}^2 + \frac{1}{2}\norm{\wX \wbXd \wtheta_*}^2 \\
    &\quad + \frac{L}{N}\bigg[e^{-4\eta}\wpsi^2 (g_1(\eta)-g_2(\eta))(I_5[0]-I_2[0]) - 2 \wpsi I_3[0] + \frac{e^{-2\eta}\wpsi}{2}g_1(\eta)I_4[0] + \frac{1}{4}I_6[0]\bigg].
\end{align*}
At zero temperature, we note that
\begin{align*}
    M_\beta &= \lr{\wX^T \wX}^{-1}\\
    \wbXd &= \lr{\wX^T \wX}^{-1}\wX^T\\
    \widehat{M} &= \diag\lr{\norm{\wx_\mu}^2} \diag\lr{1-2\wpsi_\mu}^{-2} \diag\lr{(1+\wpsi_\mu) g_1(\eta) - 3\wpsi_\mu g_2(\eta)}.
\end{align*}
To leading order, under our data assumptions of \S \ref{sec:data-ass}, the integrals of Proposition~\ref{prop:t-ints} scale as
\begin{align*}
    I_2[0] &= \tr(\Sigma^2) \lr{P - \norm{\wtheta_*}^2} + \tr \lr{\wX^T \wX}^{-1} - \norm{\wX^\dag \wtheta_*}^2\\
    I_3[0] &= P - \norm{\wtheta_*}^2\\
    I_4[0] &= -2\norm{\wtheta_*}^2\\
    I_5[0] &= \tr(\Sigma^2)\lr{P^2 - 2P\norm{\wtheta_*}^2 + \norm{\wtheta_*}^4} + \tr \lr{\wX^T \wX}^{-1} + 2P - 4\norm{\wtheta_*}^2\\
    I_6[0] &= 2\lr{\norm{\wtheta_*}^4 - P\norm{\wtheta_*}^2 + P^2}.
\end{align*}
Since $I_6$ dominates $I_3$ and $I_4$, the the log evidence scales as (up to a constant offset)
\begin{align*}
    \lim_{\beta\to\infty} \log Z(0) &= \frac{1}{2}\tr \log \lr{\wX^T \wX}^{-1} - \frac{1}{2}\norm{\wtheta_*}^2 \\
    &\quad + \frac{1}{2}\frac{LP}{N} \lr{P + \norm{\wtheta_*}^2 + \frac{1}{P}\norm{\wtheta_*}^4}\\
    &\quad + \frac{1}{2}\frac{LP}{N}c_{\psi,\eta} \left[\tr(\Sigma^2) \lr{P - 2\norm{\wtheta_*}^2 + \frac{1}{P}\norm{\wtheta_*}^4} - \frac{1}{P}\tr\lr{\wX^T \wX}^{-1}\right]
\end{align*}
to leading order. In the theorem statement, we recover the terms that provide these contributions.

\subsubsection{Derivation of Posterior}
At zero temperature, the terms contributing from the partition function are of the form
\begin{align*}
    \log Z &= -i\tau\wtheta_*^T \wx_{||} - \frac{\tau^2}{2}\norm{\wx_\perp}^2 - \frac{LP}{N}\lr{i a \tau + \frac{b}{2} \tau^2},
\end{align*}
which produces posterior
\begin{align*}
    i\frac{\partial \log Z}{\partial \tau}\Bigg|_{\tau=0} &= \wtheta_*^T \wx_{||} + \frac{LP}{N}a\\
    -\frac{\partial^2 \log Z}{\partial \tau^2}\Bigg|_{\tau=0} &= \norm{\wx_\perp}^2 + \frac{LP}{N}b.
\end{align*}
The quantities $a, b$ above are given by
\begin{align*}
    Pa &= 2\inprod{\wtheta_*}{\lr{\widehat{m} - \wX \widehat{M} \wX^\dag} \wx} - \frac{e^{-2\eta}\wpsi}{2}g_1(\eta)\lr{3\inprod{\norm{\wx_{||}}^2 \wx - \sum_\mu \lr{\wX^\dag \wx}_\mu \norm{\wx_\mu}^2 \wx_\mu}{\wtheta_*} - \inprod{(\wx^{\otimes 3})_\perp}{\wtheta_*^{\otimes 3}}}\\
    &\quad - 2e^{-4\eta}\wpsi^2(g_1(\eta)-g_2(\eta)) \Bigg[\sum_\nu \inprod{\inprod{\wx_{||}}{\wx_\nu} \wx^{\otimes 2} - \sum_\mu \lr{\wX^\dag \wx}_\mu \inprod{\wx_\mu}{\wx_\nu} \wx_\nu^{\otimes 2}}{\lr{\wX^\dag \wtheta_*}_\nu\wx_\nu^{\otimes 2}} \\
    &\qquad\quad + \sum_\nu \inprod{\lr{\lr{\wX^\dag \wx}_\nu \wx^{\otimes 2} - \sum_\mu \lr{\wX^\dag \wx}_\mu\lr{\wX^\dag \wx_\mu}_\nu\wx_\mu^{\otimes 2}}\otimes \wtheta_*}{\wx_\nu^{\otimes 3}}\\
    &\qquad \quad - \inprod{(\wx^{\otimes 3})_\perp \otimes \wtheta_*}{\wtheta_* \otimes \wX^{\otimes 3} \wX^\dag \wtheta_*} + \inprod{(\wx^{\otimes 3})_\perp}{\wtheta_* \otimes \sum_\mu \lr{\wX^\dag \wx_\mu}_\mu \wx_\mu^{\otimes 2} - \wX^{\otimes 3} \wX^\dag \wtheta_*}\Bigg]
\end{align*}
and
\begin{align*}
    \frac{Pb}{2} &= 2\widehat m \norm{\wx_\perp}^2 - \frac{1}{2}\norm{\wx_\perp}^2 \lr{P - \norm{\wtheta_*}^2} - 3\frac{e^{-2\eta}\wpsi}{2}g_1(\eta)\lr{\norm{\wx_{||}}^2 - \inprod{\wx}{\wtheta_*}^2} \norm{\wx_\perp}^2\\
    &\quad - e^{-4\eta}\wpsi^2(g_1(\eta)-g_2(\eta)) \Bigg[\sum_{\mu\nu}\lr{\wX^T \wX}^{-1}_{\mu\nu} \inprod{(\wx^{\otimes 3})_\perp \otimes \wx_\mu}{\wx_\nu \otimes (\wx^{\otimes 3})_\perp} - \inprod{(\wx^{\otimes 3})_\perp \otimes \wtheta_*}{\wtheta_* \otimes (\wx^{\otimes 3})_\perp}\\
    &\qquad\quad + 2\inprod{(\wx^{\otimes 3})_\perp}{\wx_\perp \otimes \sum_\mu\lr{\wX^\dag \wx_\mu}_\mu \wx_\mu^{\otimes 2}} - 2\inprod{(\wx^{\otimes 3})_\perp \otimes \wtheta_*}{\wx_\perp \otimes \wX^{\otimes 3}\wX^\dag \wtheta_*} - \norm{(\wx^{\otimes 3})_\perp}^2\Bigg]
\end{align*}
for $(\wx^{\otimes 3})_\perp = \wx^{\otimes 3} - \wX^{\otimes 3}\wX^\dag \wx$. To evaluate the mean, it suffices to use identities
\begin{align*}
    \lr{\wX^\dag \wx_\mu}_\nu = (\wX^T \wX)^{-1}_{\nu\rho} (\wX^T \wX)_{\rho\mu} = \delta_{\mu\nu}, \quad \sum_\mu \lr{\wX^\dag \wx}_\mu y_\mu = \wtheta_*^T \wX (\wX^T \wX)^{-1} \wX^T \wx = \wtheta_*^T \wx_{||}
\end{align*}
and
\begin{align*}
    &\sum_\nu \inprod{\wx_\mu}{\wx_\nu}\lr{\wX^\dag \wtheta_*}_\nu = \wx_\mu^T \wX (\wX^T \wX)^{-1} \wX^T \wtheta_* = y_\mu\\
    &\sum_\nu \inprod{\wtheta_*}{\wx_\nu} \lr{\wX^\dag \wtheta_*}_\nu = \wtheta_*^T \wX (\wX^T \wX)^{-1} \wX^T \wtheta_* = \norm{\wtheta_*}^2.
\end{align*}
This gives
\begin{align*}
    i\frac{\partial \log Z}{\partial \tau}\Bigg|_{\tau=0} &= \wtheta_*^T \wx_{||} + \frac{LP}{N} c_{\psi,\eta}\sum_\mu \lr{\wX^\dag \wx}_\mu y_\mu \lr{\wx_\mu^T \Sigma \wx_\mu - \wx^T \Sigma \wx}\\
    &\quad + \frac{L}{N} c_{\psi,\eta}\Bigg[\inprod{\lr{\wx^{\otimes 3} - \wX^{\otimes 3}\wX^\dag \wx} \otimes \wtheta_*}{\wtheta_* \otimes \wX^{\otimes 3} \wX^\dag \wtheta_*} - \inprod{\wX^{\otimes 3}\wX^\dag \wx}{\wX^{\otimes 3}\wX^\dag \wtheta_*}\Bigg],
\end{align*}
where the second line is dominated by the first line when $\norm{\wtheta_*}^2 = o(P)$.
To evaluate the variance, we also need identity
\begin{align*}
    \sum_\mu \lr{\wX^\dag \wx}_\mu \inprod{\wx}{\wx_\mu} &= \norm{\wx_{||}}^2
\end{align*}
to obtain
\begin{align*}
    -\frac{\partial^2 \log Z}{\partial \tau^2}\Bigg|_{\tau=0} &= \norm{\wx_\perp}^2\Bigg\{1 + \frac{L}{N}\left[-P + \norm{\wtheta_*}^2 + 2c_{\psi,\eta}\sum_\mu \inprod{\wx}{\wx_\mu}^2 y_\mu \sum_\nu \lr{\wX^T \wX}^{-1}_{\mu\nu} y_\nu\right]\Bigg\}.
\end{align*}
As an example of our analysis, we examine the first line of the mean, and the nontrivial leading-order term in the variance. The first line of the mean originates from the term
\begin{align*}
    \inprod{(\wx^{\otimes 3})_\perp}{\wtheta_* \otimes \sum_\mu \lr{\wX^\dag \wx_\mu}_\mu \wx_\mu^{\otimes 2}} &= \inprod{\wx^{\otimes 3} - \wX^{\otimes 3}\wX^\dag \wx}{\wtheta_* \otimes \sum_\mu \lr{\wX^\dag \wx_\mu}_\mu \wx_\mu^{\otimes 2}}\\
    &= \lr{\wtheta_*^T \wx_{||}} \sum_\mu \inprod{\wx}{\wx_\mu}^2 - \sum_{\mu\nu} \lr{\wX^\dag \wx}_\mu y_\mu \inprod{\wx_\mu}{\wx_\nu}^2.
\end{align*}
Under the self-averaging assumption $\inprod{\wx}{\wx_\mu}^2 \approx \tr(\Sigma^2)$, the first term is order $P \lr{\wtheta_*^T \wx_{||}} \tr(\Sigma^2)$. We decompose the second term into two cases. When $\mu = \nu$, it is of the form 
\begin{align*}
    \sum_\mu \lr{\wX^\dag \wx}_\mu y_\mu \norm{\wx_\mu}^4 \approx \E{\norm{\wx_\mu}^4} \sum_\mu \lr{\wX^\dag \wx}_\mu y_\mu = \E{\norm{\wx_\mu}^4} \lr{\wtheta_*^T \wx_{||}}
\end{align*}
due to the self-averaging assumption $\norm{\wx_\mu}^4 \approx \E{\norm{\wx_\mu}^4}$. When $\mu \neq \nu$, self-averaging assumes $\inprod{\wx_\mu}{\wx_\nu}^2 \approx \tr(\Sigma^2)$, giving
\begin{align*}
    \sum_{\mu \neq \nu} \lr{\wX^\dag \wx}_\mu y_\mu \inprod{\wx_\mu}{\wx_\nu}^2 \approx P\lr{\wtheta_*^T \wx_{||}}\tr(\Sigma^2).
\end{align*}
We conclude that the term contributes to leading order when $\tr(\Sigma^2) = \Omega(1)$. To rewrite it in the form reported above, we note that
\begin{align*}
    \lr{\wtheta_*^T \wx_{||}} \sum_\mu \inprod{\wx}{\wx_\mu}^2 - \sum_{\mu\nu} \lr{\wX^\dag \wx}_\mu y_\mu \inprod{\wx_\mu}{\wx_\nu}^2 &= \sum_{\mu\nu} \lr{\wX^\dag \wx}_\mu y_\mu \lr{\inprod{\wx}{\wx_\nu^2} -  \inprod{\wx_\mu}{\wx_\nu}^2}\\
    &= \sum_{\mu\nu} \lr{\wX^\dag \wx}_\mu y_\mu \lr{\wx^T \Sigma \wx - \wx_\mu^T \Sigma \wx_\mu}.
\end{align*}
The mean contains the negative of this term.

For the variance, we expand the term
\begin{align*}
    \inprod{(\wx^{\otimes 3})_\perp \otimes \wtheta_*}{\wx_\perp \otimes \wX^{\otimes 3}\wX^\dag \wtheta_*} &= \inprod{\lr{\wx^{\otimes 3} - \wX^{\otimes 3}\wX^\dag \wx} \otimes \wtheta_*}{\wx_\perp \otimes \wX^{\otimes 3}\wX^\dag \wtheta_*}\\
    &= \norm{\wx_\perp}^2 \sum_\mu \inprod{\wx}{\wx_\mu}^2 y_\mu \lr{\wX^\dag \wtheta_*}_\mu - 0
\end{align*}
and apply the self-averaging assumption $\inprod{\wx}{\wx_\mu}^2 \approx \tr(\Sigma^2)$ to obtain
\begin{align*}
    \inprod{(\wx^{\otimes 3})_\perp \otimes \wtheta_*}{\wx_\perp \otimes \wX^{\otimes 3}\wX^\dag \wtheta_*} &\approx \norm{\wx_\perp}^2 \tr(\Sigma^2) \sum_\mu y_\mu \lr{\wX^\dag \wtheta_*}_\mu.
\end{align*}
The identity $\sum_\nu \inprod{\wtheta_*}{\wx_\nu} \lr{\wX^\dag \wtheta_*}_\nu = \norm{\wtheta_*}^2$ provided above gives
\begin{align*}
    \inprod{(\wx^{\otimes 3})_\perp \otimes \wtheta_*}{\wx_\perp \otimes \wX^{\otimes 3}\wX^\dag \wtheta_*} &\approx \norm{\wx_\perp}^2 \tr(\Sigma^2) \norm{\wtheta_*}^2.
\end{align*}
Note that we can rewrite
\begin{align*}
    \lr{\wX^\dag \wtheta_*}_\mu = \sum_\nu \lr{\wX^T \wX}^{-1}_{\mu\nu} y_\nu
\end{align*}
to obtain the form
\begin{align*}
    \inprod{(\wx^{\otimes 3})_\perp \otimes \wtheta_*}{\wx_\perp \otimes \wX^{\otimes 3}\wX^\dag \wtheta_*} &= \norm{\wx_\perp}^2 \sum_\mu \inprod{\wx}{\wx_\mu}^2 y_\mu \sum_\nu \lr{\wX^T \wX}^{-1}_{\mu\nu} y_\nu
\end{align*}
reported in the variance. For the higher moments of the partition function ($\tau^3$ and $\tau^4$), a similar analysis shows that all terms vanish at zero temperature.

\subsection{Power Law Data Model}\label{sec:power}
At zero temperature, we reported the evidence and posterior under only a self-averaging assumption. Here, we introduce an explicit data model to provide more interpretable results in at finite temperature. We impose a power law spectrum in the covariance matrix and chose labels in a direction corresponding to a singular value.

\begin{definition}[$k$th-principal power law model]
\label{def:powerlaw}
Let $\wX \in \R^{N_0\times P}$ be the dataset reshaped by the nonlinearity, as defined in \S \ref{sec:notation}. Let $Y \in \R^P$ be the labels. Define the SVD decomposition $\wX = \sqrt{P} \sum_j \sqrt{\lambda_j} u_j v_j^T$. The \emph{$k$th-principal power law model} assumes input data satisfying, for all $j=1,\dots, P$,
\begin{align}
    \lambda_j \propto j^{-\alpha}
\end{align}
with $\alpha > 0$ and proportionality constant such that $\tr(\Sigma) = 1$, and labels $Y = \sqrt{P}v_k + \epsilon$ and for $\epsilon \in \R^P$ with i.i.d. entries sampled from $\mathcal{N}(0, \sigma_\epsilon^2)$.
\end{definition}
Note that the power law is defined over the reshaped dataset $\wX$ rather than the original dataset $X$; this choice is natural due to $LP/N=0$ being a kernel model in the reshaped variables. We will discard certain parameter regimes of the power law data model, since they produce a pathological scaling.
\begin{remark}
A $k$th-principal power law model satisfying either i) $\alpha < 1$ and $k\in [P]$ or ii) $\alpha > 1$ and $P^{1/\alpha} = o(k)$, satisfies $Pk^{-\alpha}/\sum_j j^{-\alpha} = o(1)$.
\end{remark}
When $\alpha < 1$ this property occurs since, in the large $N_0, P$ limit, the labeled direction $k$ provides a vanishingly small signal: the slow decay causes additional data to extensively dilute the correlation between input data and label. For $\alpha > 1$, a similar property holds when $k$ is too large. This behavior prompts the following refinement of the original data model.
\begin{definition}[Non-vanishing $k$th-principal power law model]
A $k$th-principal power law model is \emph{non-vanishing} if $\alpha > 1$ and $k = o(P^{1/\alpha})$.
\end{definition}
Among non-vanishing power law models, we observe that if the variance $\se^2$ of the label noise grows too large, the size of the minimum-norm interpolant is $\norm{\wtheta_*}^2 = k^\alpha + \se^2 P^{-1+\alpha}$. Since the signal can have $k \sim 1$, the noise can dominate the interpolant when $\se^2 > P^{1-\alpha}$; for this reason, we will consider label noise of variance at most $\se^2 = O(P^{1-\alpha})$. Altogether, these observations and the zero-temperature results of \S \ref{sec:zero-T-analysis} produce the following corollary.

\begin{corollary}[Scaling of $LP/N$ correction at zero temperature]\label{cor:zero-T-scaling}
Consider the $k$th-principal power law model with $k = o(P^{1/\alpha})$. Linear and nonlinear networks receive the following $LP/N$ corrections.
\begin{itemize}
    \item Neglecting logarithmic factors, when $\se^2 = O(P^{1-\alpha})$, the linear network at zero temperature has evidence of the form
    \begin{align*}
        \log Z_\mathrm{lin}(0, LP/N>0) &= \log Z_\mathrm{lin}(0, LP/N=0)\left[1 + \frac{LP}{N} \times (\text{linear correction})\right]
    \end{align*}
    for data-dependent factor $(\text{linear correction}) = O(1)$. At $LP/N>0$, the mean is unchanged and the variance acquires an $O(LP/N)$ correction.
    \item Neglecting logarithmic factors, when $1 < \alpha < 2$ and $\se^2=0$, the nonlinear network at zero temperature has evidence of the form
    \begin{align*}
        \log Z(0, LP/N>0) &= \log Z_\mathrm{lin}(0, LP/N=0)\Bigg[1 + \frac{LP}{N} \times (\text{linear correction}) \\
        &\qquad\qquad\qquad\qquad\qquad\quad + c_{\psi,\eta} \tr(\Sigma^2) \frac{LP}{N} \times (\text{nonlinear correction})\Bigg]
    \end{align*}
    for data-dependent factor $(\text{nonlinear correction}) = O(1)$. At $LP/N>0$, the mean acquires an $O(c_{\psi,\eta} \tr(\Sigma^2) LP/N)$ correction, and the variance acquires an $O(LP/N)$ correction.
\end{itemize}
\end{corollary}

\subsection{Analysis at Positive Temperature}\label{sec:gen-T-analysis}
We now turn to finite-temperature computations of the evidence and posterior.

\begin{proposition}[Evidence]
Let $\wX, Y$ be generated by a non-vanishing $k$th-principal power law model, i.e., with $k < P^{1/\alpha}$ and label noise $\se^2$. Define (neglecting an additive constant)
\begin{align}
    \log Z(0) &\simeq \frac{1}{2}\tr\log (M_\beta) + \frac{1}{2\beta} \norm{\wbXd \wtheta_*}^2 - \norm{M_\beta^{1/2} \wX^T \wtheta_*}^2 + \frac{1}{2}\norm{\wX\wbXd\wtheta_*}^2 \nonumber \\
    &\quad + \frac{1}{4}\frac{L}{N}\Bigg[\norm{\wX\wbXd\wtheta_*}^4 - 2\norm{\wX\wbXd\wtheta_*}^2\sum_{\mu\nu} \lr{M_\beta}_{\mu\nu} \inprod{\wx_\mu}{\wx_\nu} + \tr(\wbXd\wX)^2\Bigg] \nonumber \\
    &\quad + \frac{c_{\psi,\eta}}{2}\frac{L}{N}\Bigg[\sum_{\mu\nu} \inprod{\wx_\mu}{\wx_\nu}^2 \lr{\wbXd \wx_\mu}_\mu \lr{\wbXd \wx_\nu}_\nu \nonumber \\
    &\qquad - 2\inprod{\sum_\mu \lr{\wbXd \wx_\mu}_\mu \wx_\mu^{\otimes 2} \otimes \wX\wbXd\wtheta_*}{\wX^{\otimes 3}\wbXd\wtheta_*} \nonumber \\
    &\qquad + \inprod{\wX^{\otimes 3}\wbXd\wtheta_*\otimes \wX\wbXd\wtheta_*}{\wX\wbXd\wtheta_* \otimes \wX^{\otimes 3}\wbXd\wtheta_*} \nonumber \\
    &\qquad - \sum_{\mu\nu} \inprod{\wx_\mu}{\wx_\nu}^2 \lr{M_\beta}_{\mu\nu} \lr{\inprod{\wx_\mu \otimes \wX\wbXd\wtheta_*}{\wX\wbXd\wtheta_* \otimes \wx_\nu} + \inprod{\wx_\mu}{\wx_\nu}}\Bigg].
\end{align}
For a linear network ($c_{\psi,\eta}=0$), $Z(0)$ is the evidence for label noise at most $\se^2 = O(P^{1-\alpha})$. For a nonlinear network ($c_{\psi,\eta}\neq 0$), $Z(0)$ is the evidence when $\alpha < 2$ and $\se^2=0$.
\end{proposition}
Since the perturbative expansion holds for linear neural networks with noisy labels, we maximize evidence with respect to temperature in this model. When the optimal temperature is $\beta\to\infty$, the model chooses to overfit the training data; we refer to this as benign overfitting. We show that increasing the depth improves the evidence in the regime of benign overfitting.

\begin{corollary}[Bayesian benign overfitting for linear network]
\label{cor:bayes-benign}
Let $\wX, Y$ be generated by a non-vanishing $k$th-principal power law model, i.e., with $k = o(P^{1/\alpha})$ and label noise $\se^2 = O(P^{1-\alpha})$. Assume a linear network ($c_{\psi,\eta}=0$). The following facts hold.
\begin{itemize}
    \item For all temperatures $\beta P > P^\alpha$, the log evidence is equivalent up to constant factors.
    \item The evidence is maximized at a temperature satisfying $\beta P > P^\alpha$. Moreover, 
    \begin{align*}
        \log Z(0, \beta P < P^\alpha) < \log Z(0, \beta P > P^\alpha).
    \end{align*}
    \item At any temperature, the evidence increases with respect to $LP/N$.
\end{itemize}
In particular, omitting some positive constants, the following holds:
\begin{align*}
    \log Z(0) &\simeq \begin{cases}
        (\alpha-1) P \log P - \se^2 P^\alpha - k^\alpha + \frac{L}{N}\left[P^2 + \se^4 P^{2\alpha} + k^{2\alpha}\right] & B > P^\alpha, k < B^{1/\alpha}\\
        P \log B/P + \se^2 B - k^\alpha + \frac{L}{N}\left[B^{2/\alpha} + \se^4P^{-2}B^{2+2/\alpha} + k^{2\alpha}\right] & B < P^\alpha, k < B^{1/\alpha}\\
        P \log B/P + \se^2 B - B + \frac{L}{N}\left[B^{2/\alpha} + \se^4P^{-2}B^{2+2/\alpha} + B^4 k^{-2\alpha}\right] & B < P^\alpha, k > B^{1/\alpha},
    \end{cases}
\end{align*}
where the depth correction is $O(LP/N)$.
\end{corollary}

The above result refers to benign overfitting in a Bayesian sense, where zero temperature maximizes evidence. We also examine benign overfitting in terms of generalization error.

\begin{proposition}[Generalization error benign overfitting for linear network]
\label{prop:generror}
Let $\wX, Y$ be generated by a non-vanishing $k$th-principal power law model, i.e., with $k = o(P^{1/\alpha})$ and label noise $\se^2 = o(P^{1-\alpha})$. Assume a linear network ($c_{\psi,\eta}=0$). For generalization error $E = \langle \lr{f(x) - y}^2 \rangle$, the following facts hold.
\begin{itemize}
    \item For all temperatures $\beta P > P^\alpha$, the generalization error is equivalent up to constant factors.
    \item The generalization error is minimized at a temperature satisfying $\beta P > P^\alpha$. Moreover, 
    \begin{align*}
        E_{\beta P > P^\alpha} < E_{\beta P < P^\alpha}.
    \end{align*}
    \item At any temperature, the generalization decreases with respect to $LP/N$.
\end{itemize}
\end{proposition}
We see that the Bayesian sense of benign overfitting coincides here with the traditional definition in terms of generalization error: both the evidence and generalization error agree on when to overfit, and depth affects both quantities similarly. (We remark that the generalization error result is only given for label noise $\se^2 = o(P^{1-\alpha})$ as opposed to $\se^2=O(P^{1-\alpha})$, due to our approximation scheme being too coarse to ascertain the signs of relevant constants when $\se^2 \sim P^{1-\alpha}$; however, we expect the evidence and generalization error to agree even when $\se^2 \sim P^{1-\alpha}$.)

Finally, we address zero-temperature versus finite-temperature evidence for a nonlinear network. The perturbative expansion holds for power laws with $1 < \alpha < 2$; we also take $\se^2=0$. Much like the linear network, we find that the evidence is maximized at zero temperature, and it increases with depth. 

\begin{corollary}[Finite-temperature nonlinear network evidence]
\label{cor:nonlin-benign}
Let $\wX, Y$ be generated by a non-vanishing $k$th-principal power law model, i.e., with $k < P^{1/\alpha}$; moreover, assume the power law spectrum has exponent $1 < \alpha < 2$, and assume noiseless labels $\se^2 = 0$. The following facts hold.
\begin{itemize}
    \item For all temperatures $\beta P > P^\alpha$, the log evidence is equivalent up to constant factors.
    \item The evidence is maximized at a temperature satisfying $\beta P > P^\alpha$. Moreover, 
    \begin{align*}
        \log Z(0, \beta P < P^\alpha) < \log Z(0, \beta P > P^\alpha).
    \end{align*}
    \item At any temperature, the evidence increases with respect to $LP/N$.
\end{itemize}
In particular, omitting some positive constants, the following holds:
\begin{align*}
    \log Z(0) &\simeq \begin{cases}
        (\alpha-1) P \log P - k^\alpha + \frac{L}{N}\left[\lr{1+c_{\psi,\eta}}\lr{P^2 + k^{2\alpha}} - c_{\psi,\eta} P^\alpha\right] & B > P^\alpha, k < B^{1/\alpha}\\
        P \log B/P - k^\alpha + \frac{L}{N}\left[\lr{1+c_{\psi,\eta}}\lr{B^{2/\alpha} + k^{2\alpha}} - c_{\psi,\eta}B\right] & B < P^\alpha, k < B^{1/\alpha}\\
        P \log B/P - B + \frac{L}{N}\left[\lr{1+c_{\psi,\eta}}\lr{B^{2/\alpha} + B^4 k^{-2\alpha}} - c_{\psi,\eta}B\right] & B < P^\alpha, k > B^{1/\alpha},
    \end{cases}
\end{align*}
where the depth correction is $O(LP/N)$.
\end{corollary}

We provide a similar phase diagram to the zero temperature result shown in Fig.~\ref{fig:zero-temp-phase}. Below, we report the phase diagram of the nonlinear network at arbitrary temperature to leading order (Fig.~\ref{fig:finite-temp-phase}).

\begin{figure}[h]
     \centering
     \begin{subfigure}[b]{0.49\textwidth}
         \centering
         \includegraphics[width=\textwidth]{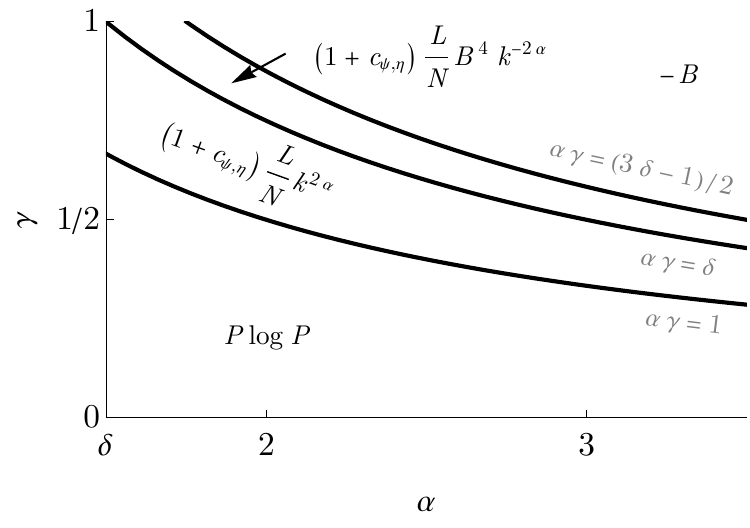}
         \caption{$1 < \delta < 2$ and $\alpha > \delta$}
     \end{subfigure}
     \hfill
     \begin{subfigure}[b]{0.49\textwidth}
         \centering
         \includegraphics[width=\textwidth]{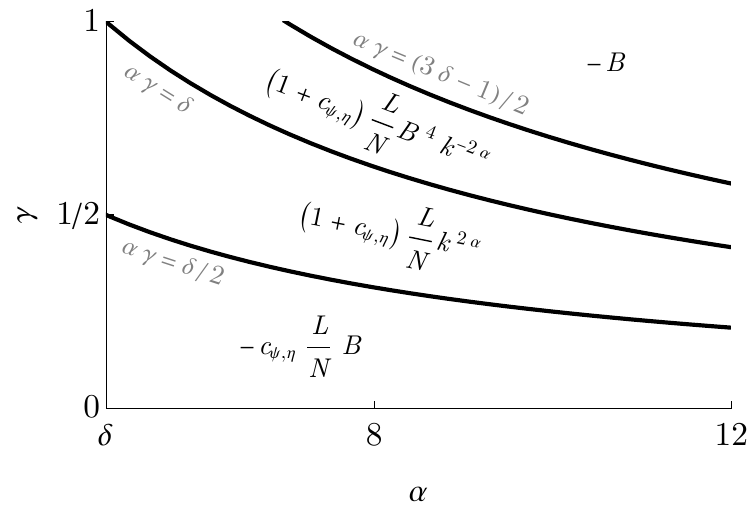}
         \caption{$\delta > 2$ and $\alpha > \delta$}
     \end{subfigure}
        \caption{Phase diagram of the leading-order log evidence of a deep nonlinear network at temperature $B:=\beta P = P^\delta$ (neglecting positive constants). The dataset covariance matrix has a power law spectrum $\lambda_j \sim j^{-\alpha}$ \eqref{eq:power-law} and the label vector lies in the $k$th direction \eqref{eq:power-label} for $k = P^\gamma$. When $\alpha < \delta$, the network is in the zero-temperature regime shown in Fig.~\ref{fig:zero-temp-phase}. We only show here solutions for $\delta > 1$; when $\delta < 1$, the log evidence is $-P \log P$ to leading order for all $\alpha, \gamma$.}
        \label{fig:finite-temp-phase}
\end{figure}

\subsubsection{Derivation of Evidence}
The log evidence is
\begin{align*}
    \log Z(0) &\simeq \frac{1}{2}\tr\log (2\pi M_\beta) - \norm{M_\beta^{1/2} \wX^T \wtheta_*}^2 + \frac{1}{2}\norm{\wX \wbXd \wtheta_*}^2 \\
    &\quad + \frac{L}{N}\bigg[e^{-4\eta}\wpsi^2 (g_1(\eta)-g_2(\eta))(I_5[0]-I_2[0]) - 2 \wpsi I_3[0] + \frac{e^{-2\eta}\wpsi}{2}g_1(\eta)I_4[0] + \frac{1}{4}I_6[0]\bigg].
\end{align*}
We begin by applying the self-averaging data assumption to obtain the scaling of terms:
\begin{align*}
    I_2[0] &= \tr(\Sigma^2)\lr{\tr(\wbXd \wX) - \norm{\wX \wbXd\wtheta_*}^2} + \E{\norm{\wx_\mu}^6}\lr{\tr(M_\beta) - \norm{\wbXd\wtheta_*}^2}\\
    I_3[0] &= \E{\widehat{M}_{\mu\mu}} \lr{\tr(\wbXd \wX) - \norm{\wX \wbXd \wtheta_*}^2}\\
    I_4[0] &= \frac{1}{P}\norm{\wX\wbXd\wtheta_*}^2 \norm{\wX^T\wX\wbXd\wtheta_*}^2 - 3\tr(\wbXd \wX)\frac{\norm{\wX^T\wX\wbXd\wtheta_*}^2}{P} \\
    &\quad + \frac{3}{P} \tr(M_\beta(\wX^T \wX)^2) \left[-\norm{\wX \wbXd \wtheta_*}^2+\tr(\wbXd \wX) \right]\\
    I_6[0] &= \norm{\wX \wbXd \wtheta_*}^4 - 2\norm{\wX \wbXd \wtheta_*}^2 \tr(\wbXd \wX) - 4 \norm{M_\beta^{1/2} \wX^T \wX \wbXd \wtheta_*}^2 + \tr(\wbXd\wX)^2 + 2 \tr((\wbXd \wX)^2)
\end{align*}
and
\begin{align*}
    I_5[0] &= \tr(\Sigma^2) \left[2\tr((\wbXd \wX)^2) + \tr(\wbXd \wX)^2\right] \\
    &\quad + \frac{\E{\norm{\wx_\mu}^4}}{P}\left[\tr(M_\beta)\tr(\wX^T \wX \wbXd \wX) + 2\tr(\wbXd \wX)^2\right]\\
    &\quad - 2\tr(\Sigma^2)\norm{M_\beta^{1/2} \wX^T\wX\wbXd\wtheta_*}^2 - \frac{2}{P} \E{\norm{\wx_\mu}^4} \norm{\wX \wbXd \wtheta_*}^2 \tr(\wbXd \wX)\\
    &\quad - 2\lr{\tr(\Sigma^2)+\frac{\E{\norm{\wx_\mu}^4}}{P}}\tr(\wbXd \wX) \norm{\wX \wbXd \wtheta_*}^2\\
    &\quad -\tr(\Sigma^2) \norm{M_\beta^{1/2}\wX^T \wX \wbXd \wtheta_*}^2 - \frac{\E{\norm{\wx_\mu}^4}}{P}\tr(M_\beta) \norm{\wX^T \wX \wbXd \wtheta_*}^2\\
    &\quad - \tr(\Sigma^2)\norm{\wX\wbXd\wtheta_*}^2 - \E{\norm{\wx_\mu}^6} \norm{\wbXd\wtheta_*}^2\\
    &\quad + \lr{\tr(\Sigma^2) + \frac{\E{\norm{\wx_\mu}^4}}{P}} \norm{\wX \wbXd \wtheta_*}^4.
\end{align*}
We normalize the temperature as $B = \beta P$ and apply the results of Appendix~\ref{app:quantities} to determine which terms survive to leading order in the non-vanishing $k$th-principal power law model. We now examine the evidence of linear networks (with label noise $\se^2 > 0$) and nonlinear networks (with $\se^2=0$) in three regimes:
\begin{enumerate}
	\item $k < B^{1/\alpha}$ and $B > P^\alpha$ (low temperature, $k < P^{1/\alpha}$)
	\item $k < B^{1/\alpha}$ and $B < P^\alpha$ (high temperature, strong signal $k \leq \min\lr{B^{1/\alpha}, P^{1/\alpha}}$)
	\item $k > B^{1/\alpha}$ and $B < P^\alpha$ (high temperature, weak signal $B^{1/\alpha} \leq k < P^{1/\alpha}$).
\end{enumerate}
\textbf{Linear network.} Terms $I_2, \dots, I_5$ do not contribute to the linear network. The contribution from $I_6$ scales as
\begin{align*}
    I_6[0] &= \begin{cases}
		\se^4 P^{-2}B^{2+2/\alpha} - 4\se^2 P^{-1}B^{1+1/\alpha} + B^{2/\alpha} + 2B^{1/\alpha} & B < P^\alpha\\
		\se^4 P^{2\alpha} - 4\se^2 P^\alpha + P^2 + 2P & B > P^\alpha
	\end{cases} \\
    &\quad + \begin{cases}
		k^{2\alpha} - 4k^\alpha & k < B^{1/\alpha}\\
		B^{4}k^{-2\alpha} - 4B^3 k^{-2\alpha} & k \geq B^{1/\alpha}
	\end{cases}.
\end{align*}
This expansion allows us to identify the leading-order terms of $I_6$:
\begin{align*}
     I_6[0] = \norm{\wX\wbXd\wtheta_*}^4 - 2\norm{\wX\wbXd\wtheta_*}^2\sum_{\mu\nu} \lr{M_\beta}_{\mu\nu} \inprod{\wx_\mu}{\wx_\nu} + \tr(\wbXd\wX)^2.
\end{align*}
Performing similar expansions for the initial terms of the partition function using the results of Appendix~\ref{app:quantities}, we obtain the leading-order scaling of the partition function. The three regimes are dominated by (omitting some positive constants)
\begin{enumerate}
	\item $\log Z(0) \simeq (\alpha-1) P \log P - \se^2 P^\alpha - k^\alpha + \frac{L}{N}\left[P^2 + \se^4 P^{2\alpha} + k^{2\alpha}\right]$
	\item $\log Z(0) \simeq P \log B/P + \se^2 B - k^\alpha + \frac{L}{N}\left[B^{2/\alpha} + \se^4P^{-2}B^{2+2/\alpha} + k^{2\alpha}\right]$
	\item $\log Z(0) \simeq P \log B/P + \se^2 B - B + \frac{L}{N}\left[B^{2/\alpha} + \se^4P^{-2}B^{2+2/\alpha} + B^4 k^{-2\alpha}\right]$.
\end{enumerate}
This perturbative expansion holds when $k = o(P^{1/\alpha})$ and $\se^2 = O(P^{1-\alpha})$.
To obtain Corollary~\ref{cor:bayes-benign}, we observe that $B > P^\alpha$ is optimal for all $\se^2 = O(P^{1-\alpha})$, and that the derivative of $\log Z(0)$ with respect to $LP/N$ is positive.

\textbf{Nonlinear network.} We evaluate the model in the absence of label noise ($\se^2=0$). The expansions of Appendix~\ref{app:quantities} give dominant terms
\begin{align*}
	I_5[0] &= \begin{cases}
		\tr(\Sigma^2)B^{2/\alpha} + \E{\norm{\wx_\mu}^4}B & B < P^\alpha\\
		\tr(\Sigma^2)P^2 + \E{\norm{\wx_\mu}^4}P^\alpha & B > P^\alpha
	\end{cases} \\
	&\quad - 2\tr(\Sigma^2)\lr{\begin{cases}
		B^{1/\alpha} & B < P^\alpha\\
		P & B > P^\alpha
	\end{cases}}\lr{\begin{cases}
		k^\alpha & k < B^{1/\alpha}\\
		B^2k^{-\alpha} & k \geq B^{1/\alpha}
	\end{cases}}\\
	&\quad -\E{\norm{\wx_\mu}^4}\lr{\begin{cases}
		B & B < P^\alpha\\
		P^{\alpha} & B > P^\alpha
	\end{cases}}\lr{\begin{cases}
		1 & k < B^{1/\alpha}\\
		B^2k^{-2\alpha} & k \geq B^{1/\alpha}
	\end{cases}}\\
	&\quad + \lr{\tr(\Sigma^2) + \frac{\E{\norm{\wx_\mu}^4}}{P}}\begin{cases}
		k^{2\alpha} & k < B^{1/\alpha}\\
		B^4 k^{-2\alpha} & k > B^{1/\alpha}
	\end{cases}
\end{align*}
coming from terms
\begin{align*}
	I_5[0] &= \sum_{\mu\nu} \inprod{\wx_\mu}{\wx_\nu}^2 \lr{\wbXd \wx_\mu}_\mu \lr{\wbXd \wx_\nu}_\nu - 2\inprod{\sum_\mu \lr{\wbXd \wx_\mu}_\mu \wx_\mu^{\otimes 2} \otimes \wX\wbXd\wtheta_*}{\wX^{\otimes 3}\wbXd\wtheta_*} \\
    &\quad - \sum_{\mu\nu} \inprod{\wx_\mu}{\wx_\nu}^2 \lr{M_\beta}_{\mu\nu} \inprod{\wx_\mu\otimes\wX\wbXd\wtheta_*}{\wX\wbXd\wtheta_*\otimes \wx_\nu} \\
    &\quad + \inprod{\wX^{\otimes 3}\wbXd\wtheta_*\otimes \wX\wbXd\wtheta_*}{\wX\wbXd\wtheta_* \otimes \wX^{\otimes 3}\wbXd\wtheta_*}.
\end{align*}
Similar analysis for the remaining integrals show that $I_3$ and $I_4$ are dominated by $I_5$ and $I_6$. However, $I_2$ contributes the term
\begin{align*}
    I_2[0] &= \sum_{\mu\nu} \inprod{\wx_\mu}{\wx_\nu}^3 \lr{M_\beta}_{\mu\nu} = \E{\norm{\wx_\mu}^6}\begin{cases}
		P^\alpha & B > P^\alpha\\
		B & B < P^\alpha
	\end{cases}
\end{align*}
to leading order.
We obtain a perturbative expansion that holds for $k = o(P^{1/\alpha})$ and $\alpha < 2$. The three regimes are dominated by (omitting some positive constants)
\begin{enumerate}
	\item $\log Z(0) \simeq (\alpha-1) P \log P - k^\alpha + \frac{L}{N}\left[\lr{1+c_{\psi,\eta}}\lr{P^2 + k^{2\alpha}} - c_{\psi,\eta} P^\alpha\right]$
	\item $\log Z(0) \simeq P \log B/P - k^\alpha + \frac{L}{N}\left[\lr{1+c_{\psi,\eta}}\lr{B^{2/\alpha} + k^{2\alpha}} - c_{\psi,\eta}B\right]$
	\item $\log Z(0) \simeq P \log B/P - B + \frac{L}{N}\left[\lr{1+c_{\psi,\eta}}\lr{B^{2/\alpha} + B^4 k^{-2\alpha}} - c_{\psi,\eta}B\right]$.
\end{enumerate}
To obtain Corollary~\ref{cor:nonlin-benign}, we maximize the evidence with respect to $B$. For all $B < P^\alpha$, the second and third regimes have smaller evidence than the first regime; hence, evidence is maximized by all $B > P^\alpha$. To show that the derivative with respect to $LP/N$ is positive, we simply impose the condition that $\alpha < 2$.

In total, the $LP/N$ correction to the evidence is dominated by terms from $I_5$ (for the nonlinear network) and $I_6$ (for both linear and nonlinear networks):
\begin{align*}
    \log Z(0) &\simeq \frac{1}{2}\tr\log (M_\beta) + \frac{1}{2\beta} \norm{\wbXd \wtheta_*}^2 - \norm{M_\beta^{1/2} \wX^T \wtheta_*}^2 + \frac{1}{2}\norm{\wX\wbXd\wtheta_*}^2 \\
    &\quad + \frac{1}{4}\frac{L}{N}\Bigg[\norm{\wX\wbXd\wtheta_*}^4 - 2\norm{\wX\wbXd\wtheta_*}^2\sum_{\mu\nu} \lr{M_\beta}_{\mu\nu} \inprod{\wx_\mu}{\wx_\nu} + \tr(\wbXd\wX)^2\Bigg] \\
    &\quad + \frac{c_{\psi,\eta}}{2}\frac{L}{N}\Bigg[\sum_{\mu\nu} \inprod{\wx_\mu}{\wx_\nu}^2 \lr{\wbXd \wx_\mu}_\mu \lr{\wbXd \wx_\nu}_\nu \\
    &\qquad - 2\inprod{\sum_\mu \lr{\wbXd \wx_\mu}_\mu \wx_\mu^{\otimes 2} \otimes \wX\wbXd\wtheta_*}{\wX^{\otimes 3}\wbXd\wtheta_*}  \\
    &\qquad + \inprod{\wX^{\otimes 3}\wbXd\wtheta_*\otimes \wX\wbXd\wtheta_*}{\wX\wbXd\wtheta_* \otimes \wX^{\otimes 3}\wbXd\wtheta_*} \\
    &\qquad - \sum_{\mu\nu} \inprod{\wx_\mu}{\wx_\nu}^2 \lr{M_\beta}_{\mu\nu} \lr{\inprod{\wx_\mu \otimes \wX\wbXd\wtheta_*}{\wX\wbXd\wtheta_* \otimes \wx_\nu} + \inprod{\wx_\mu}{\wx_\nu}}\Bigg].
\end{align*}

\subsubsection{Derivation of Posterior}
We show Proposition~\ref{prop:generror}: for the non-vanishing $k$th principal model with label noise, we must evaluate the generalization error of a deep linear model at arbitrary temperature. The $LP/N$ correction to the partition function of the linear network only requires $I_6$. The $\tau$-dependent terms (i.e., excluding the evidence $Z(0)$) are of the form
\begin{align*}
    \log Z &= -i\tau m_1 - \frac{1}{2}\tau^2 m_2 + i\tau m_3 + \tau^4 m_4
\end{align*}
for
\begin{align*}
    m_1 &= -\frac{P}{B} \inprod{\wbXd \wtheta_*}{\wbXd \wx} + \inprod{\wX\wbXd\wtheta_*}{\wx} \\
    &\quad + \frac{L}{N}\lr{\inprod{\wX\wbXd\wtheta_*}{\wxpb}\lr{\norm{\wX\wbXd\wtheta_*}^2 - \tr(\wbXd \wX)} - 2\inprod{\wX\wbXd\wxpb}{\wX\wbXd\wtheta_*}}\\
    m_2 &= \frac{P}{B}\norm{\wbXd \wx}^2 + \norm{\wxpb}^2 \\
    &\quad + \frac{L}{N}\lr{\norm{\wxpb}^2\lr{\norm{\wX\wbXd\wtheta_*}^2 - \tr(\wbXd \wX)}-2\inprod{\wX\wbXd\wxpb}{\wxpb} + 2\inprod{\wX\wbXd\wtheta_*}{\wxpb}^2}\\
    m_3 &= \frac{L}{N}\inprod{\wX\wbXd \wtheta_*}{\wxpb}\norm{\wxpb}^2\\
    m_4 &= \frac{1}{4}\frac{L}{N}\norm{\wxpb}^4,
\end{align*}
where $\wxpb = \wx - \wX \wbXd \wx$.
We determine the scaling of the terms using the expansions of Appendix~\ref{app:quantities}. We use the SVD decomposition $\wX = \sqrt{P} U D^{1/2} V^T$ and define column vectors $u_i \in \R^{N_0}$ and $v_i \in \R^P$ of $U$ and $V$. For label vector $Y = \sqrt{P}v_k + \epsilon$ and $\epsilon \sim \mathcal{N}(0, \se^2)$, we find
\begin{align*}
    m_1 &= -B^{-1}\inprod{\wx_{||}}{u_k} \begin{cases}
        k^{3\alpha/2} & k < B^{1/\alpha}\\
        B^2 k^{-\alpha/2} & k \geq B^{1/\alpha}
    \end{cases}\\
    &\quad + \sum_j \lr{\delta_{j=k} + P^{-1/2}\inprod{\epsilon}{v_j}} \inprod{\wx_{||}}{u_j} \begin{cases}
        j^{\alpha/2} & j < B^{1/\alpha}\\
        Bj^{-\alpha/2} & j > B^{1/\alpha}
    \end{cases}\\
    &\quad + \frac{L}{N}\lr{\sum_{j>B^{1/\alpha}} Bj^{-\alpha/2} \lr{\delta_{j=k} + P^{-1/2}\inprod{v_j}{\epsilon}} \inprod{\wx_{||}}{u_j}}\\
    &\qquad\qquad \times\lr{\begin{cases}
        k^\alpha & k < B^{1/\alpha}\\
        B^2 k^{-\alpha} & k \geq B^{1/\alpha}
    \end{cases} - \begin{cases}
        P - \se^2 P^\alpha & B > P^\alpha\\
        B^{1/\alpha} - \se^2 P^{-1} B^{1+1/\alpha} & B < P^\alpha
    \end{cases}}
\end{align*}
and
\begin{align*}
    m_2 &= B^{-1} \sum_j \inprod{\wx_{||}}{u_j}^2 \begin{cases}
        j^\alpha & j < B^{1/\alpha}\\
        B^2 j^{-\alpha} & j \geq B^{1/\alpha}
    \end{cases} \\
    &\quad \lr{\norm{\wx_\perp}^2 + \sum_{j > B^{1/\alpha}}^P \inprod{\wx_{||}}{u_j}^2}\lr{1 + \frac{L}{N}\lr{\begin{cases}
        k^\alpha & k < B^{1/\alpha}\\
        B^2 k^{-\alpha} & k \geq B^{1/\alpha}
    \end{cases} - \begin{cases}
        P - \se^2 P^\alpha & B > P^\alpha\\
        B^{1/\alpha} - \se^2 P^{-1} B^{1+1/\alpha} & B < P^\alpha
    \end{cases}}}\\
    &\quad + \frac{2L}{N}\lr{-\sum_{j > B^{1/\alpha}}^P Bj^{-\alpha} \inprod{\wx_{||}}{u_j}^2 + \lr{\sum_{j>B^{1/\alpha}}^P Bj^{-\alpha/2} \lr{\delta_{j=k} + P^{-1/2} \inprod{v_j}{\epsilon}} \inprod{\wx_{||}}{u_j}}^2}\\
    m_3 &= \frac{L}{N}\lr{\norm{\wx_\perp}^2 + \sum_{j > B^{1/\alpha}}^P \inprod{\wx_{||}}{u_j}^2}\sum_{j>B^{1/\alpha}}^P Bj^{-\alpha/2} \lr{\delta_{j=k} + P^{-1/2} \inprod{v_j}{\epsilon}} \inprod{\wx_{||}}{u_j}\\
    m_4 &= \frac{L}{4N}\lr{\norm{\wx_\perp}^2 + \sum_{j > B^{1/\alpha}}^P \inprod{\wx_{||}}{u_j}^2}^2.
\end{align*}
Evaluating the expectation over $x,X$ and $\epsilon$, the generalization error is of the form
\begin{align*}
    \E{\langle f(x) - y \rangle^2} &= \E{(m_1 - y)^2 + m_2}.
\end{align*}
We choose a centered data model with $\E{\inprod{\wx_{||}}{u_k}} = 0$. Note that $\E{y^2} = 1 + \se^2$ and $\E{\norm{\wx_\perp}^2} = P^{1-\alpha}$.

We evaluate the generalization error when $k = o(P^{1/\alpha})$ and $\se^2 = o(P^{1-\alpha})$, and we identify the temperature that minimizes the error. We find that temperatures satisfying $k > B^{1/\alpha}$ produce constant generalization error, while temperatures satisfying $k < B^{1/\alpha}$ produce vanishing generalization error. It thus suffices to obtain the following cases, written here to leading order and neglecting positive constant factors:
\begin{align*}
	\E{\langle f(x) - y \rangle^2} &= \lr{1 - \frac{LP}{N}}\begin{cases}
    	P^{1-\alpha} & k < B^{1/\alpha}, B > P^\alpha\\
	    B^{-1+1/\alpha} & k < B^{1/\alpha}, B < P^\alpha.
	\end{cases}
\end{align*}
We remark that the perturbative expansion no longer holds when $\se^2 \gg P^{1-\alpha}$; when $\se^2 \sim P^{1-\alpha}$ it holds but our approximation scheme is insufficiently fine-grained to determine the signs of constants.

\bibliography{bibliography}
\bibliographystyle{alpha}

\newpage
\appendix
\section{Reference Quantities}
\label{app:quantities}
We provide here quantities useful for computations in the $k$th principal power law data model introduced in \S \ref{sec:power}. For convenience, we reproduce the definition here.
\begin{definition}
Let $\wX \in \R^{N_0\times P}$ be the dataset reshaped by the nonlinearity, as defined in \S \ref{sec:notation}. Let $Y \in \R^P$ be the labels. Define the SVD decomposition $\wX = \sqrt{P} \sum_j \sqrt{\lambda_j} u_j v_j^T$. The \emph{$k$th-principal power law model} assumes input data satisfying, for all $j=1,\dots, P$,
\begin{align}
    \lambda_j \propto j^{-\alpha}
\end{align}
with $\alpha > 0$ and proportionality constant such that $\tr(\Sigma) = 1$, and labels $Y = \sqrt{P}v_k + \epsilon$ and for $\epsilon \in \R^P$ with i.i.d. entries sampled from $\mathcal{N}(0, \se^2)$.
\end{definition}
We write our quantities in terms of the normalized inverse temperature
\begin{align*}
    B = \beta P.
\end{align*}
We will use subscript or superscript $< 1/B$ and $> 1/B$ to denote the projector in the basis of $\wX^T \wX$ that only preserves directions spanned by the eigenvectors corresponding to eigenvalues $< 1/B$ or $> 1/B$, respectively. We define
\begin{align*}
    M_\beta = \lr{\wX^T \wX + I/\beta}^{-1}, \quad \wbXd = M_\beta \wX^T
\end{align*}
and adopt the expansion
\begin{align*}
    M_\beta \approx P^{-1}\lr{(\wX^T \wX)_{> 1/B}}^{-1} + P^{-1} B I_{< 1/B}.
\end{align*}
We recall that $\theta_*$ denotes the minimum-norm least-squares solution,
\begin{align*}
    \theta_* = \wX\lr{\wX^T \wX}^{-1} Y.
\end{align*}
In the large $P, N_0$ limit with constant $P/N_0 < 1$, we evaluate various relevant quantities for the evidence and posterior. Finally, we assume $\alpha > 1$ and $k=o(P)$ to invoke the self-averaging assumptions discussed in the main text. Quantities independent of the labels are as follows.

\[
\begin{array}{
    |c!{\vrule width 1pt}c|c|
}
\hline
\text{\textbf{term}} & B < P^\alpha & B > P^\alpha\\
\noalign{\hrule height 1pt}
\dim_{< 1/B} & P &  0\\
\hline
\dim_{> 1/B} & B^{1/\alpha} & P\\
\hline
\tr(\Sigma_{< 1/B}) &  B^{(1-\alpha)/\alpha} & 0 \\
\hline
\tr(\Sigma_{< 1/B}^2) & B^{(1-2\alpha)/\alpha} & 0 \\
\hline
\tr(M_\beta) & B & P^\alpha \\
\hline
\tr(\wbXd \wX) & B^{1/\alpha} & P \\
\hline
\frac{1}{P}\tr(\wX \wbXd \wX \wX^T) & 1 & 1 \\
\hline
\tr((\wbXd \wX)^2) & B^{1/\alpha} & P \\
\hline
\tr(\log M_\beta) & P\log\beta & P \log P \\
\hline
\end{array}
\]
\newline
\noindent
The label-dependent quantities are as follows. Each quantity in the leftmost column is given by the sum of one of the entries in the next two columns and one of the entries in the last two columns, according to the relation between $k, B$ and $P$. Note that when, for example, $B > P^\alpha$, then $k$ necessarily satisfies $k < B^{1/\alpha}$.

\[
\begin{array}{
    |c!{\vrule width 1pt}c|c|c|c|
}
\hline
\text{\textbf{term}} & k > B^{1/\alpha} & k < B^{1/\alpha} & B < P^\alpha & B > P^\alpha\\
\noalign{\hrule height 1pt}
\norm{M_\beta^{1/2}\wX^T \wtheta_*}^2 & B & k^\alpha & \sigma^2 P^{-1} B^{1+1/\alpha} & \sigma^2 P^\alpha\\
\hline
\norm{\wbXd \theta_*}^2 & P^{-1}B^2 & P^{-1}k^{2\alpha} & \sigma^2 P^{-1}B^2 & \sigma^2 P^{2\alpha-1}\\
\hline
\norm{\wX \wbXd \theta_*}^2 & B^2 k^{-\alpha}& k^\alpha & \sigma^2 P^{-1} B^{1+1/\alpha} & \sigma^2 P^\alpha\\
\hline
\frac{1}{P}\norm{\wX^T\wX\wbXd \theta_*}^2 & B^2 k^{-2\alpha}& 1 & \sigma^2 B^2 & \sigma^2 P^{2\alpha}\\
\hline
\norm{M_\beta^{1/2}\wX^T\wX\wbXd \theta_*}^2 & B^3 k^{-2\alpha}& k^\alpha & \sigma^2 P^{-1} B^{1+1/\alpha} & \sigma^2 P^\alpha\\
\hline
\end{array}
\]
\newline
\noindent
As an example of how these quantities are obtained, we evaluate one of the more involved terms: $\norm{M_\beta^{1/2}\wX^T \wX \wbXd \wtheta_*}$. Since $\wX^T \wtheta_* = Y$, we have
\begin{align*}
    \norm{M_\beta^{1/2}X^TX\wbXd \theta_*}^2 &= \lr{\sqrt{P}v_k + \epsilon}^T \wX^T M_\beta (\wX^T \wX) M_\beta (\wX^T \wX) M_\beta \wX^T \lr{\sqrt{P}v_k + \epsilon}.
\end{align*}
Note that $M_\beta$ and $\wX^T \wX$ commute: in the eigenbasis of $\wX^T \wX$, up to constant factors
\begin{align*}
    M_\beta &= P^{-1} V \,\diag\lr{\begin{cases}
        j^\alpha & j > 1/B\\
        B & j < 1/B
    \end{cases}}V^T, \quad \wX^T \wX = P V\, \diag(j^{-\alpha}) V^T
\end{align*}
Consequently, our term can be written as
\begin{align*}
    \norm{M_\beta^{1/2}\wX^T\wX\wbXd \theta_*}^2 &= \lr{\sqrt{P}v_k + \epsilon}^T (\wX^T \wX)^2 M_\beta^3 \lr{\sqrt{P}v_k + \epsilon}\\
    &= \lr{\sqrt{P}v_k + \epsilon}^T V P^{-1}\,\diag\lr{\begin{cases}
    	j^\alpha & j < B^{1/\alpha}\\
		B^3 j^{-2\alpha} & j > B^{1/\alpha}
    \end{cases}} V^T\lr{\sqrt{P}v_k + \epsilon}.
\end{align*}
Since the noiseless part of the label model corresponds to a column vector of $V$, and the label noise satisfies $\E{\inprod{\epsilon}{v_k}}=0$ and $\E{\epsilon_j^2} = \se^2$ for each entry $\epsilon_j$ of $\epsilon$, this evaluates to
\begin{align*}
    \norm{M_\beta^{1/2}\wX^T\wX\wbXd \theta_*}^2 &= \se^2 P^{-1} \lr{\sum_{j=1}^{\min(P, B^{1/\alpha})} j^\alpha + \sum_{j=B^{1/\alpha}}^{P} B^3j^{-2\alpha}} + \begin{cases}k^\alpha & k < B^{1/\alpha} \\ B^3 k^{-2\alpha} & k > B^{1/\alpha}\end{cases}\\
    &= \se^2 P^{-1} \begin{cases} B^{1+1/\alpha} + B^3 B^{-2+1/\alpha} & B < P^\alpha \\ P^{\alpha+1} & B > P^\alpha \end{cases} + \begin{cases}k^\alpha & k < B^{1/\alpha} \\ B^3 k^{-2\alpha} & k > B^{1/\alpha}\end{cases}\\
    &= \se^2 \begin{cases} P^{-1} B^{1+1/\alpha} & B < P^\alpha \\ P^\alpha & B > P^\alpha \end{cases} + \begin{cases}k^\alpha & k < B^{1/\alpha} \\ B^3 k^{-2\alpha} & k > B^{1/\alpha}\end{cases}
\end{align*}
to leading order, neglecting constant factors. The same style of computation applies to the other rows of the above tables.

We also provide identities used in the proof of Proposition~\ref{prop:generror}. For $\wX = \sqrt{P}UD^{1/2}V^T$ and corresponding column vectors $u_i, v_i$, we have the following:
\begin{align*}
    \lr{\wbXd \wx}_\mu &= P^{-1/2} \sum_j \lr{\delta_{j < B^{1/\alpha}} + Bj^{-\alpha} \delta_{j > B^{1/\alpha}}} j^{\alpha/2} \inprod{\wx_{||}}{u_j} V_{\mu j}\\
    \wX \wbXd \wx &= \sum_j \lr{\delta_{j < B^{1/\alpha}} + B j^{-\alpha} \delta_{j > B^{1/\alpha}}} \inprod{\wx_{||}}{u_j} u_j\\
    \lr{\wX \wbXd}^2 \wx &= \sum_j \lr{\delta_{j < B^{1/\alpha}} + B^2j^{-2\alpha}\delta_{j > B^{1/\alpha}}} \inprod{\wx_{||}}{u_j} u_j\\
    \wxpb &= \wx_\perp + \sum_{j \geq B^{1/\alpha}} \lr{1 - Bj^{-\alpha}} \inprod{\wx_{||}}{u_j} u_j \approx \wx_\perp + \sum_{j \geq B^{1/\alpha}} \inprod{\wx_{||}}{u_j} u_j.
\end{align*}
In the noisy data-generating process, we impose label vector $Y = \sqrt{P} v_k + \epsilon$ to also obtain identities
\begin{align*}
    \lr{\wbXd \wtheta_*}_\mu &= P^{-1}V_{\mu j}\lr{j^\alpha \delta_{j < B^{1/\alpha}} + B\delta_{j > B^{1/\alpha}}} \lr{\sqrt{P} \inprod{v_j}{v_k} + \inprod{v_j}{\epsilon}}\\
    &= P^{-1/2} V_{\mu k} \lr{k^\alpha \delta_{k < B^{1/\alpha}} + B\delta_{k > B^{1/\alpha}}} + P^{-1} \sum_j V_{\mu j} \lr{j^\alpha \delta_{j < B^{1/\alpha}} + B\delta_{j > B^{1/\alpha}}} \inprod{v_j}{\epsilon}\\
    \wX \wbXd \wtheta_* &= \lr{k^{\alpha/2} \delta_{k < B^{1/\alpha}} + Bk^{-\alpha/2}\delta_{k > B^{1/\alpha}}} u_k + P^{-1/2} \sum_j \lr{j^{\alpha/2} \delta_{j < B^{1/\alpha}} + Bj^{-\alpha/2}\delta_{j > B^{1/\alpha}}} \inprod{v_j}{\epsilon} u_j.
\end{align*}
Since $\inprod{v_j}{\epsilon} \sim \mathcal{N}(0, \se^2)$, this produces quantities
\begin{align*}
    \norm{\wbXd \wx}^2 &= P^{-1} \sum_j \inprod{\wx_{||}}{u_j}^2 \begin{cases}
        j^\alpha & j < B^{1/\alpha}\\
        B^2 j^{-\alpha} & j > B^{1/\alpha}
    \end{cases}\\
    \inprod{\wbXd \wtheta_*}{\wbXd \wx} &= P^{-1} \inprod{\wx_{||}}{u_k} \begin{cases}
        k^{3\alpha/2} & k < B^{1/\alpha}\\
        B^2 k^{-\alpha/2} & k > B^{1/\alpha}
    \end{cases} + P^{-3/2} \sum_j \inprod{v_j}{\epsilon} \inprod{\wx_{||}}{u_j} \begin{cases}
        j^\alpha & j < B^{1/\alpha}\\
        B^2 j^{-\alpha} & j > B^{1/\alpha}
    \end{cases}\\
    \inprod{\wX \wbXd \wtheta_*}{\wx} &= P^{-1/2}\sum_j \inprod{\wx_{||}}{u_j} \lr{\sqrt{P}\delta_{j=k} + \inprod{\epsilon}{v_j}} \begin{cases}
        j^{\alpha/2} & j < B^{1/\alpha}\\
        Bj^{-\alpha/2} & j > B^{1/\alpha}
    \end{cases}\\
    \norm{\wX\wbXd\wtheta_*}^2 &= \begin{cases}
        k^\alpha & k < B^{1/\alpha}\\
        B^2 k^{-\alpha} & k > B^{1/\alpha}
    \end{cases} + \se^2 \begin{cases}
        P^\alpha & B > P^\alpha\\
        P^{-1} B^{1+1/\alpha} & B < P^\alpha
    \end{cases}
\end{align*}
and
\begin{align*}
    \norm{\wxpb}^2 &= \norm{\wx_\perp}^2 + \sum_{j > B^{1/\alpha}}^P \inprod{\wx_{||}}{u_j}^2\\
    \inprod{\wX \wbXd \wtheta_*}{\wxpb} &= \sum_{j>B^{1/\alpha}}^P Bj^{-\alpha/2} \lr{\delta_{j=k} + P^{-1/2} \inprod{v_j}{\epsilon}} \inprod{\wx_{||}}{u_j},
\end{align*}
and since $\wX \wbXd \wxpb = \wX \wbXd \wx - \lr{\wX \wbXd}^2 \wx$,
\begin{align*}
    \inprod{\wX \wbXd \wxpb}{\wxpb} &= \inprod{\wX \wbXd \wx}{\wx} - 2\norm{\wX \wbXd \wx}^2 + \inprod{\lr{\wX \wbXd}^2 \wx}{\wX \wbXd \wx}\\
    &= \sum_{j > B^{1/\alpha}}^P Bj^{-\alpha} \inprod{\wx_{||}}{u_j}^2\\
    \inprod{\wX \wbXd \wxpb}{\wX \wbXd \wtheta_*} &= \inprod{\wX \wbXd \wx}{\wX \wbXd \wtheta_*} - \inprod{\lr{\wX \wbXd}^2 \wx}{\wX \wbXd \wtheta_*}\\
    &= \sum_{j > B^{1/\alpha}}^P B^2 j^{-3\alpha/2} \lr{\delta_{j=k} + P^{-1/2}\inprod{v_j}{\epsilon}} \inprod{\wx_{||}}{u_j}.
\end{align*}

\end{document}